\documentclass[10pt,journal,compsoc]{IEEEtran}

\usepackage{marginnote}
\usepackage{amsmath,graphicx}
\usepackage[american]{babel}
\usepackage{color}
\usepackage{lineno,hyperref}
\usepackage{amssymb,amsmath,amsthm}
\usepackage{multirow}
\usepackage{bm}
\usepackage{subfigure}
\usepackage{booktabs}
\usepackage{threeparttable,multirow}
\usepackage{makecell}
\usepackage[ruled,linesnumbered]{algorithm2e}
\usepackage{microtype}

\usepackage{schemata}
\newcommand\AB[2]{\schema{\schemabox{#1}}{\schemabox{#2}}}
\usepackage{enumitem}

\usepackage{ragged2e}
\usepackage{times}  
\usepackage{helvet}  
\usepackage{courier}  
\usepackage{url}  
\usepackage{graphicx}  
\usepackage{chronology}
\newtheorem{definition}{Definition}

\newtheorem{theorem}{Theorem}
\newtheorem{lemma}{Lemma}

\usepackage{tikz}
\usetikzlibrary{calc,fit,arrows.meta,calc,positioning,decorations.pathreplacing}
\usepackage{xifthen}

\colorlet{A}{gray}
\colorlet{B}{lightgray}
\colorlet{C}{white}

\pgfdeclarelayer{background}
\pgfdeclarelayer{foreground}
\pgfsetlayers{background,main,foreground}

\tikzset{
        timeline/.style={arrows={}}%
        ,timeline style/.style={timeline/.append style={#1}}%
        ,year label/.style={font=\small\bfseries,below}
        ,year label style/.style={year label/.append style={#1}}%
        ,year tick/.style={tick size=0pt}%
        ,year tick style/.style={year tick/.append style={#1}}%
        ,minor tick/.style={tick size=0pt, very thin}%
        ,minor tick style/.style={minor tick/.append style={#1}}%
        ,period/.style={solid,line width=\timelinewidth,line cap=square}%
        ,periodbox/.style={font=\small\bfseries,text=black}
        ,eventline/.style={draw,red,thick,line cap=round,line join=round}%
        ,eventbox/.style={rectangle,rounded corners=3pt,inner sep=3pt,fill=red!25!white,text width=3cm,anchor=west,text=black,align=left,font=\small}%
        ,tick size/.code={\def\ticksize{#1}}%
        ,labeled years step/.code={\def\yearlabelstep{#1}}%
        ,minor tick step/.code={\def\minortickstep{#1}}%
        ,year tick step/.code={\def\yeartickstep{#1}}%
        ,enlarge timeline/.code={\def\enlarge{#1}}%
        ,eventboxa/.style={eventbox,text width=#1,draw=A,fill=none}%
        ,eventboxb/.style={eventbox,text width=#1,draw=A,fill=none}%
}

\newcommand*{\drawtimeline}[5][]{%
        \def\fromyear{#2}%
        \def\toyear{#3}%
        \def\timelinesize{#4}%
        \def\timelinewidth{#5}%
        \pgfmathsetmacro{\timelinesizept}{\timelinesize}%
        \pgfmathsetmacro{\timelinewidthpt}{\timelinewidth}%
        \pgfmathsetmacro{\timelineoffset}{\timelinewidth/2}
        \pgfmathsetmacro{\timelineoffsetpt}{\timelineoffset}
        \begin{scope}[x=1pt, y=1pt, 
                labeled years step=1,
                minor tick step=0.25,%
                enlarge timeline=0cm,%
                year tick step=1,#1]
                \pgfmathsetmacro{\enlargept}{\enlarge}
                \pgfmathsetmacro{\yearticksep}{\timelinesize/((\toyear-\fromyear)/\yeartickstep)}
                \pgfmathsetmacro{\minorticksep}{\timelinesize/((\toyear-\fromyear)/\minortickstep)}
                \pgfmathsetmacro{\minorticklast}{\minorticksep/\minortickstep}
                \foreach \y[remember=\y as \lasty (initially 0), count=\i from \fromyear] in {0,\yearticksep,...,\timelinesizept}{
                        \coordinate (Y-\i) at (\y,0);
                        \draw[year tick] (\y,-\ticksize/2) --  ++(0,\ticksize);
                        \ifnum\i=\toyear\breakforeach\else
                        \foreach \q[count=\j from 0] in {0,\minorticksep,...,\minorticklast}
                        {
                                \coordinate (Y-\i-\j) at (\q+\y,0);
                                \draw[minor tick] (\q+\y,-\ticksize/2) -- ++(0,\ticksize);
                        };\fi};%
                \pgfmathsetmacro{\nextyear}{int(\fromyear+\yearlabelstep)}
                \draw[timeline] (0,0) -- ++(-\enlargept,0) (0,0) -- ++(\timelinesizept,0) coordinate (end) -- ++(\enlargept,0);
        \end{scope}%

}

\newcommand{\vevent}[9]{
        \pgfmathtruncatemacro{\syr}{#2}
        \pgfmathtruncatemacro{\smth}{#3-1}
        \pgfmathsetmacro{\dim}{#4/31}
        \ifthenelse{#3=12}{%
                \pgfmathtruncatemacro{\fyr}{#2+1}
                \pgfmathtruncatemacro{\fmth}{0}
        }{%
                \pgfmathtruncatemacro{\fyr}{#2}
                \pgfmathtruncatemacro{\fmth}{#3}
        }
        \draw[eventline,#1]($(Y-\syr-\smth)!\dim!(Y-\fyr-\fmth)$) -- ++(#5) -- ++(#6) node[#7] (#8) {#9};
}

\tikzset{
        block/.style 2 args = {
                draw=none, inner sep=0, outer sep=0,
                rounded corners=3pt,
                fit=(#1) (#2)}
}

\usetikzlibrary{calc,matrix,backgrounds,fit,positioning}
\tikzset{%
	nodeStyleGreen/.style={
		draw=green!30!black,
		fill=green!50!lime!30,
		align=left,
		very thick,
		rounded corners
	},
	nodeStyleRed/.style={
		draw=red,
		fill=red!50!lime!30,
		align=left,
		very thick,
		rounded corners
	},
	nodeStyleBlue/.style={
		draw=blue,
		fill=blue!50!lime!30,
		align=left,
		very thick,
		rounded corners
	},
	nodeStyleYellow/.style={
	draw=yellow!80!black,
	fill=yellow!50!lime!30,
	align=left,
	very thick,
	rounded corners
	},
	lineStyleRed/.style={
		color=red,opacity=0.75,line width=2pt,
	},
	lineStyleGreen/.style={
		color=green!40!black,opacity=0.75,line width=2pt,
	},
	lineStyleBlue/.style={
		color=blue,opacity=0.75,line width=2pt,
	},
	lineStyleYellow/.style={
	color=yellow!80!black,opacity=0.75,line width=2pt,
},
}

\usetikzlibrary{arrows,chains,matrix,positioning,scopes}
\makeatletter
\tikzset{join/.code=\tikzset{after node path={%
			\ifx\tikzchainprevious\pgfutil@empty\else(\tikzchainprevious)%
			edge[every join]#1(\tikzchaincurrent)\fi}}}
\makeatother
\tikzset{>=stealth',every on chain/.append style={join},
	every join/.style={->}}
\tikzstyle{labeled}=[execute at begin node=$\scriptstyle,
execute at end node=$]

\ifCLASSOPTIONcompsoc
  \usepackage[nocompress]{cite}
\else
  \usepackage{cite}
\fi
\ifCLASSINFOpdf
\else
\fi

\begin{document}

\clearpage

\twocolumn
\pagenumbering{arabic}
\setcounter{page}{1}
%
\title{Towards a Unified Quadrature Framework for Large-Scale Kernel Machines}
%
%
%

\author{Fanghui Liu, Xiaolin Huang, Yudong Chen, Johan A.K. Suykens
\thanks{
F. Liu and J.A.K. Suykens are with the Department of Electrical Engineering
(ESAT-STADIUS), KU Leuven, B-3001 Leuven, Belgium (email: \{fanghui.liu;johan.suykens\}@esat.kuleuven.be).
 }
\thanks{
	X. Huang is with Institute of Image Processing and Pattern Recognition, and also with Institute of Medical Robotics, Shanghai Jiao Tong University, Shanghai 200240, P.R. China (e-mail: xiaolinhuang@sjtu.edu.cn).
}
\thanks{
	Y. Chen is with School of Operations Research and Information Engineering, Cornell University, Ithaca, NY 14850 USA (e-mail: yudong.chen@cornell.edu).
}
}

\markboth{}%
{Shell \MakeLowercase{\textit{et al.}}: Bare Advanced Demo of IEEEtran.cls for IEEE Computer Society Journals}
%



\IEEEtitleabstractindextext{%
\begin{abstract}
	\justifying  
In this paper, we develop a quadrature framework for large-scale kernel machines via a numerical integration representation. 
Considering that the integration domain and measure of typical kernels, e.g., Gaussian kernels, arc-cosine kernels, are \emph{fully symmetric}, we leverage deterministic fully symmetric interpolatory rules to efficiently compute quadrature \emph{nodes} and associated weights for kernel approximation.
The developed interpolatory rules are able to reduce the number of needed nodes while retaining a high approximation accuracy.
Further, we randomize the above deterministic rules by the classical Monte-Carlo sampling and \emph{control variates} techniques with two merits: 1) The proposed stochastic rules make the dimension of the feature mapping flexibly varying, such that we can control the discrepancy between the original and approximate kernels by tuning the dimnension.
2) Our stochastic rules have nice statistical properties of unbiasedness and variance reduction with fast convergence rate.
In addition, we elucidate the relationship between our deterministic/stochastic interpolatory rules and current quadrature rules for kernel approximation, including the sparse grids quadrature and stochastic spherical-radial rules, thereby unifying these methods under our framework.
Experimental results on several benchmark datasets show that our methods compare favorably with other representative kernel approximation based methods.
\justifying  
\end{abstract}

}

\maketitle

\IEEEdisplaynontitleabstractindextext

%
\IEEEpeerreviewmaketitle

\ifCLASSOPTIONcompsoc
\IEEEraisesectionheading{\section{Introduction}\label{sec:introduction}}
\else
\section{Introduction}
\fi
\IEEEPARstart{K}{ernel} methods \cite{Sch2003Learning,suykens2002least,kafai2018croification} have shown to be powerful in statistical machine learning, but often scale poorly to large datasets in terms of space and time complexity \cite{liu2020learning,dang2020large,liu2018generalization}.
To make kernel methods scalable, the class of random Fourier features (RFF) \cite{rahimi2007random}  is one of the most effective kernel approximation techniques.
RFF transforms input features to a new space for approximating the original kernel function, and then conducts linear learning in this space.
It spawns the new direction on kernel approximation for scaling up traditional kernel methods \cite{lopez2014randomized,sun2018but}, recent convolutional neural tangent kernel \cite{arora2019exact}, and attention in Transformers \cite{choromanski2021rethinking,peng2021random}.
Partly due to its remarkable repercussions, Rahimi and Recht \cite{rahimi2007random} won the test-of-time award for their seminal work on RFF at NeurIPS 2017.

Formally, given a positive definite kernel $k(\cdot, \cdot): \mathbb{R}^d \times \mathbb{R}^d \rightarrow \mathbb{R}$, 
we focus on kernel approximation in which the kernel $k$ admits the following $d$-dimensional integral representation $I_d$
\begin{equation}\label{originte}
k(\bm x, \bm y) \! := \! I_d\left(f_{\bm x \bm y} \right) \!=\! \int_{\mathbb{R}^d} f_{\bm x \bm y}(\bm \omega) \mu(\mathrm{d}{\bm \omega}) \!=\! \mathbb{E}_{\bm \omega \sim \mu} [f_{\bm x \bm y}(\bm \omega)]\,,
\end{equation}
with the integral (probability) measure $\mu$ being standard multivariate Gaussian, i.e., $\bm \omega = [\omega_1, \cdots, \omega_d]^{\!\top} \sim \mathcal{N}(\bm 0, \bm I_d)$.
The integrand $f_{\bm x \bm y}$, $f$ for short, is defined as $f(\bm \omega):=\langle  \phi\left(\bm \omega^{\!\top} \bm{x}\right),  \phi\left(\bm \omega^{\!\top} \bm{y}\right) \rangle $ with a nonlinear activation function $\phi$.
As demonstrated by \cite{lyu2017spherical,munkhoeva2018quadrature}, various kernels admit this $d$-dimensional integration representation by choosing different $\phi$.
For example, the popular Gaussian kernel corresponds to $\phi(x) = [\cos(x), \sin(x)]^{\!\top}$;
the zero-order arc-cosine kernel \cite{cho2009kernel} admits this representation by choosing $\phi(x)$ as the Heaviside function; and the first-order arc-cosine kernel \cite{cho2009kernel} corresponds to $\phi(x) = \max \{ 0, x \}$, i.e., the ReLU activation function commonly-used in deep neural networks.

To approximate the kernel function in Eq.~\eqref{originte}, RFF\footnote{In the original paper \cite{rahimi2007random}, RFF builds on Bochner's theorem \cite{wendland2004scattered} that requires the kernel to be shift-invariant, i.e., $k(\bm x, \bm y) = k(\bm x - \bm y)$, which excludes arc-cosine kernels used in this paper. However, RFF is still able to provide an unbiased approximation of arc-cosine kernels by Monte-Carlo sampling according to the integral representation~\eqref{originte}.} uses Monte-Carlo sampling to draw random features $\{ \bm \omega_i \}_{i=1}^N$ from $\mathcal{N}(\bm 0, \bm I_d)$ such that $k(\bm x, \bm y) \approx {1}/{N} \sum_{i=1}^N f_{\bm x \bm y}(\bm \omega_i)$.
Apart from such random sampling based scheme, an alternative way is to use quadrature rules in a deterministic fashion
\begin{equation}\label{origintenew}
k(\bm x, \bm y)  \approx \! \sum_{i=1}^N \! a_i f_{\bm x \bm y}(\bm \gamma_i) = \langle \Phi(\bm x), \Phi(\bm y) \rangle\,,
\end{equation}
where $\bm \gamma_i  \in \mathbb{R}^d$ are called the quadrature \emph{nodes}, $a_i \in \mathbb{R}$ are the corresponding weights, and $\Phi: \mathbb{R}^d \rightarrow \mathbb{R}^N$ is the related explicit feature mapping. 
The nodes and weights are \emph{deterministically} given by various quadrature rules such that there is no approximation error whenever the integrand $f$ belongs to all polynomials with a total degree up to $2L-1$, where $L$ is the accuracy level. 
For example, in the univariate case ($d=1$), Gaussian quadrature (GQ) uses $L$ nodes to deliver the exact value of polynomials up to $(2L-1)$-degree without approximation error for $\omega_1^{i_1} \omega_2^{i_2} \cdots \omega_d^{i_d}$ with $\sum_{j=1}^d i_j \leq 2L-1$.  
{If the integrand $f$ is general, beyond a ($2L-1$)-degree polynomial, Gaussian quadrature still works well. Under this setting, if $f$ has $c$-order bounded derivatives, the mean squared error (MSE) of Gaussian quadrature decreases asymptotically as $\mathcal{O}(N^{-c})$, which is better than the $N^{-1/2}$-consistency of Monte-Carlo sampling \cite{davis2007methods}.}
Gaussian quadrature in the univariate case can be easily extended to multidimensional cases ($d>1$) by product rules but suffers from ``curse of dimensionality": the number of required nodes is $N=L^d$ in an exponential order of $d$.
To tackle this issue, sparse grid quadrature (SGQ) \cite{heiss2008likelihood} uses a linear combination of low-level tensor products of univariate quadrature rules, and thus the number of nodes $N$ by SGQ can be decreased in a polynomial order of $d$.

Recall Eq.~\eqref{originte}, where the integral is not generic but has a nice property: both the integration domain $\mathbb{R}^d$ and the Gaussian measure are \emph{fully symmetric}, with definition deferred to Section~\ref{sec:pre}.
Benefiting from this property, the nodes and the weights can be efficiently obtained from a pre-given vector through permutations and sign changes of its coordinates.
Furthermore, such \emph{fully symmetric} property is helpful to reduce the number of the required nodes $N$ in quadrature rules.
For example, considering $d=25$ with an accuracy level $L=4$ for seventh-degree polynomial exactness approximation, Gaussian quadrature requires $4^{25}$ nodes; SGQ needs 24,751 nodes; while the Deterministic Fully Symmetric intepolatory rule (termed as D-FS) \cite{genz1996fully} needs 22,151 nodes, which reduces over 10\% nodes.
In some cases, the required nodes can be even reduced over 50\%  \cite{novak1999simple}.
Figure~\ref{sgqours} demonstrates the superiority of D-FS against SGQ on time cost and the reduction on required nodes\footnote{D-FS requires the same number of nodes $N=2d+1$ with SGQ in the third-degree rule but needs smaller $N$ than SGQ in higher-degree rules.}.

	\begin{figure} 
	\centering 
	\subfigure[Time cost]{\label{sgqourstime}
		\includegraphics[width=0.232\textwidth]{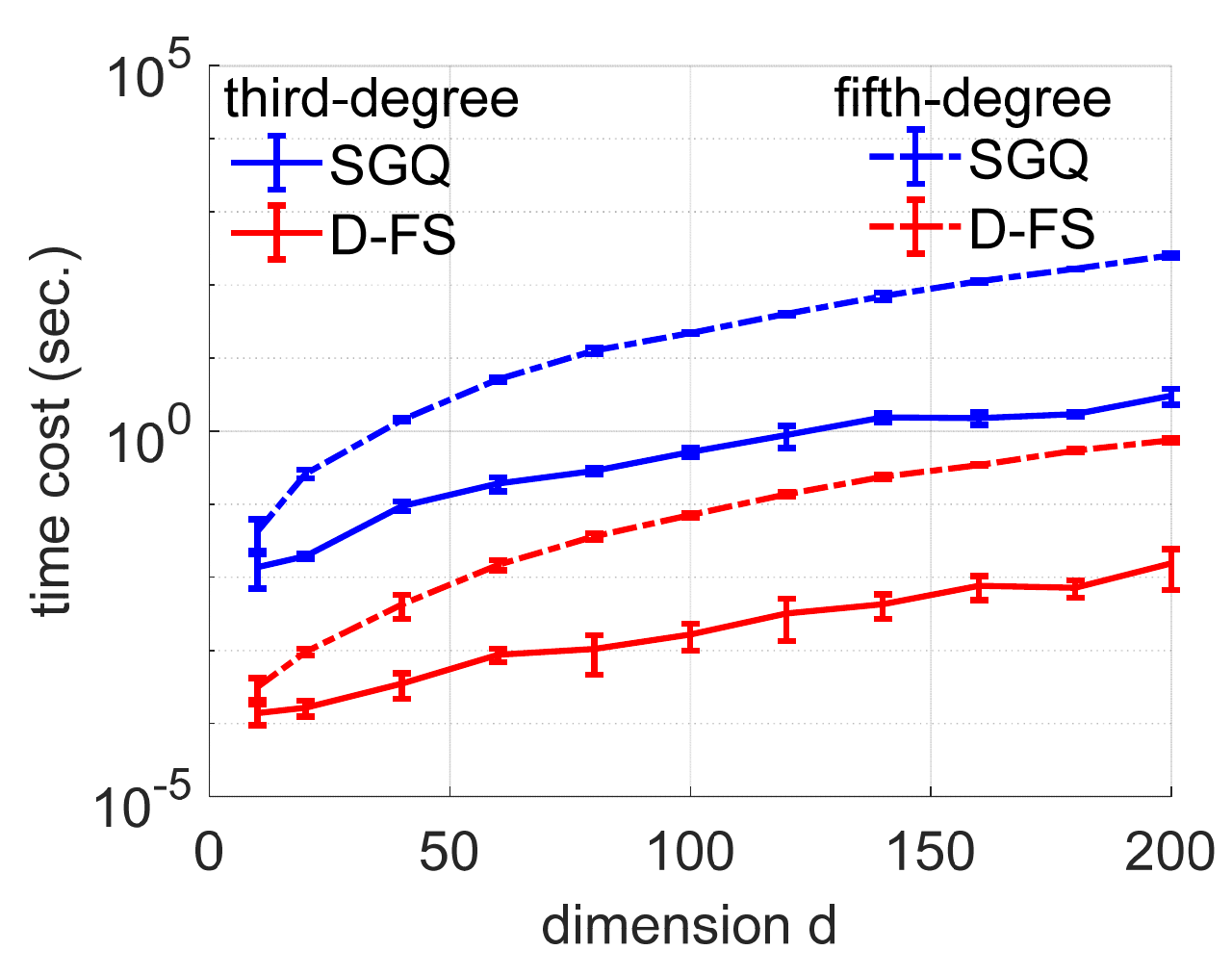}}
\subfigure[${N_{\text{SGQ}} - N_{\text{D-FS}}}$]{\label{sgqom}
		\includegraphics[width=0.215\textwidth]{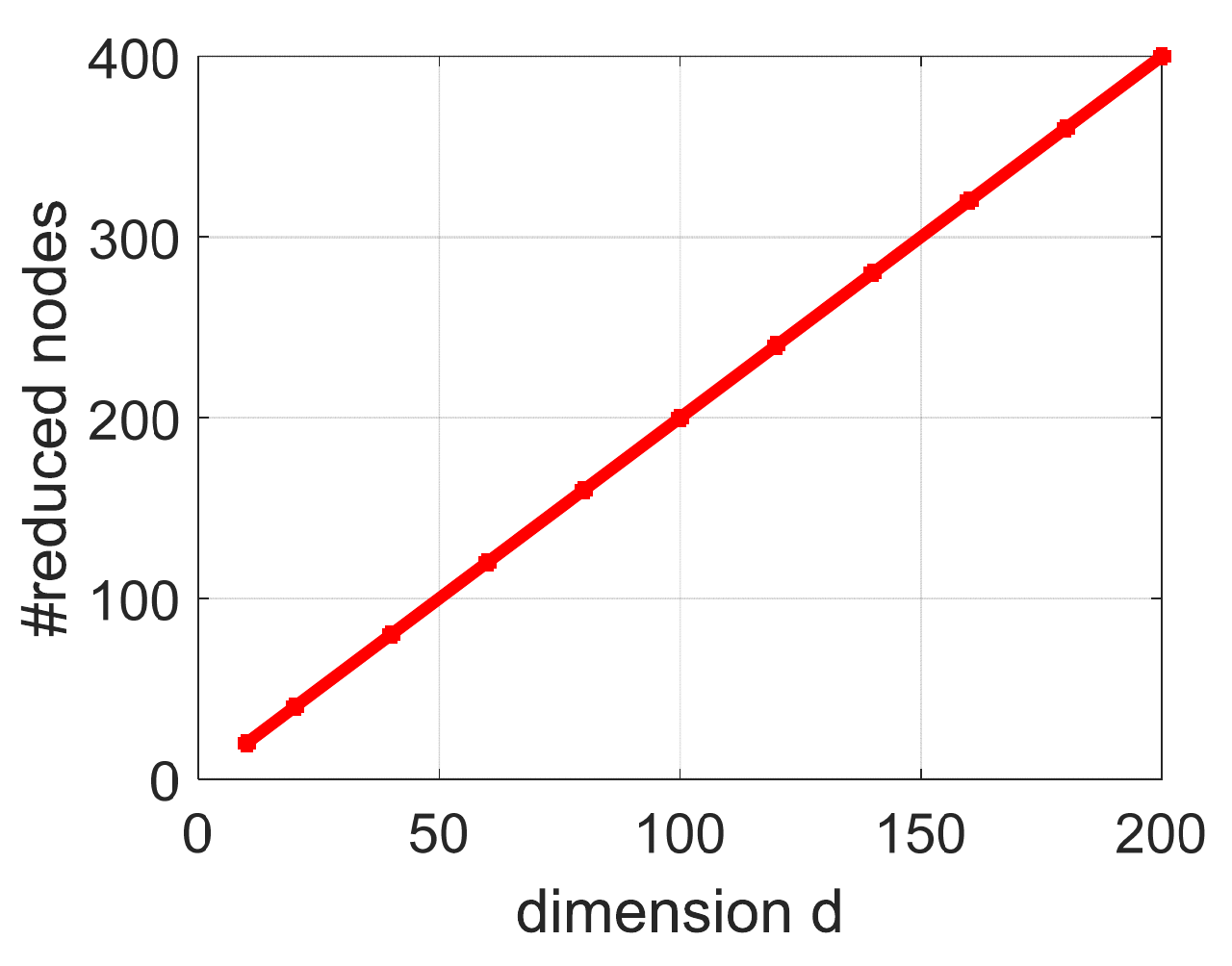}}
	\caption{Benefits of D-FS against SGQ in time cost (a), and the reduction on the required \emph{nodes} in fifth-degree rules (b).}\label{sgqours}
\end{figure}%

Based on the above analysis, we propose to use deterministic \emph{fully symmetric} interpolatory rules \cite{genz1996fully}, i.e., D-FS, for kernel approximation.
Besides, we randomize such deterministic rules to new stochastic versions, termed as S-FS (here ``S" denotes stochastic), which exhibit nice statistical properties: unbiased estimation and variance reduction.
Furthermore, we elucidate the relationship among SGQ \cite{dao2017gaussian}, stochastic spherical-radial (SSR) rules \cite{munkhoeva2018quadrature} and the developed D-FS/S-FS. Thereby, the proposed framework unifies these methods, as shown in Figure~\ref{framework}. 
We make the following contributions:
\begin{itemize}
	\item By virtue of the \emph{fully symmetric} property of the integration~\eqref{originte}, we derive the third/fifth-degree D-FS for kernel approximation. The obtained feature mapping $\Phi(\cdot)$ is fixed-size given $d$, e.g., $N=2d+1$ in the third-degree rule and $N=1+2d^2$ in the fifth-degree rule, and thus our method achieves $\mathcal{O}(d)$ time and space complexity, see Section~\ref{sec:full}. 
	\item We randomize D-FS to a stochastic version S-FS by combining the classical Monte-Carlo sampling and \emph{control variates} techniques. The proposed S-FS has two merits: 1) The dimension of the obtained feature mapping by S-FS can be easily tuned to an arbitrary value for practical requirements. 2) S-FS is theoretically demonstrated to be an unbiased estimator for kernel approximation and achieves variance reduction with fast convergence rates, see in Section~\ref{sec:sfsir}.
	\item We build a unifying quadrature framework for kernel approximation as shown in Figure~\ref{framework}, which unifies our D-FS/S-FS, SGQ and SSR.
	We show that i) by choosing suitable nodes and weights in the third-degree SGQ, it is equivalent to the third-degree D-FS;
	ii) SSR can be regarded as a doubly stochastic version of D-FS: one stochasticity comes from random projection and another source is using a randomized \emph{generator} vector; see in Section~\ref{sec:relation}.
\end{itemize}
Besides, experimental results on several benchmark datasets show that the developed deterministic/stochastic fully symmetric interpolatory rules achieve promising kernel approximation quality and also performs well on classification tasks.

\begin{figure}[t]
	\centering
	\includegraphics[width=0.48\textwidth]{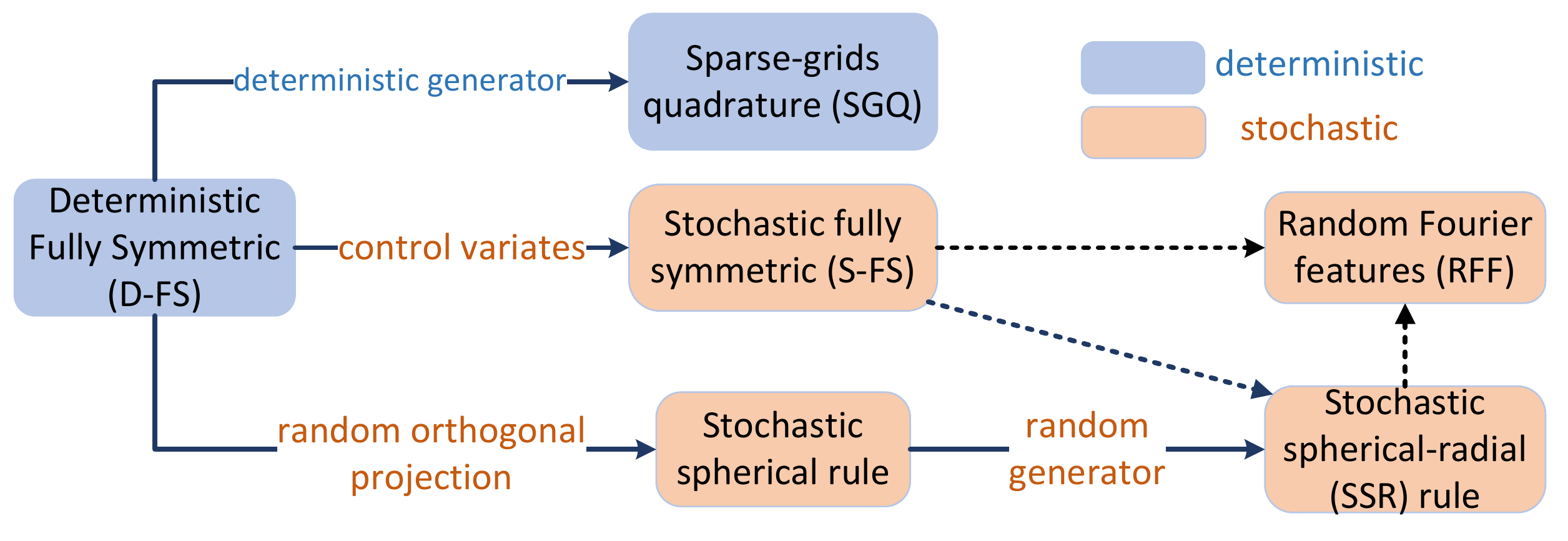}
	\caption{Relationship between quadrature based methods.}\label{framework}
\end{figure}

\section{Related Works and Preliminaries}
\label{sec:relatedwork}

In this section, we give an overview of representative random features based algorithms for kernel approximation, refer to a survey \cite{liu2020survey} for details.
Then we briefly introduce the related \emph{fully symmetric} concepts and basic ideas behind deterministic fully symmetric interpolatory rules in numerical integration.

\subsection{Related Works}
To approximate the kernel function in Eq.~\eqref{originte}, current kernel approximation methods for finding the weights and nodes $\{ a_i, \bm \gamma_i \}_{i=1}^N$, given by Eq.~\eqref{origintenew}, can be divided into \emph{Monte Carlo} and \emph{quadrature} based approaches.

Monte Carlo based methods are often equal-weight rules where the nodes $\{ \bm \gamma_i \}_{i=1}^N$ are obtained by variants of Monte Carlo sampling, and then provide an unbiased estimator of the original kernel.
For example, to approximate the kernel in Eq.~\eqref{originte}, the standard random Fourier features (RFF) adopt $\bm \gamma_i \equiv \bm \omega_i \sim \mathcal{N}(\bm 0, \bm I_d)$ by Monte Carlo sampling and the equal weights $a_1 = \cdots = a_N \equiv 1/N$. 
To reduce the approximation variance, orthogonal random features (ORF) \cite{Yu2016Orthogonal} incorporates an orthogonality constraint on the transformation matrix $\bm W = [\bm \omega_1, \cdots, \bm \omega_N]$, demonstrated by theoretical guarantees on variance reduction \cite{choromanski2017unreasonable}.
Sampling theory \cite{niederreiter1992random} suggests that the convergence rate of Monte-Carlo used in RFF/ORF can be significantly improved by sampling in a deterministic scheme instead of i.i.d. version.
Accordingly, quasi-Monte Carlo (QMC) sampling \cite{Avron2016Quasi}, as a possible middle-ground method, utilizes a low-discrepancy sequence for sampling, and achieves a convergence rate of order $\mathcal{O}((\log N)^c/N)$ on discrepancy \cite{caflisch1998monte}, where $c$ is a constant independent of $N$, but may depend on $d$. The convergence rate can be further improved if the integrand has bounded variation, or higher-order smoothness \cite{leobacher2014introduction,dick2011higher}.
In fact, a series of empirical and theoretical results \cite{lyu2017spherical,choromanski2019unifying} have demonstrated that, coupling samples to be orthogonal to one another (i.e., uniformly distributed over the space), rather than being i.i.d., can significantly improve statistical efficiency.
Apart from the above data-independent sampling schemes used in random features, another line is to utilize data-dependent sampling strategy for better approximation quality and generalization properties for random features. Typical examples include leverage score based sampling \cite{avron2017random}, fast leverage score approximation \cite{rudi2018fast,liu2020random}, Christoffel functions \cite{pauwels2018relating}, and Fourier sparse leverage scores \cite{erdelyi2020fourier}.
We also note that the existence of MCMC based algorithms, incorporating with RFFs for various applications, e.g., function estimation in a distribution sense \cite{kammonen2020adaptive}, and fitting Gaussian processes based latent variable models \cite{gundersen2021latent}, but they are beyond the scope of this paper.

In quadrature based methods, the nodes are usually given by \emph{deterministic} rules (can be extended to stochastic versions) and the weights are often not equal. Examples include Gaussian quadrature \cite{evans1993practical} and SGQ \cite{dao2017gaussian} based on the Smolyak formula \cite{heiss2008likelihood}.
Instead of directly approximating the $d$-dimensional integration, the stochastic spherical-radial (SSR) rules \cite{genz1998stochastic} transform the integration in Eq.~\eqref{originte} to a double-integral over the unit $d$-sphere and over the radius, which are then approximated by stochastic spherical rules and stochastic radial rules, respectively. The idea of SSR has been successfully applied to kernel approximation \cite{munkhoeva2018quadrature} and achieves promising approximation quality.

\subsection{Preliminaries: Fully Symmetric Properties and Rules}
\label{sec:pre}

Next we briefly introduce \emph{fully symmetric sets} and related symmetry concepts, which is needed in this paper. 

\begin{definition}\label{deffullys}
	(Fully symmetric set \cite{cools1997constructing,karvonen2018fully}) Given an integer-valued vector $\bm p =[p_1,p_2,\cdots, p_d] $ with $ p_i \in \{0, 1,\dots,m \}$, 
let $\Pi_{\bm p}$ be the set of all permutations of $\bm p$ and $\mathcal{V}_d$ be the set of all vectors with the form $\bm \nu = [\nu_1, \nu_2, \cdots, \nu_d]$ with $\nu_i = \pm 1$. Then, given a vector $\bm \lambda_{\bm p} =[\lambda_{p_1}, \lambda_{p_2}, \cdots, \lambda_{p_d}]^{\!\top}$, the point set
\begin{equation*}
	\{ \bm \lambda_{\bm p} \} := \bigcup_{q \in \Pi_{\bm p}} \bigcap_{\bm \nu \in \mathcal{V}_d} \Big\{ \left(\nu_{1} \lambda_{q_{1}}, \nu_{2} \lambda_{q_{2}}, \ldots, \nu_{d} \lambda_{q_{d}}\right) \Big\} \subset \mathbb{R}^d \,,
\end{equation*}
is the fully symmetric set generated by $\lambda_{\bm p}$.
\end{definition}
Based on the above definition, the concepts of fully symmetric domain, function, and measure follow naturally.
To be specific, a point domain $\mathcal{A} \subseteq  \mathbb{R}^d$ is said to be \emph{fully symmetric} if $\bm \lambda \in \mathcal{A}$ implies $\bm \lambda' \in \mathcal{A}$, where $\bm \lambda'$ is obtained by permutations and sign changes on the coordinates of $\bm \lambda$. Naturally, $\mathbb{R}^d$ is a fully symmetric domain.
A function $f: \mathbb{R}^d \rightarrow \mathbb{R}$ is \emph{fully symmetric} if it is constant in each fully symmetric set, i.e., $f (\bm x) = f(\bm x')$ for any $\bm x, \bm x' \in \{ \bm \lambda_{\bm p} \}$.
A measure $\mu$ is \emph{fully symmetric} if its density (with respect to the Lebesgue measure) is a fully symmetric function. The Gaussian measure used in Eq.~\eqref{originte} satisfies this condition.
In Definition~\ref{deffullys}, $\bm \lambda_{\bm p}$ is called a \emph{generator} vector and its individual elements are called \emph{generators}.
Further, assuming $\lambda_0 = 0$, the \emph{fully symmetric basic rule} $ f(\bm \lambda_{\bm p})$ is defined by \cite{genz1996fully}
\begin{equation*}
	f(\bm \lambda_{\bm p}) = \sum_{q \in \Pi_{\bm p}} \sum_{\bm \nu \in \mathcal{V}_d} f\left(\nu_{1} \lambda_{q_{1}}, \nu_{2} \lambda_{q_{2}}, \ldots, \nu_{d} \lambda_{q_{d}}\right)\,.
\end{equation*}
{For example, when $d=4$ and $\bm p = (2,0,0,0)$, we have
	\begin{equation*}
		\begin{split}
			f(\bm \lambda_{\bm p}) &= f(\lambda_2,0,0,0) + f(- \lambda_2,0,0,0) + f(0, \lambda_2,0,0) \\
			&\quad + f(0, -\lambda_2,0,0) + f(0, 0, \lambda_2,0) + f(0, 0, -\lambda_2,0) \\
			&\quad + f(0, 0, 0, \lambda_2) + f(0, 0, 0, -\lambda_2) \,.
		\end{split}
	\end{equation*}
}

\begin{definition}\label{deffsir} (Fully symmetric interpolatory rules \cite{genz1996fully}) Define
	$\mathcal{P}^{(m, d)}$ as a set of all distinct $d$-partitions of the integers $\{ 0,1, \dots, m \}$, i.e.
	\begin{equation*}
		\mathcal{P}^{(m, d)}=\left\{\bm p \in \mathbb{N}^d | p_1 \geq p_2 \geq \cdots \geq p_{d} \geq 0, \| \bm p \|_1 \leq m\right\}\,,
	\end{equation*}
the fully symmetric interpolatory rule is defined as
\begin{equation}\label{IQmdf}
	Q^{(m,d)}(f) = \sum_{\bm p \in \mathcal{P}^{(m,d)}} a_{\bm p}^{(m,d)} f(\bm \lambda_{\bm p})\,,
\end{equation}
where the weight $a_{\bm p}^{(m,d)}$ is given by
\begin{equation}\label{weightapmd}
	a_{\bm p}^{(m, d)}=2^{-K} \sum_{\|\bm u \|_1 \leq m-\| \bm p \|_1} \prod_{i=1}^{d} \frac{b_{u_{i}+p_{i}}}{\prod_{j=0, \neq p_{i}}^{u_{i}+p_{i}}\left(\lambda_{p_{i}}^{2}-\lambda_{j}^{2}\right)}\,,
\end{equation}
where $\bm u = [u_1, u_2, \cdots, u_d]$ is the set of the integers $\{ 0,1, \dots, m \}$ and $K$ is the number of nonzero components in $\bm p$. If $\bm q$ is one of the permutations of $\bm p$, then $a^{(m,d)}_{\bm q} = a^{(m,d)}_{\bm p}$.
The coefficient $b_0=1$ and $b_i$ ($i \geq 1$) satisfies
\begin{equation}\label{bi}
	b_{i}=\frac{1}{\sqrt{2 \pi}} \int_{-\infty}^{+\infty} e^{-x^{2} / 2} \prod_{j=0}^{i-1}\left(x^{2}-\lambda_{j}^{2}\right) \mathrm{d} x \quad(i \geq 1)\,.
\end{equation}
\end{definition}
According to Eq.~\eqref{IQmdf}, $Q^{(m,d)}(f)$ stands for the weighted sum of evaluations of $f$ at the nodes of the fully symmetric set on the distinct $d$-partitions $\mathcal{P}^{(m, d)}$.
The theory for fully symmetric interpolatory rules \cite{genz1996fully} demonstrates that $Q^{(m,d)}(f)$ is an approximation to $I_d(f)$ that is exact for all polynomials with the total degree $2m+1$ or less.
The third/fifth-degree rules $Q^{(1,d)}(f)$ and $Q^{(2,d)}(f)$ correspond to $m=1$ and $m=2$, respectively. 
Note that, although the mathematical foundations and derivations of the fully symmetric rule are relatively complex, the obtained feature mapping for kernel approximation in this paper is quite simple and easy to be implemented.
We will illustrate this in the next section.

\noindent{\bf Different from previous works:} When compared to the original work on fully symmetric interpolatory rules \cite{genz1996fully}, the contribution of this paper lies in developing deterministic rules for kernel approximation, especially the derivation of the fifth-degree rule, providing a new stochastic version as an unbiased estimator for kernel approximation, demonstrating nice statistical properties with theoretical guarantees, and casting typical quadrature rules in a unifying framework.

\section{Deterministic Rules for Kernel Approximation}
\label{sec:full}
In this section, we present the third-degree and fifth-degree D-FS for kernel approximation based on the fully symmetric interpolatory rules \cite{genz1996fully}.
We do not employ higher-degree rules in this paper due to sufficient approximation and efficient computation, refer to \cite{arasaratnam2009cubature} with detailed discussion.

According to Eq.~\eqref{IQmdf},  $Q^{(m,d)}(f)$ is a weighted sum of \emph{fully symmetric} basic rules $f(\bm \lambda_{\bm p})$.
Therefore, the kernel $k$, a.k.a. the $d$-dimensional integration~\eqref{originte}, can be approximated by a weighted sum of evaluations of $f$ at the nodes of the fully symmetric set on the distinct $d$-partitions $\mathcal{P}^{(m, d)}$
\begin{equation}\label{IQmdfnew}
k(\bm x, \bm y) \approx Q^{(m,d)}(f) = \sum_{\bm p \in \mathcal{P}^{(m,d)}} a_{\bm p}^{(m,d)} f(\bm \lambda_{\bm p})\,,
\end{equation}
where the weights $a_{\bm p}^{(m,d)}$ and the \emph{generator} vector $\bm \lambda_{\bm p}$ play significant roles in quadrature rules.
Different generation schemes for $\bm \lambda_{\bm p}$ lead to various quadrature based approaches.
For example, SGQ \cite{dao2017gaussian} uses \emph{deterministic} values to generate $\bm \lambda_{\bm p}$; while $\bm \lambda_{\bm p}$ in SSR \cite{munkhoeva2018quadrature} is sampled from a probability distribution.
The developed D-FS in this section follows \cite{genz1996fully} that selects $\bm \lambda_{\bm p}$ in a \emph{deterministic} scheme.
In the next we present this generation procedure equipped with the third/fifth-degree rules for kernel approximation.

\noindent {\bf Third-degree rule:} the kernel $k$ in Eq.~\eqref{originte} is approximated by the third-degree rule $Q^{(1,d)}(f)$ such that $k(\bm x, \bm y) \approx  Q^{(1,d)}(f)$
\begin{equation}\label{q1df}
 Q^{(1,d)}(f) = a_{0}^{(1, d)} f(\bm 0) + a_{1}^{(1, d)} \sum_{i=1}^d \left[ f(\lambda_1 \bm e_i) + f(-\lambda_1 \bm e_i) \right],
\end{equation}
where $\bm e_i$ is a unit vector with the $i$-th element being 1.
The weights are given by $a_{0}^{(1, d)} = 1-{d}/{\lambda_1^2}$ and $a_{1}^{(1, d)} = {1}/{(2\lambda_1^2)}$ according to Eq.~\eqref{weightapmd}.
Finally, the third-degree rule outputs $\{a_i, \bm \gamma_i \}_{i=0}^{2d}$ with 
\begin{equation}\label{fsirwa}
\left\{
\begin{array}{rcl}
\begin{split}
& \bm \gamma_{i} =\bm 0_{d \times 1} ;~  a_{i}=1-{d}/{\lambda_1^2};~ i =0 \\
& \bm \gamma_{i} =\lambda_1 \bm{e}_{i};~ a_{i} ={1}/{2\lambda_1^2};~ 1  \leq i \leq d \\
& \bm \gamma_{i} =-\lambda_1 \bm{e}_{i} ;~ a_{i} ={1}/{2\lambda_1^2};~ d+1 \leq i \leq 2 d\,,
\end{split}
\end{array}\right.
\end{equation}
which results in the number of nodes $N = 2d+1$.
The \emph{generator} vector $\bm \lambda = [\lambda_0, \lambda_1]^{\!\top}$ with $\lambda_0 = 0$ usually selects $\lambda_1$ by successive extensions of the one-dimensional 3-point Gauss–Hermite rule
so that certain sets of weights vanish, i.e., $\lambda_1 = \sqrt{3}$.

\noindent{\bf Feature mapping:} According to the third-degree rule $Q^{(1,d)}(f)$, we finally obtain the explicit feature mapping $\Phi$ for kernel approximation
\begin{equation}\label{determap}
\Phi(\bm x) \!=\! [\sqrt{a_0}\phi(\bm \gamma_0^{\!\top} \bm x), \sqrt{a_1}\phi(\bm \gamma_1^{\!\top} \bm x), \! \cdots \!, \sqrt{a_{2d}}\phi(\bm \gamma_{2d}^{\!\top} \bm x) ]^{\!\top}\,,
\end{equation}
such that $k(\bm x, \bm y) \approx Q^{(1,d)}(f) = \langle \Phi(\bm x), \Phi(\bm y) \rangle$.
Here the weight $a_0 = 1 - d/\lambda_1^2$ might be negative, and we consider the complex number $\sqrt{a_0}$, and thus the approximated kernel is still real-valued.
It can be observed that, generating $\{ a_i, \bm \gamma_i \}_{i=0}^{2d}$ is data-independent and deterministic.
The transformation matrix $\bm W = [\bm \gamma_0, \bm \gamma_1, \cdots, \bm \gamma_{2d}] \in \mathbb{R}^{d \times (2d+1)}$ can be obtained by Eq.~\eqref{fsirwa} as the following
\begin{equation}\label{wmatrix}
\bm W = \left[
\begin{matrix}
0      & -\lambda_1   & \lambda_1  &0 & 0 & \cdots & 0 & 0    \\
0      & 0     &0  &- \lambda_1 & \lambda_1 & \cdots & 0   & 0   \\
\vdots & \vdots & \vdots & \vdots & \vdots & \ddots & \vdots & \vdots  \\
0      & 0      & 0 & 0 & 0 & \cdots   & -\lambda_1   & \lambda_1  \\
\end{matrix}
\right]\,.
\end{equation}

For better illustration of our method, here we take the Gaussian kernel $k(\bm x, \bm y) = \exp\left(-{\| \bm x - \bm y\|_2^2}/{(2\sigma^2)}\right)$ as an example and discuss the difference with RFF.
According to Eq.~\eqref{originte}, the Gaussian kernel can be approximated by $k(\bm x, \bm y) \approx \sum_{i=1}^N a_i \cos[\bm \omega^{\!\top}_i (\bm x - \bm y)]$ with the transformation matrix $\bm W = [\bm \omega_1, \cdots, \bm \omega_N] \in \mathbb{R}^{d \times N}$ with $N$ features as follows.
\begin{figure}[!htb]
	\AB{}
	{
		\AB{RFF:}
		{
			dense: $\bm W = [W_{ij}]_{d \times N}$ with $W_{ij} \sim \mathcal{N}(0,1/\sigma^2)$ \\
			$a_i \equiv 1/N$
		}\\
		\AB{Ours:}
		{
			sparse: $\bm W = [\bm \gamma_0, \bm \gamma_1, \cdots, \bm \gamma_{2d}]$ in Eq.~\eqref{wmatrix}\\
			the weight $a_i$ is given in Eq.~\eqref{fsirwa}
		}
	}
\end{figure}

\noindent In RFF, $a_i \equiv 1/N$ and $W_{ij} \sim \mathcal{N}(0,1/\sigma^2)$ by Monte Carlo sampling. The number of random features $N$ can be manually specified to an arbitrary value; while in D-FS, the transformation matrix $\bm W  \in \mathbb{R}^{d \times (2d+1)}$ and the weights $a_i$ are \emph{deterministic}. Given $d$, the number of needed nodes is $N = 2d+1$ and cannot be easily tuned.
One oblivious advantage of D-FS is that, $\bm W$ is extremely sparse with only $2d$ non-zero elements $\pm \lambda_1$.
Accordingly, generating $\bm W$ needs $\mathcal{O}(d)$ space and time complexity, which is better than RFF with $\mathcal{O}(Nd)$ complexity.
More importantly, when $d$ is given, the nodes, the weights, and the transformation matrix in our deterministic rules can be directly determined, see Eqs.~\eqref{fsirwa} and~\eqref{wmatrix}.
That means, our deterministic rules can be much more efficient for kernel approximation by a look-up table.

\noindent{\bf Fifth-degree rule:} When choosing $m=2$ in Eq.~\eqref{IQmdfnew}, we obtain a fifth-degree rule $Q^{(2,d)}$ with $\| \bm p \|_1 \leq 2$ to further improve the approximation quality. 
To derive the fifth-degree D-FS, we cast it to three cases, i.e., $\| \bm p \|_1 = 0$, $\| \bm p \|_1 = 1$, and $\| \bm p \|_1 = 2$.
Note that, the derivation of the fifth-degree rule is relatively technical and lengthy, so we put it in Appendix~\ref{sec:fd}.
In fifth-degree rules, the number of required nodes in D-FS is $N=1+2d^2$, which is smaller than SGQ with $1+2d^2+2d$.
Further, the feature mapping in our fifth-degree rule can be obtained in a similar way with that of the third-degree rule, and thus we omit it here.

\section{Stochastic Rules and its properties}
\label{sec:sfsir}
{The above D-FS rules are \emph{determinstic}: given $d$, the number of required nodes $N$ is fixed, e.g., $N=2d+1$ in the third-degree rule and $N=1+2d^2$ in the fifth-degree rule. 
Unlike RFF, we cannot flexibly tune $N$ to control the discrepancy between the original and
approximate kernels.
This is a common issue in deterministic rules, e.g., SGQ \cite{dao2017gaussian}.}
{To tackle this issue for kernel approximation, we randomize the above deterministic rules by combining the classical Monte-Carlo sampling and control variates techniques \cite{rubinstein1985efficiency} for the design of stochastic rules S-FS. By doing so, we can flexibly tune the dimension of the obtained feature mapping with nice statistical properties.}

\subsection{Formulation of Stochastic Rules}
We begin with the design of the third-degree stochastic rule.
{Based on Eq.~\eqref{wmatrix}, we keep the nodes unchanged to maintain the sparse transformation matrix and randomize the weights in Eq.~\eqref{fsirwa}. 
Observing that $\mathbb{E}[\sum_{i=1}^d \omega_i^2] = d$ with $\bm \omega=[\omega_1, \cdots, \omega_d]^{\!\top} \sim \mathcal{N}(\bm 0, \bm I_d)$, we define the weights in our third-degree S-FS as functions of $\bm \omega$}
\begin{equation}\label{stoweight}
	\left\{
	\begin{array}{rcl}
		\begin{split}
			& \tilde{a}_0^{(1,d)}(\bm \omega) \equiv \tilde{a}_0^{(1,d)} = 1-{\sum_{i=1}^d \omega_i^2}/{\lambda_1^2} \\
			& \tilde{a}_1^{(1,d)}(\bm \omega) \equiv \tilde{a}_1^{(1,d)} = {\sum_{i=1}^d \omega_i^2}/{(2d\lambda_1^2)} \,.
		\end{split}
	\end{array}\right.
\end{equation}
Accordingly, by randomizing the weights, the stochastic version of the third-degree D-FS $Q^{(1,d)}(f)$ in Eq.~\eqref{q1df} is given by  
\begin{equation}\label{m1df}
	M^{(1,d)}(f, \bm \omega) \!=\! \tilde{a}_0^{(1,d)} f(\bm 0) + \tilde{a}_1^{(1,d)} \! \sum_{i=1}^d  \Big[ f(\lambda_1 \bm e_i) + f(-\lambda_1 \bm e_i) \Big]\,.
\end{equation}
Besides, the third-degree stochastic rule can be extended to general degrees
\begin{equation*}
	M^{(m,d)}(f, \bm \omega) = \sum_{\bm p \in P^{(m,d)}} \tilde{a}^{(m,d)}(\bm \omega) f(\bm \lambda_{\bm p}) \,,
\end{equation*}
where the nodes $\bm \lambda_{\bm p}$ are the same as that of deterministic rules in Eq.~\eqref{IQmdfnew}, while the randomized weights $\tilde{a}_{\bm p}^{(m,d)}(\bm \omega)$ are defined as
\begin{equation*}
\tilde{a}_{\bm p}^{\!(m,d)\!}(\bm \omega) = \left\{
\begin{array}{rcl}
\begin{split}
& 1;~ \mbox{if $\|\bm p\|_1=0$ and $\|\bm u\|_1=0$} \\
& \frac{2^{-K}}{d}\!\!\!\!\sum_{\|\bm u \|_1 \leq m-\| \bm p \|_1}  \sum_{i=1}^{d} \frac{\prod_{j=0}^{u_{i}+p_{i}-1}\left(\omega_i^2 - \lambda_{j}^{2}\right)}{\prod_{j=0, \neq p_{i}}^{u_{i} + p_{i}}\left(\lambda_{p_{i}}^{2} - \lambda_{j}^{2}\right)}\,,
\end{split}
\end{array}\right.
\end{equation*}
where $K$ is the number of nonzero components in $\bm p$.
The formulation of $\tilde{a}_{\bm p}^{(m,d)}(\bm \omega)$ is based on Eq.~\eqref{weightapmd} to ensure the summation to $1$. Besides, the continued product in Eq.~\eqref{weightapmd} is substituted by the summation to provide a tighter estimate, as $\mathbb{E} (\sum_{i=1}^d \omega_i^2) \leq \mathbb{E} (\prod_{i=1}^d \omega_i^2)$.

Since typical quadrature based methods for kernel approximation, e.g., SGQ \cite{dao2017gaussian} and SSR \cite{munkhoeva2018quadrature}, adopt the third-degree rule instead of higher-degree rules, in the next we focus on the third-degree stochastic rule.
The feature mapping associated with $M^{(1,d)}(f)$ is given by
\begin{equation}\label{randmap1}
\widetilde{\Phi}(\bm x, \bm \omega) = [\sqrt{\tilde{a}_0(\bm \omega)}\phi(\bm \gamma_1^{\!\top} \bm x), \cdots, \sqrt{\tilde{a}_{2d}(\bm \omega)}\phi(\bm \gamma_{2d}^{\!\top} \bm x) ]^{\!\top}\,,
\end{equation}
where $\{ \bm \gamma \}_{i=0}^{2d}$ are given by Eq.~\eqref{fsirwa}, the randomized weights $\{ \tilde{a}_i \}_{i=0}^{2d}$ refer to Eq.~\eqref{stoweight}, and $\bm \omega \sim \mathcal{N}(\bm 0, \bm I_d)$.
Therefore, $M^{(1,d)}(f, \bm \omega)$ is a \emph{randomized} rule such that $M^{(1,d)}(f, \bm \omega) \approx \langle \widetilde{\Phi}(\bm x), \widetilde{\Phi}(\bm y) \rangle$.
Unfortunately, unlike that RFF is an unbiased estimator of the original kernel, the obtained estimator $M^{(1,d)}(f, \bm \omega)$ is biased, i.e., $\mathbb{E}_{\bm \omega} [M^{(1,d)}(f, \bm \omega)] = Q^{(1,d)}(f) \neq I_d(f)$.
Besides, albeit stochastic, the designed $M^{(1,d)}(f, \bm \omega)$ still outputs the fixed dimension of the feature mapping, i.e., $N=2d+1$. In this case, we cannot flexibly tune it for practical requirements.

To tackle the above two issues, by virtue of Monte-Carlo sampling and control variates techniques \cite{rubinstein1985efficiency}, the designed S-FS is to pursue an unbiased estimator based on the formulation of $M^{(1,d)}(f, \bm \omega)$. Besides, the dimension of the feature mapping by S-FS can be flexibly tuned.
{According to Eq.~\eqref{originte}, we have the following equality
\begin{equation*}
\begin{split}
	k(\bm x, \bm y) & = \mathbb{E}_{\bm \omega}[M^{(1,d)}(f,\bm \omega)] + \mathbb{E}_{\bm \omega} [f(\bm \omega) - M^{(1,d)}(f,\bm \omega)] \\
	& =  Q^{(1,d)}(f) + \mathbb{E}_{\bm \omega}[f(\bm \omega) - M^{(1,d)}(f,\bm \omega)] \,.
\end{split}
\end{equation*}
As a result, the Monte-Carlo sampling for $f(\bm \omega)$ in Eq.~\eqref{originte} is transformed to estimate the difference $f(\bm \omega) - M^{(1,d)}(f,\bm \omega)$.} If $M^{(1,d)}(f,\bm \omega)$ is close to $f(\bm \omega)$ in the sense that the difference has smaller variance than $f(\bm \omega)$, variance reduction can be achieved.\footnote{It is possible to design other estimators close to $f(\bm \omega)$ for variance reduction. Roughly speaking, if the estimator is closer to $f(\bm \omega)$, then more variance reduction can be achieved.} 
Formally, by defining
\begin{equation}\label{rmf}
	\begin{split}
		R_1(f, \bm \omega) = Q^{(1,d)}(f) + f(\bm \omega) - M^{(1,d)}(f,\bm \omega)\,,
	\end{split}
\end{equation}
then our third-degree S-FS is defined as $\bar{R}_1(f, \bm \omega)$ such that 
\begin{equation}\label{Rmf}
	k(\bm x, \bm y) \approx \bar{R}_1(f, \bm \omega) := \frac{1}{D} \sum_{i=1}^D R_1(f, \bm \omega_i)\,,
\end{equation}
with $\{ \bm \omega_i \}_{i=1}^D \sim \mathcal{N}(\bm 0, \bm I_d)$.
Then by defining
\begin{equation*}
	\varphi(\bm x)=1/\sqrt{D}[\phi(\bm \omega_1^{\!\top}\bm x), \cdots, \phi(\bm \omega_D^{\!\top}\bm x)]^{\!\top} \in \mathbb{R}^D \,,
\end{equation*}
the final feature mapping associated with $\bar{R}_1(f, \bm \omega)$ is given by
\begin{equation}\label{feamapf}
\widehat{\Phi}(\bm x) \!=\!\! \left[\varphi(\bm x)^{\!\top}, \left(\!\frac{\mathfrak{i}}{D}\sum_{i=1}^D\widetilde{\Phi}(\bm x, \bm \omega_i)\! \right)^{\!\top}\!, \Phi(\bm x)^{\!\top} \!\right]^{\!\top}\!\!\! \in \mathbb{R}^{D+4d+2}\,,
\end{equation}
where {the symbol $\mathfrak{i}$ is the imaginary unit}, $\{ \bm \omega_i \}_{i=1}^D \!\sim\! \mathcal{N}(\bm 0, \bm I_d)$, and the mappings $\Phi(\bm x)$ and $\widetilde{\Phi}(\bm x, \bm \omega_i)$ are given by Eqs.~\eqref{determap} and \eqref{randmap1}, respectively.
As a consequence, we have $k(\bm x, \bm y) = \mathbb{E} \langle \widehat{\Phi}(\bm x), \widehat{\Phi}(\bm y) \rangle$.\\
{\bf Remark:} We make the following remarks.\\
{1) The considered kernels in this paper are real-valued. To approximate them, we introduce the imaginary unit in the feature mapping~\eqref{feamapf} due to the difference operation, i.e., $a-b=\langle (\sqrt{a}, \mathfrak{i}\sqrt{b}), (\sqrt{a}, \mathfrak{i}\sqrt{b}) \rangle$ for any $a,b \geq 0$, but the approximated kernels still remain real-valued.}\\
2) We can control the discrepancy between the original and approximated kernels by varying $D$ in the feature mapping $\widehat{\Phi}(\bm x) \in \mathbb{R}^{D+4d+2}$. Note that, the feature mappings $\widetilde{\Phi}(\bm x, \bm \omega_i)$ and $\Phi(\bm x)$ in Eq.~\eqref{feamapf} have only $2d$ non-zero elements. The nodes are independent of the sampling process and can be pre-given by Eq.~\eqref{fsirwa}.
In this case, S-FS still achieves the same space and time complexity $\mathcal{O}(Dd)$ with RFF.\\
3) Sampling $\{ \bm \omega_i \}_{i=1}^D \sim \mathcal{N}(\bm 0, \bm I_d)$ is not limited to the standard Monte Carlo sampling. It can be extended to other advanced approaches, e.g., QMC, SSR, as alternative ways, for pursuing further variance reduction. 
Our experimental results also verify this, see Section~\ref{sec:exp:sto} for details.

\subsection{Statistical Properties}
\label{sec:stati}
This subsection elucidates that i) our third-degree S-FS is unbiased, see Theorem~\ref{theun}; 
ii) exhibits a variance reduction property in Theorem~\ref{thmvar}.

\begin{theorem}\label{theun}
	(Unbiased estimation) Our stochastic rule $\bar{R}_1(f)$ in Eq.~\eqref{Rmf} is an unbiased third-degree rule for $I_d(f)$ in Eq.~\eqref{originte}.
\end{theorem}
\begin{proof}
	Refer to Appendix~\ref{sec:theun}.
\end{proof}
\noindent{\bf Remark:} We have $Q^{(1,d)}(f) = \mathbb{E}_{\bm \omega \sim \mu}[M^{(1,d)}(f, \bm \omega)]$, and thus $M^{(1,d)}(f, \bm \omega)$ is an asymptotically unbiased estimator of $I_d(f)$.

Based on the above unbiased estimation, in the next, we derive the variance of our third-degree S-FS for Gaussian kernel approximation.
Before proceeding, we introduce some notations and definitions.
For the Gaussian kernel $k(\bm x, \bm y) = \exp\left(-{\| \bm x - \bm y\|_2^2}/{(2\sigma^2)}\right)$, we use the convenient shorthands $\bm z := (\bm x - \bm y)/\sigma$ and  $z :=\| \bm z \|_2$.
For an algorithm ${\tt A}$ sampling $\{ \bm \omega_i \}_{i=1}^D \sim \mu$, we define its expectation $\mathbb{E}[{\tt A}] := \mathbb{E}_{\bm \omega_i \sim \mu} \left[ 1/D \sum_{i=1}^D \cos(\bm \omega_i^{\!\top} \bm z) \right]$ and variance $\mathbb{V}[{\tt A}]:=  \mathbb{V}_{\bm \omega_i \sim \mu}\left[ {1}/{D} \sum_{i=1}^D \cos(\bm \omega_i^{\!\top} \bm z) \right]$.

\begin{theorem}\label{thmvar}
	(Lower variance)	For the Gaussian kernel $k(\bm x, \bm y) = \exp\left(-{\| \bm x - \bm y\|_2^2}/{(2\sigma^2)}\right)$, denoting $z :=\| \bm z \|_2$ with $\bm z := (\bm x - \bm y)/\sigma$, $Q := Q^{(1,d)}(f)$ for notational simplicity, then the variance of our third-degree S-FS~\eqref{Rmf} is
	\begin{equation}\label{varrmfthe}
	\begin{split}
	\mathbb{V}[\bar{R}_1(f,\bm \omega)] \!-\!  \mathbb{V}[\text{RFF}] &\!=\! \frac{2}{Dd} \underbrace{ \left(\! \left[\! (1\!-\!Q) \!-\! \frac{1}{2} z^2 e^{-\frac{z^2}{2}} \!\right]^2\!\! \!\!-\! \frac{1}{4} z^4 e^{-{z^2}} \! \right)}_{\triangleq h_{\text{S-FS}}(\bm z)} ,
	\end{split}
	\end{equation}
	where $\mathbb{V}[\text{RFF}] = {\left(1-e^{-z^2}\right)^2}/{(2D)}$ is given by \cite{Yu2016Orthogonal}.
	In particular, the variance reduction can be achieved by
	\begin{equation}\label{varneg}
	\begin{split}
	\mathbb{V}[\bar{R}_1(f, \bm \omega)] \!-\!  \mathbb{V}[\text{RFF}] < 0 ~~\mbox{when}~~ 1-Q < z^2 e^{-\frac{z^2}{2}} \,.
	\end{split}
	\end{equation}
\end{theorem}
\begin{proof}
	Refer to Appendix~\ref{sec:thmvar}.
\end{proof}
{\bf Remark:} The condition $1-Q < z^2 e^{-\frac{z^2}{2}}$ in Eq.~\eqref{varneg} holds for most cases with detailed discussion in Section~\ref{sec:condition}.
Even if this condition does not hold in some rare cases, there is an alternative way to make it attainable: normalizing $z:=\| \bm x - \bm y \|_2/\sigma$ to $z:=\| \bm x - \bm y \|_2/\sqrt{d \sigma^2}$ by a scaling factor $\sqrt{d}$.
This normalization strategy implies that the used Gaussian kernel admits $k(\bm x, \bm y) = \exp(-\| \bm x - \bm y \|^2_2/(d\sigma^2))$, which is quite common in practice and theory.
For example, in SSR \cite{munkhoeva2018quadrature}, the authors directly employ the formulation $k(\bm x, \bm y) = \exp(-\| \bm x - \bm y \|^2_2/d)$.
In fact, this setting is well studied in random matrix theory and high-dimensional statistics, see \cite{el2010spectrum,jacot2020kernel,liang2020just}, and accordingly the used normalization strategy depending $d$ is common and fair.

Here we compare the obtained theoretical results with other representative methods on the estimated variance reduction.

\noindent {\bf Variance of ORF} \cite{Yu2016Orthogonal} is bounded by
\begin{equation*}
\mathbb{V}[\text{ORF}] -  \mathbb{V}[\text{RFF}] \leq \frac{1}{D} \underbrace{\left( \frac{g(z)}{d} - \frac{(d-1)e^{-z^2}z^4}{2d} \right)}_{\triangleq h_{\text{ORF}}(z)} \,,
\end{equation*}
where the function $g$ is
$g(z) = {e^{z^{2}}\left(z^{8}+6 z^{6}+7 z^{4}+z^{2}\right)}/{4}$ $+ {e^{z^{2}} z^{4}\left(z^{6}+2 z^{4}\right)}/{(2 d)}$, at an exponential growth of $z$.

\noindent {\bf Variance of SSR} \cite{munkhoeva2018quadrature} {is bounded by 
\begin{equation}\label{varssr}
\mathbb{V}[\text{SSR}] -  \mathbb{V}[\text{RFF}]  \! \leq \! \frac{1}{D} \underbrace{\left(\! \frac{8d\!+\!12}{d-2} - \frac{(1\!-\!e^{-z^2})^2}{2} \! \right)}_{\triangleq h_{\text{SSR}}(z) > 0} \,,
\end{equation}
with the positive $h_{\text{SSR}}(z)$ satisfying $\lim_{z \rightarrow \infty} h_{\text{SSR}}(z) = 8$.
}

\begin{figure} [t]
	\centering 
	\subfigure[$d=10$]{\label{orfssr}
		\includegraphics[width=0.222\textwidth]{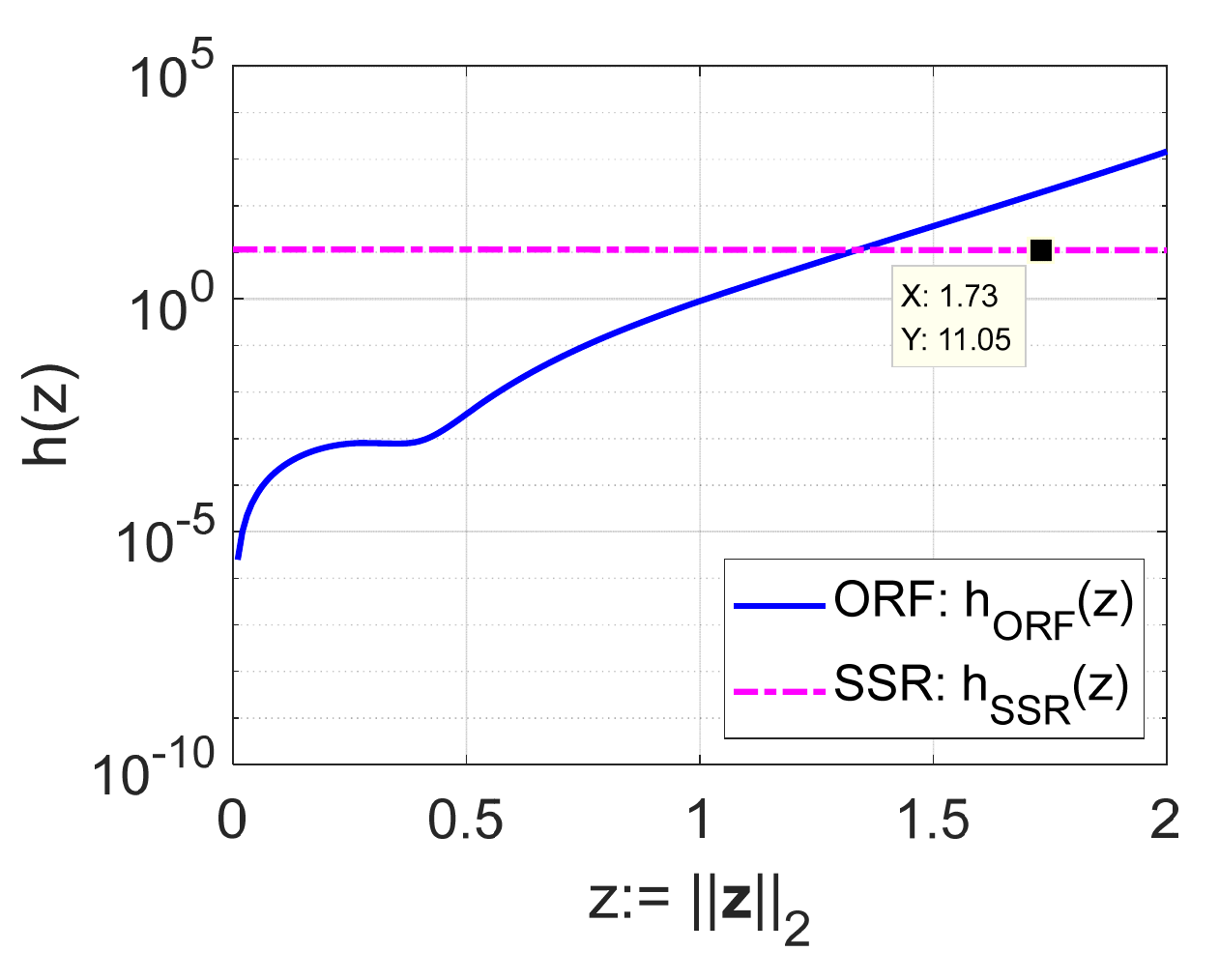}}
	\subfigure[$d=1$]{\label{oursvar}
		\includegraphics[width=0.222\textwidth]{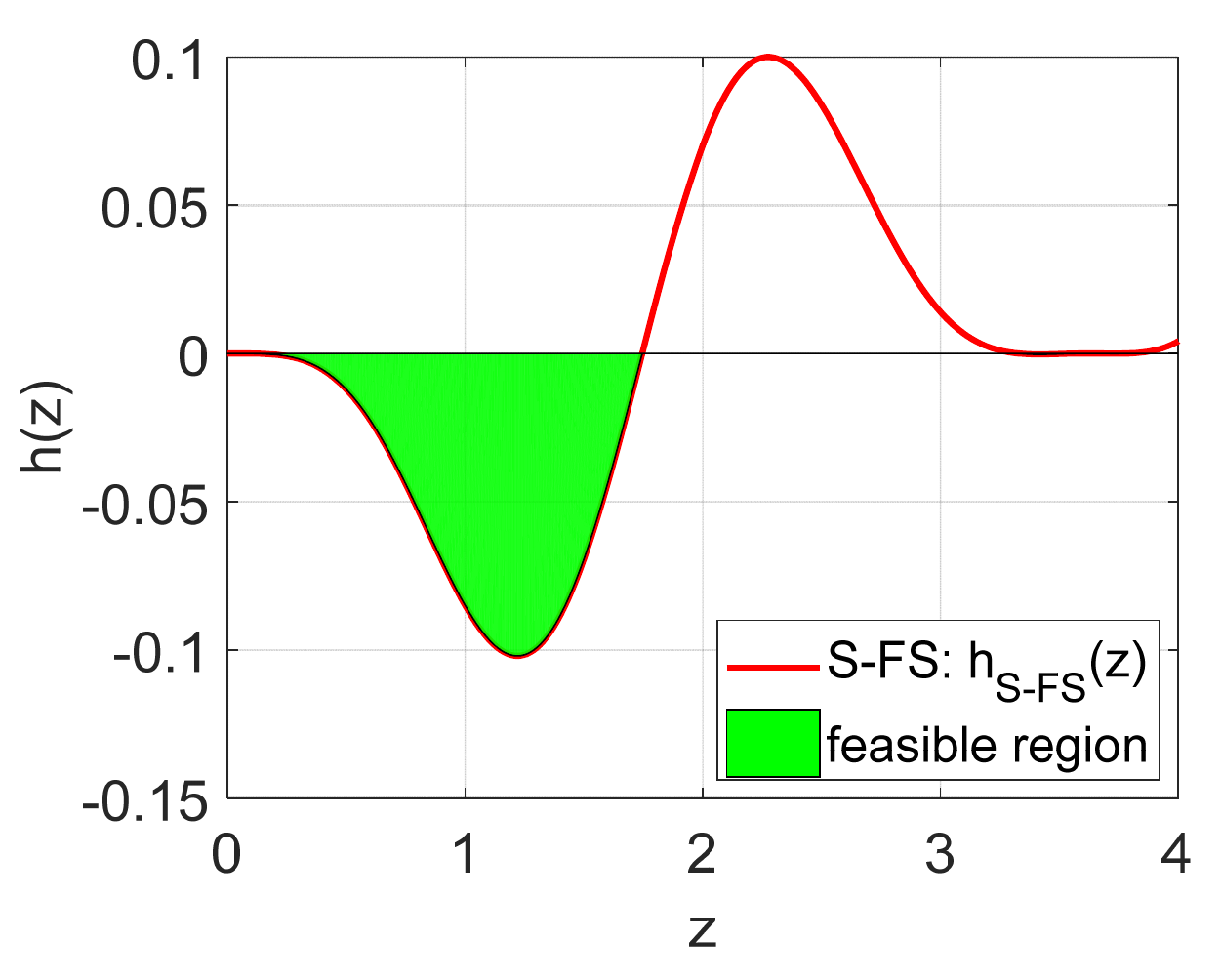}}
	\caption{Comparison of $h(z)$ versus the distance $z:= \| \bm z \|_2$ across ORF, SSR (a) and S-FS (b). Since $h_{\text{S-FS}}(\bm z)$ is no loner a radial function of $\bm z$ due to $Q$ depending on $\bm z$, we just present the univariate case of $h_{\text{S-FS}}(z)$ for intuitive display.}\label{orfssroursvar}
\end{figure}%

For better illustration, we plot the function $h(z)$ including $h_{\text{ORF}}$, $h_{\text{SSR}}$, $h_{\text{S-FS}}$ versus the distance $z:=\| \bm z \|_2$ in Figure~\ref{orfssroursvar} for intuitive explanation.
In our simulation, we set $d=10$ as an example.
It can be found that, 1) $h_{\text{ORF}}$ is positive and thus the variance reduction cannot be demonstrated in theory. More specifically, $h_{\text{ORF}}$ almost increases at an exponential order of $z$, which leads to a quite loose bound for variance estimation. 
2) SSR cannot strictly guarantee $\mathbb{V}[\text{SSR}] < \mathbb{V}[\text{RFF}]$ due to $h_{\text{SSR}}(z) > 0$ in Eq.~\eqref{varssr}.
Instead, our theoretical result in Theorem~\ref{thmvar} admits $\mathbb{V}[\bar{R}_1(f, \bm \omega)] <  \mathbb{V}[\text{RFF}]$ under the condition in Eq.~\eqref{varneg} for variance reduction, as demonstrated by Figure~\ref{oursvar} in the univariate case. Even if this condition does not hold, we still have $\mathbb{V}[\bar{R}_1(f, \bm \omega)] \!-\!  \mathbb{V}[\text{RFF}] \leq {2}/{(Dd)}$ at a certain $\mathcal{O}\left(1/(Dd) \right)$ rate as $h_{\text{S-FS}}(z)$ is bounded.
This is faster than SSR converging at a certain $\mathcal{O}(1/D)$ rate.

\subsection{Discussion on the Condition~\eqref{varneg} in Theorem~\ref{thmvar}}
\label{sec:condition}

{Here we verify that the condition~\eqref{varneg} for $\mathbb{V}[R_1(f, \bm \omega)] -  \mathbb{V}[\text{RFF}] < 0$ in Theorem~\ref{thmvar} holds for most cases.
The description of the used four datasets (\emph{magic04}, \emph{letter}, \emph{ijcnn1}, \emph{covtype}) for numerical validation is deferred to our experiments in Section~\ref{sec:exp}.

Under the Gaussian kernel setting, recall our third-degree D-FS \eqref{q1df}, the condition~\eqref{varneg} is equivalent to
\begin{equation}\label{jzfunc}
\frac{d}{3} - \frac{1}{3} \sum_{i=1}^d \cos(\sqrt{3} \bm e_i^{\!\top} \bm z) - \| \bm z\|_2^2 \exp(-\| \bm z \|_2^2/2) < 0\,.
\end{equation}
For notational simplicity, we denote the left-hand side of the above inequality as $J(\bm z)$.
In the next, we first study the existence of solutions to Eq.~\eqref{jzfunc} and then numerically validate that the condition under these solutions holds for most cases. 

\begin{table}[t]
	\centering
	\caption{The maximum radius of the hyper-ball $\mathcal{S}^d(r)$ under various $d$.}
	\begin{threeparttable}
		\begin{tabular}{cccccccccccccc}
			\toprule
			$d$ & $r_{\max}$ & $d$ & $r_{\max}$  \\
			\midrule
			$10$ (\emph{magic04})  &$1.208$ & $50$ & $1.1837$  \\
			\hline
			$16$ (\emph{letter}) & $1.1964$ & $54$ (\emph{covtype}) & $1.1831$ \\
			\hline
			$20$ & $1.1896$ & $100$ & $1.18$ \\
			\hline
			$22$ (\emph{ijcnn1}) & $1.1909$  & $200$ & $1.1787$
			\\
			\bottomrule
		\end{tabular}
	\end{threeparttable}\label{tabcondition}
\end{table}

\begin{figure}[t]
	\centering
	\subfigure[\emph{magic04}]{
		\includegraphics[width=0.225\textwidth]{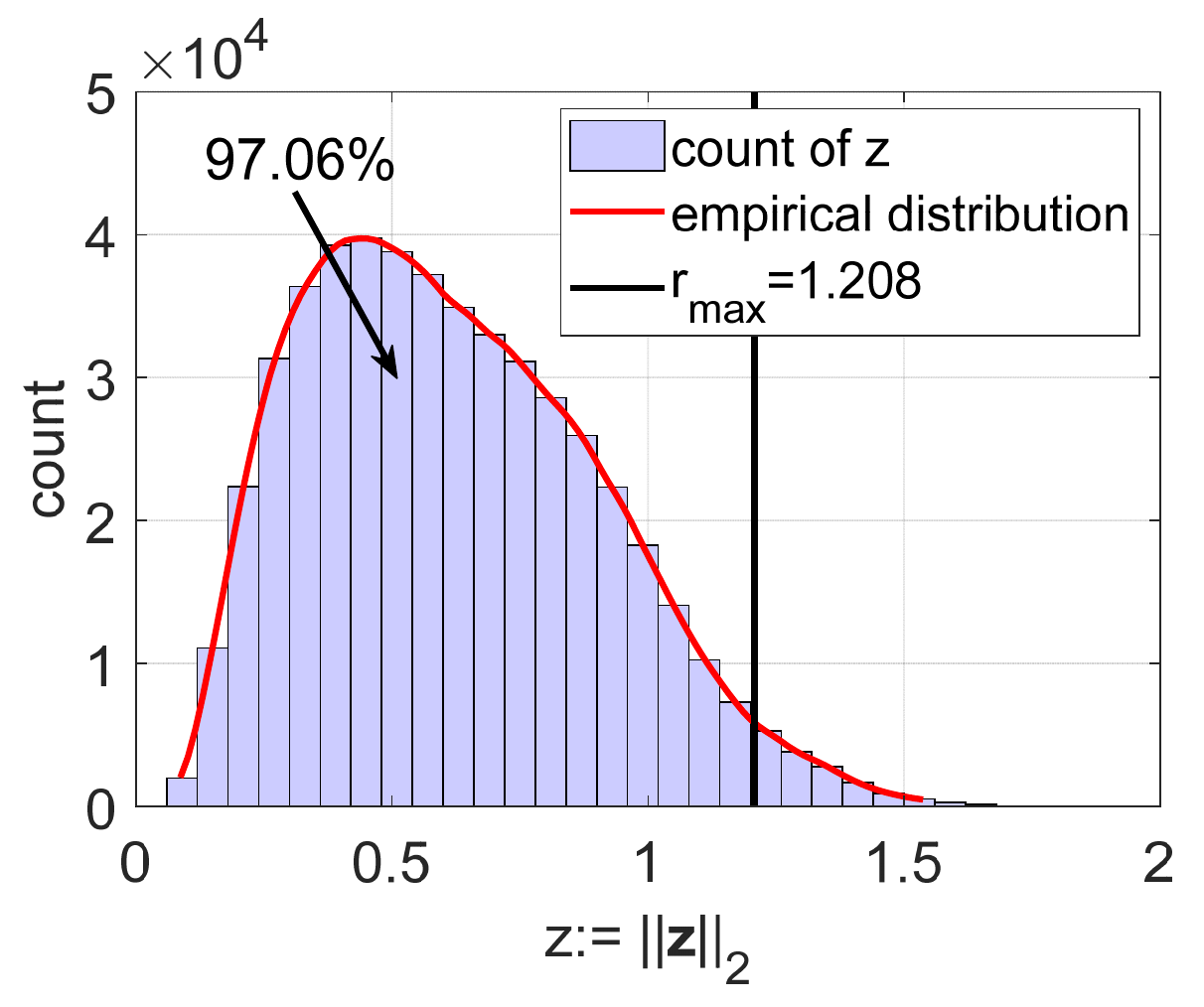}}
	\subfigure[\emph{letter}]{
		\hspace{0.01cm}
		\includegraphics[width=0.225\textwidth]{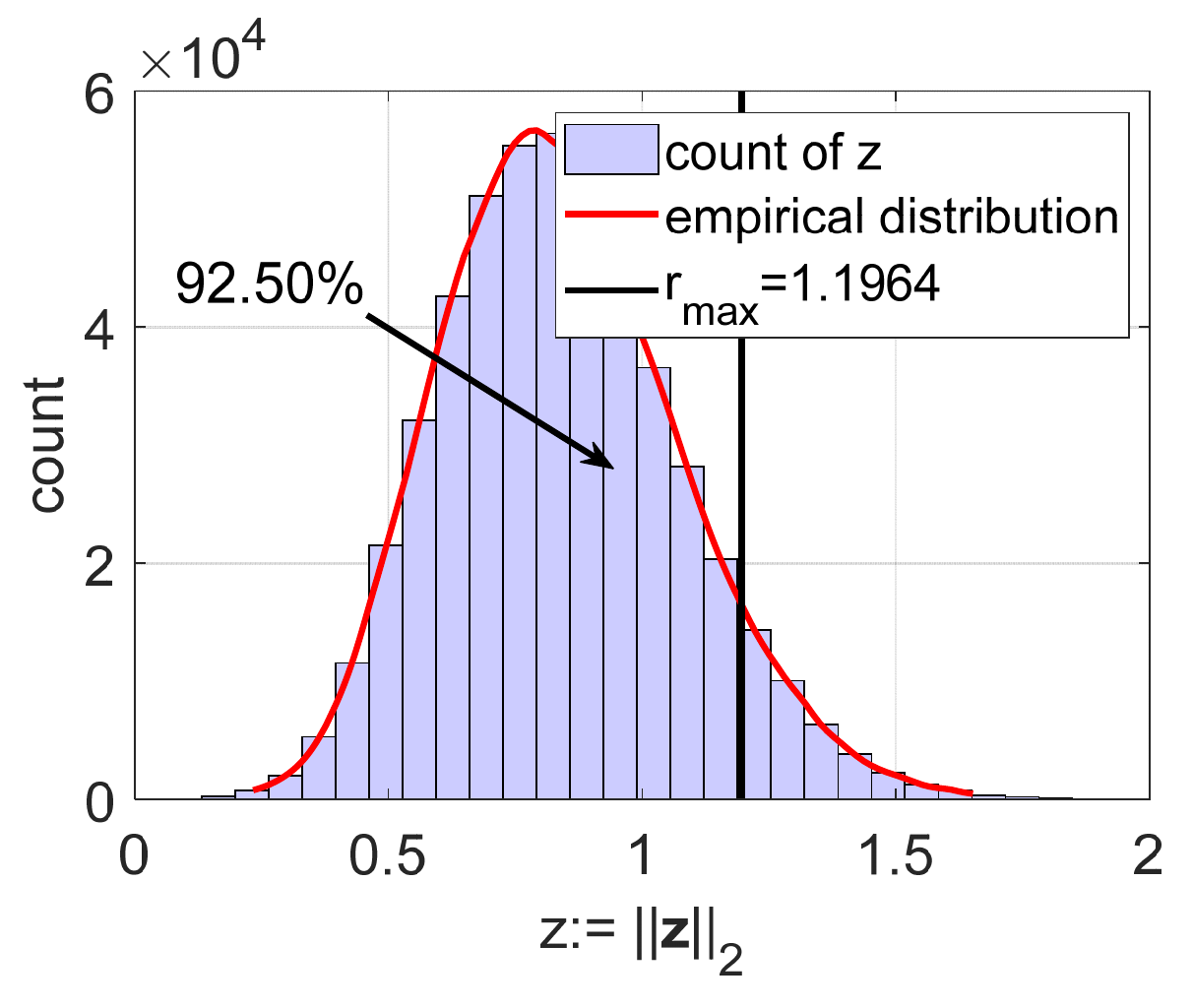}}
	\subfigure[\emph{ijcnn1}]{
		\hspace{0.01cm}
		\includegraphics[width=0.225\textwidth]{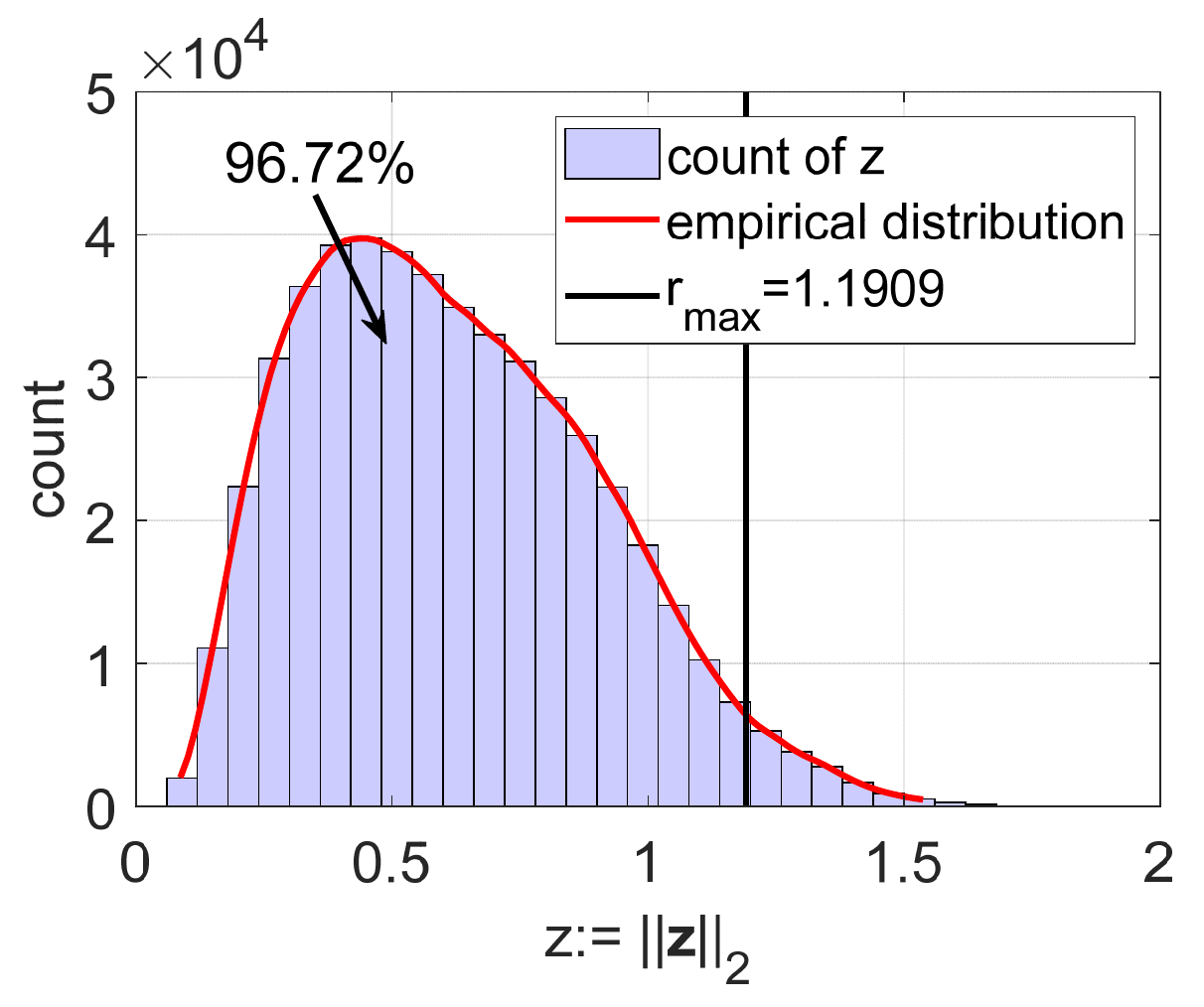}}
	\subfigure[\emph{covtype}]{
		\hspace{0.01cm}
		\includegraphics[width=0.225\textwidth]{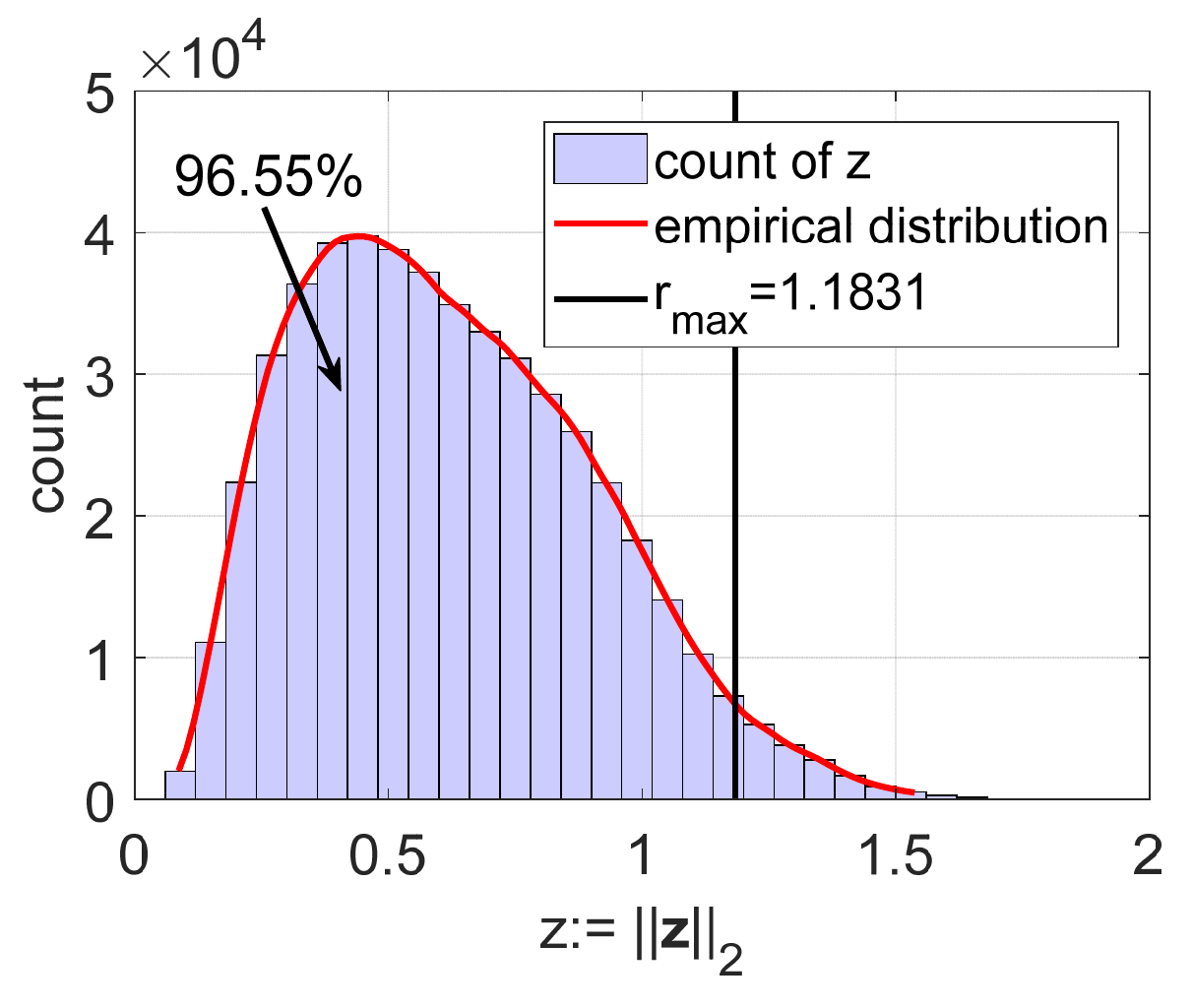}}
	\caption{Empirical distribution of $z$ in four datasets used in this paper.}\label{appcon}
	\vspace{-0.05cm}
\end{figure}

\subsubsection{Existence}
We consider a simple case: finding a $d$-dimensional Euclidean ball $\mathcal{S}^d(r)=\{ \bm z \in \mathbb{R}^d: \| \bm z \|_2 \leq r \}$ as the feasible region, such that all points in $\mathcal{S}^d(r)$ admit $J(\bm z) < 0$.
As a result, our target is transformed to maximize $r$ by solving a one-dimensional optimization problem
\begin{equation*}
	\max~r~\,,~\mbox{s.t.}~~
	\frac{1}{3} - \frac{1}{3} \cos\left(\frac{\sqrt{3}r}{\sqrt{d}}\right) - \frac{r^2}{d} \exp\left(\frac{r^2}{2} \right) < 0\,.
\end{equation*}
After numerical calculation, the maximum radius $r_{\max}$ under different $d$ is reported in Table~\ref{tabcondition}.
That means, given $d$, there exists a hyper-ball $\mathcal{S}^d(r_{\max})$ such that any vector $\bm z \in \mathbb{R}^d$ with $z := \| \bm z \|_2 \leq r_{\max}$ admits the condition~\eqref{varneg}.

\subsubsection{Numerical validation}

Here we numerically validate that the obtained $r_{\max}$ in Table~\ref{tabcondition} holds for most cases in four datasets used in this paper.

In our numerical simulation, on each dataset, we randomly select 1,000 data points $\{ \bm x_i \}_{i=1}^{1000}$ to compute the distance $z_{ij} = \| \bm x_i - \bm x_j \|_2$ with $1\leq i,j \leq 1000$, and then construct a histogram with 30 bins for counting $z_{ij}$.\footnote{The diagonal elements $z_{ii}=0$ are not counted and non-diagonal elements are counted only once.}
Figure~\ref{appcon} shows the histogram for counting $z_{ij}$ and the fitted empirical distribution of $z$ on these four datasets.
Observe that all the datasets admit $z_{ij}<2$, which shows the consistency with \cite{Yu2016Orthogonal} (see Figure 2(c) in their paper).
Further, we also plot $r_{\max}$ in Table~\ref{tabcondition} on each dataset (see the black line) in Figure~\ref{appcon}, and find that, over 90\% of $\{ z_{ij} \}_{i,j=1}^{1000}$ satisfy $z:= \| \bm z \|_2 \leq r_{\max}$. That means, our condition~\eqref{varneg} holds for most cases and thus is fair and attainable.
}

\section{Unifying Framework for Quadrature Methods}
\label{sec:relation}
In this section, we investigate the relations among third-degree rules, including SGQ \cite{dao2017gaussian}, SSR \cite{munkhoeva2018quadrature}, and our deterministic/stochastic rules, i.e., D-FS/S-FS.
Subsequently, we cast them in our unifying framework for kernel approximation.

\subsection{Relations to SGQ}
The sparse grids used in \cite{dao2017gaussian} are based on the Smolyak rule \cite{heiss2008likelihood} which can be approximated by a sequence of nested univariate quadrature rules in a tensor product fashion
\begin{equation}\label{sgq}
I_d(f) \!\approx \! A_{d, L}(f)\!=\!\sum_{q=0}^{L-1} \sum_{\bm{i} \in \mathcal{C}_{q}^{d}} \! \left(\Delta_{i_{1}} \otimes \cdots \otimes \Delta_{i_{d}} \right) \!(f) \,,
\end{equation}
with the index vector $\bm i = [i_1, i_2, \cdots, i_d]$.
The set $\mathcal{C}_{q}^{d}=\left\{\bm{i} \in \mathbb{N}^{d}: \sum_{j=1}^{d} i_{j}=d+q\right\}$ determines the possible accuracy level $i_j$ for each univariate quadrature and the nonnegative $q$ prescribes the range of the accuracy level $i_j$ in each dimension.
$V_{i_j}$ is the univariate quadrature rule with the accuracy level $i_j \in \bm i$, which generates the difference
$\Delta_i(f) = V_i(f) - V_{i-1}(f)$, $\forall i \in \mathbb{N}$.
This rule is a weighted sum of product rules with different combinations of accuracy levels $\bm i$.

To study the relationship between SGQ and D-FS, we construct the third-degree SGQ in Eq.~\eqref{sgq} using the symmetric univariate quadrature point set $\left\{-\hat{p}_{1}, 0, \hat{p}_{1}\right\}$ and the weights $(\hat{a}_1,\hat{a}_0,\hat{a}_1)$, then the integration $I_d(f)$ can be approximated by SGQ
\begin{equation*}\label{sgqthird}
I_{d}(f)\! \approx\! \left(1\!-\!d\!+\!d \hat{a}_{0}\right) f(\bm{0})+\hat{a}_{1} \sum_{j=1}^{d}\!\big[f\left(\hat{p}_{1} \bm{e}_{j}\right)\!+\!f\left(-\hat{p}_{1} \bm{e}_{j}\right)\!\big]\,.
\end{equation*}
If the nodes and their associated weights are chosen by the following scheme
\begin{equation*}
\hat{a}_{0} := 1 - \frac{1}{\lambda_1^2},~\hat{p}_1 := \lambda_1 ,~  \hat{a}_{1} = \frac{1}{2\lambda_1^2} \,,
\end{equation*}
then the third-degree SGQ is equivalent to D-FS in Eq.~\eqref{q1df}, as shown in Figure~\ref{framework}.

\subsection{Relations to SSR}
\label{sec:relationssr}
The key step in SSR \cite{genz1998stochastic} is a change of variable from $\bm \omega \in \mathbb{R}^d$ to a radius $r$ and direction vector $\bm a \in \mathbb{R}^d$.
Let $\bm \omega = r \bm a$ with $\bm a^{\!\top} \bm a = 1$ and $r \in [0,\infty)$, we have
\begin{equation*}\label{srq}
\begin{split}
I_d(f) 
& = \frac{(2 \pi)^{-\frac{d}{2}}}{2} \int_{U_{d}} \int_{-\infty}^{\infty} |r|^{d-1} e^{-\frac{r^{2}}{2}} f(r \bm a) \mathrm{d} \tau(\bm a) \mathrm{d} r \\
& \approx \! f(\bm{0})\! \left(1 \!-\! \frac{d}{\rho^{2}}\right) + \sum_{j=1}^{d} \frac{f\!\left(-\rho \bm Q \bm{e}_{j}\right)+f\left(\rho \bm Q \bm{e}_{j}\right)}{2 \rho^{2}} \,,
\end{split}
\end{equation*}
where $\bm Q$ is a random orthogonal matrix, $\tau(\cdot)$ is the spherical surface measure or the area element on $U_d$, and $\rho \sim \chi(d+2)$.
SSR includes the following two stochastic integration rules: one is stochastic radial rule for approximating the infinite range integral  $\int_{-\infty}^{\infty}  e^{-\frac{r^{2}}{2}} |r|^{d-1} f(r) \mathrm{d} r $; the other is stochastic spherical rule for a surface integral over $U_d$
\begin{equation}\label{ssr3}
	I_{\bm Q, U_{d}}(f)=\frac{|U_{d}|}{2 d} \sum_{j=1}^{d}\left[f\left(\bm Q \bm{e}_{j}\right)+f\left(-\bm Q \bm{e}_{j}\right)\right] \,,
\end{equation}
where $|U_{d}| = 2 \sqrt{\pi^d} / \Gamma(d/2)$ is the surface area of the unit sphere with the Gamma function $\Gamma$.

\begin{table}
	\centering
	\caption{Relationship between typical kernel approximation methods.}
	\begin{threeparttable}
		\begin{tabular}{cccccccccccccc}
			\toprule
			Methods &Parameters in Eq.~\eqref{tris}  \\
			\midrule
			SSR  &$\beta:=d$ \\
			\hline
			$M^{(1,d)}(f)$ in Eq.~\eqref{m1df} & $\rho:=\lambda_1^2$ and $\bm Q:= \bm I$
			\\
			\hline
			ORF & $\beta:=d$ and $\rho \sim \chi(d)$
			\\
			\hline
			$Q^{(1,d)}(f)$ in Eq.~\eqref{q1df} & $\rho:=\lambda_1^2$, $\bm Q:= \bm I$, and $\beta:=d$
			\\
			\hline
			\multirow{2}{1cm}{SGQ} & \multirow{1}{4cm}{$\rho:=\lambda_1^2$, $\bm Q:= \bm I$, $\beta:=d$} \\
			&$\{ \hat{a}_0, \hat{p}_0, \hat{a}_1 \} \leftarrow \lambda_1 $ \\
			\bottomrule
		\end{tabular}
	\end{threeparttable}\label{tabrelation}
\end{table}

Here we present the following theorem that states the relationship between the third-degree stochastic spherical rule and the third-degree D-FS. 
\begin{theorem}\label{thirdproj}
	The third-degree stochastic spherical integration rule~\eqref{ssr3} can be obtained by the random orthogonal projection of D-FS in Eq.~\eqref{q1df}.
\end{theorem}

\begin{proof}
	Refer to Appendix~\ref{sec:proofthirdproj}.
\end{proof}
Accordingly, SSR can be obtained by D-FS in Eq.~\eqref{q1df} with the following two randomized steps.
1) random projection: according to Theorem~\ref{thirdproj}, by projecting D-FS to the spherical surface of $U_d$ with a uniform random orthogonal matrix $\bm Q$, we can obtain the third-degree stochastic spherical rule.
2) random generator: the deterministic \emph{generator} $\lambda_1$ in Eq.~\eqref{q1df} by Gaussian quadrature is substituted by a random variable $\rho$ with $\rho \sim \chi(d+2)$.
By doing so, we can transform D-FS to SSR, as shown in Figure~\ref{framework}.

\subsection{Unifying Framework}
\label{sec:unfiframework}
Apart from the relations between our deterministic rule and SSR, here we also study the relationship between S-FS and SSR.
On the one hand, in Eq.~\eqref{rmf}, if we only consider $R_1(f, \bm \omega) := f(\bm \omega)$, S-FS degenerates to RFF with the standard Monte-Carlo sampling scheme.
On the other hand, RFF (also ORF) can be regarded as spacial cases of SSR as demonstrated by
\cite{munkhoeva2018quadrature}.
Furthermore, if we only consider $R_1(f, \bm \omega) := M^{(1,d)}(f, \bm \omega)$ in Eq.~\eqref{rmf}, after the above two stochastic operations (random projection and random \emph{generator}), it is a triple-stochastic rule with the following formulation
\begin{equation}\label{tris}
I_d(f) \! \approx \! f(\bm{0})\! \left(1 \!-\! \frac{\beta}{\rho^{2}}\right) + \frac{\beta}{d}\sum_{j=1}^{d} \frac{f\!\left(-\rho \bm Q \bm{e}_{j}\right)+f\left(\rho \bm Q \bm{e}_{j}\right)}{2 \rho^{2}} \,,
\end{equation}
with $\beta \sim \chi(d)$ and $\rho \sim \chi(d+2)$.
Clearly, this rule is also an unbiased estimator of $I_d(f)$. 
Finally, we summarize the relations between D-FS/S-FS, SGQ, SSR, ORF under the unifying framework in Table~\ref{tabrelation}.

\begin{table}[t] 
		\centering
		\caption{Dataset statistics and the number of \emph{nodes} in fifth-degree rules.} 
		\label{tablarge}
		\begin{tabular}{cccc|cccccccccc}
			\toprule
			\multirow{2}{1.5cm}{\centering{datasets}}& \multirow{2}{1cm}{\centering{$d$}} &\multirow{2}{1cm}{\centering{\#training}} &\multirow{2}{1cm}{\centering{\#test}} &\multicolumn{2}{c}{\#nodes $N$} \cr
			\cmidrule(lr){5-6}
			& & & & SGQ & Ours  \\
			\midrule
			\emph{magic04} &10 &9,510 &9,510 & 221 & 201  \\
			\hline
			\emph{letter} &16 &12,000 &6,000 & 545 & 513 \\
			\hline
			\emph{ijcnn1} &22 &49,990 &91,701 & 1013 & 969 \\
			\hline
			\emph{covtype} &54 &290,506 &290,506 & 5941 & 5833 \\
			\bottomrule
		\end{tabular}
\end{table}

\section{Empirical Results}
\label{sec:exp}

In this section, we empirically compare our deterministic/stochastic rules, D-FS and S-FS, with several representative approaches for kernel approximation, and then incorporate them into the kernel ridge regression (KRR) for classification on several benchmark datasets.
Given nodes and weights in Eq.~\eqref{fsirwa}, our algorithm is straightforward to be implemented for the feature mapping in Eq.~\eqref{determap} by our deterministic rule and Eq.~\eqref{feamapf} by our stochastic rule.
We implement them in MATLAB and carry out on a PC with Intel$^\circledR$ i7-8700K CPU (3.70 GHz) and 64 GB RAM.
The source code of our implementation can be found in \url{http://www.lfhsgre.org}.

\begin{figure*}[!htb]
	\centering
	
	\subfigure{
		\includegraphics[width=0.22\textwidth]{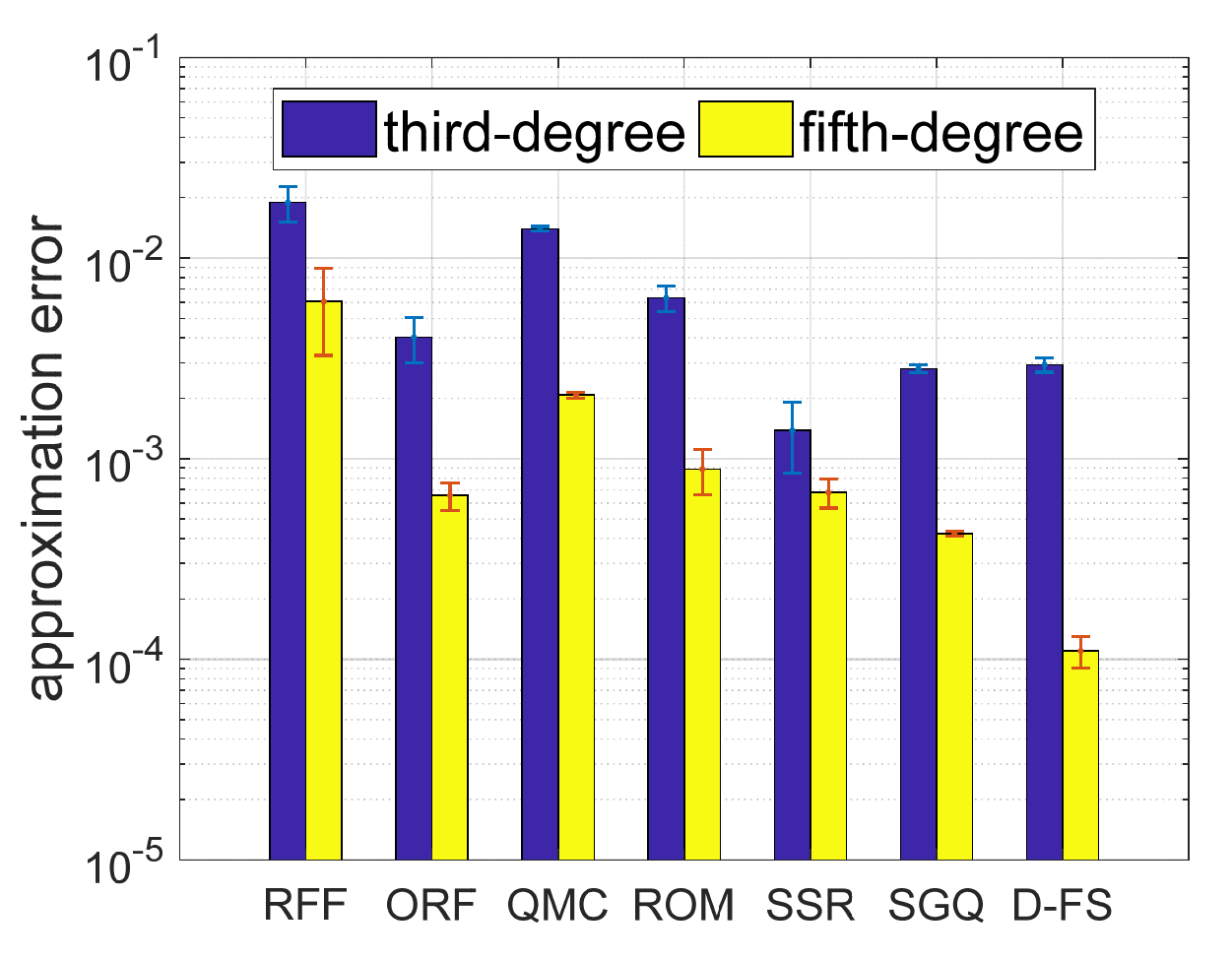}}
	\subfigure{
		\includegraphics[width=0.22\textwidth]{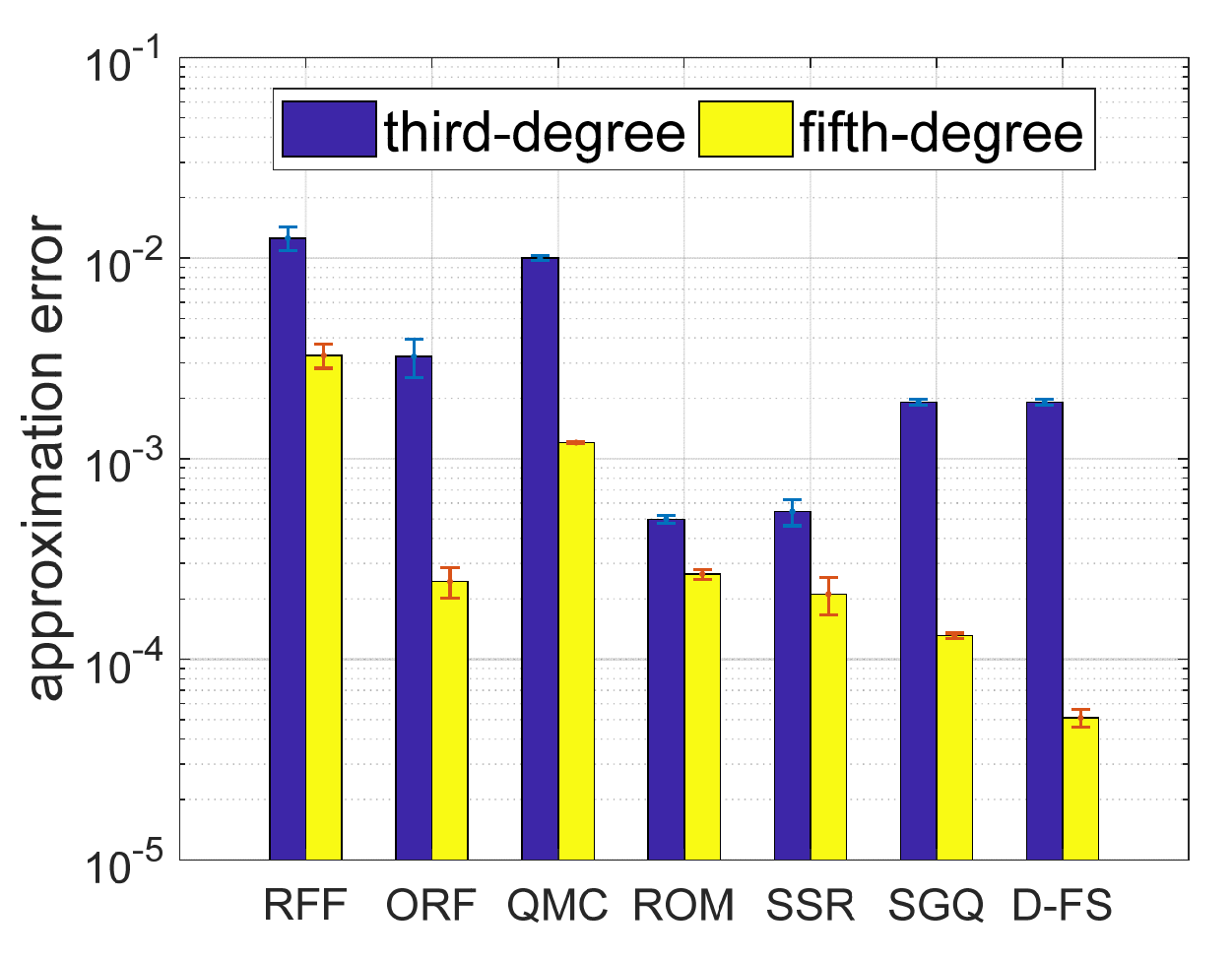}}
	\subfigure{
		\includegraphics[width=0.22\textwidth]{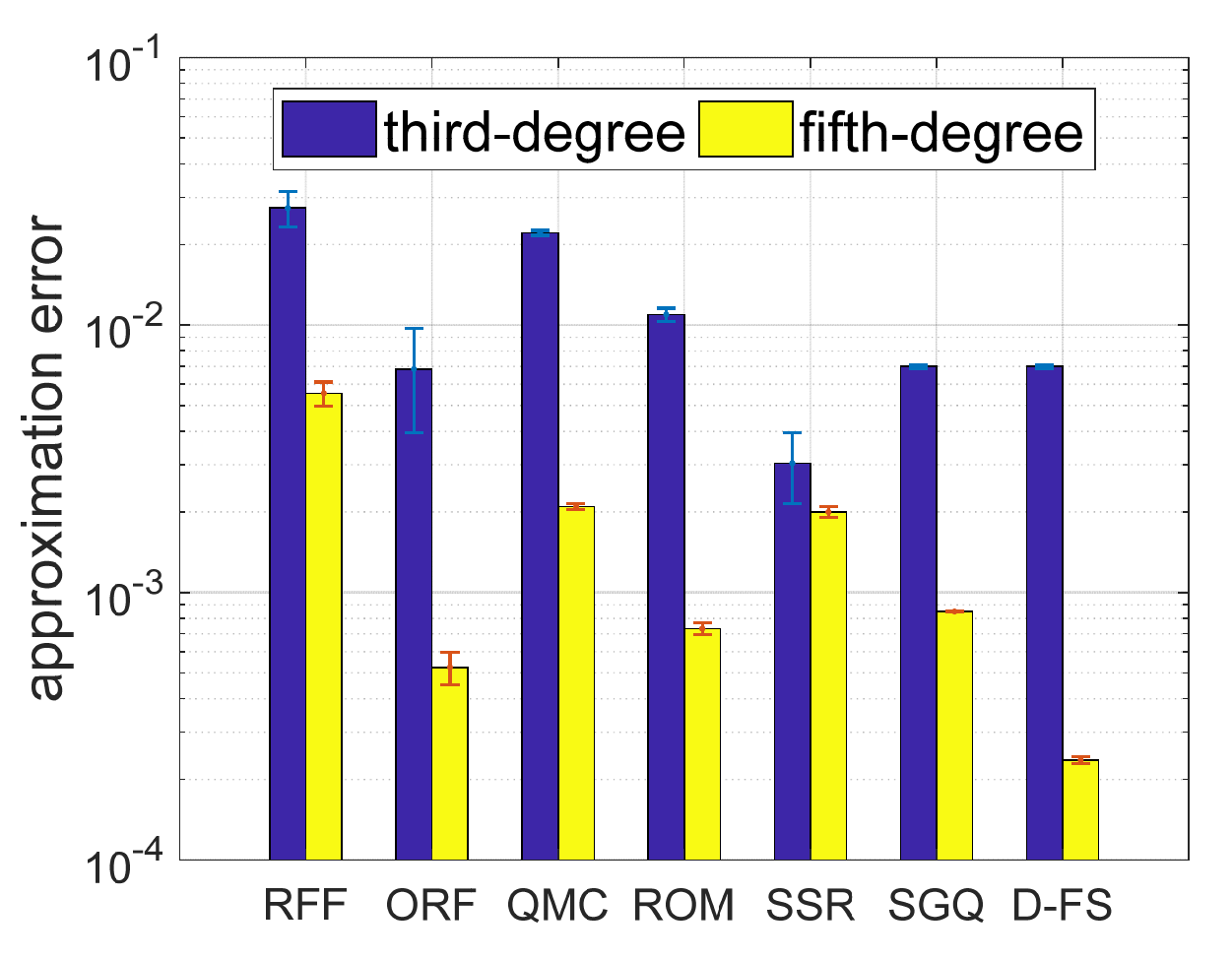}}
	\subfigure{
		\includegraphics[width=0.22\textwidth]{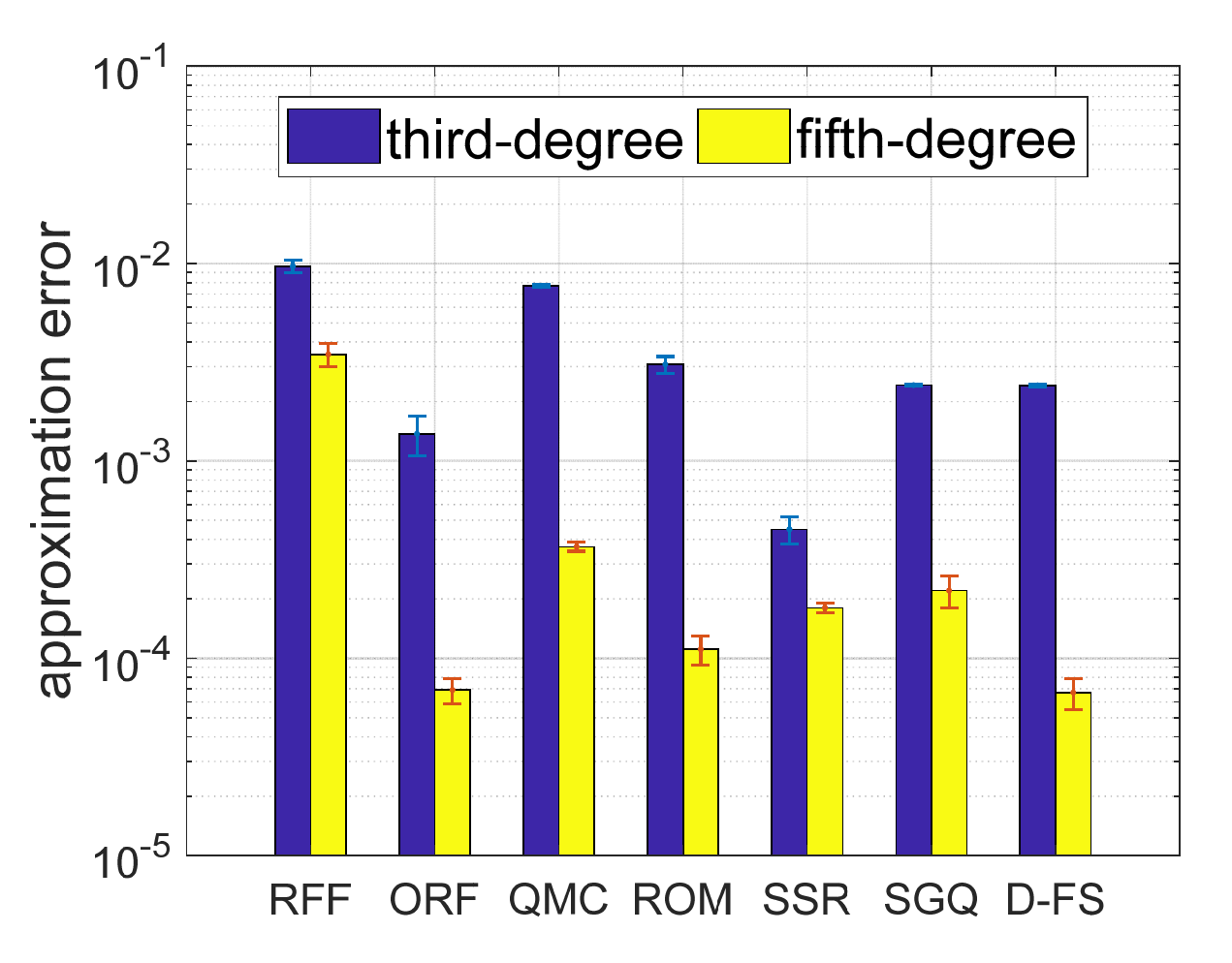}}
	
	\subfigure{
		\includegraphics[width=0.22\textwidth]{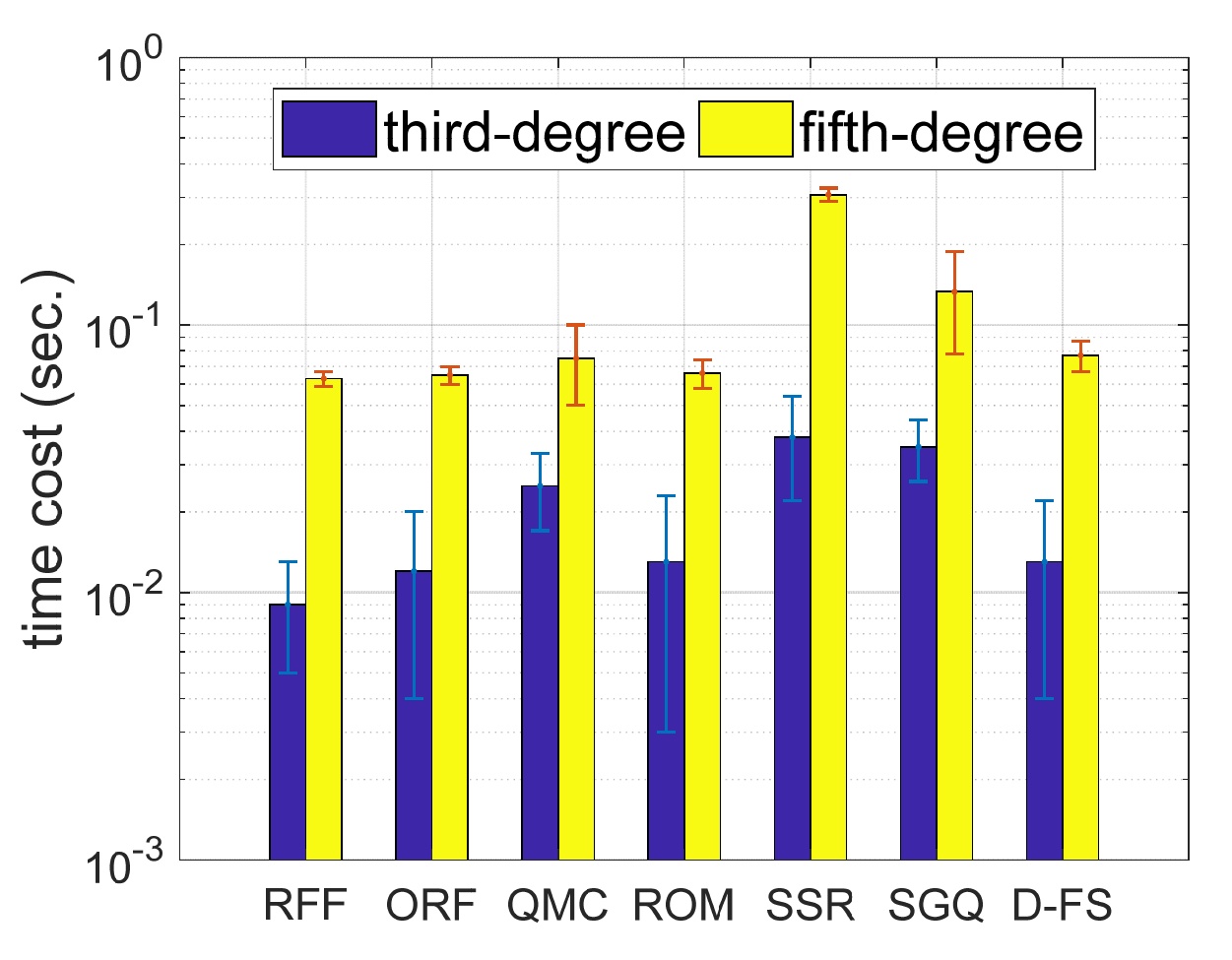}}
	\subfigure{
		\includegraphics[width=0.22\textwidth]{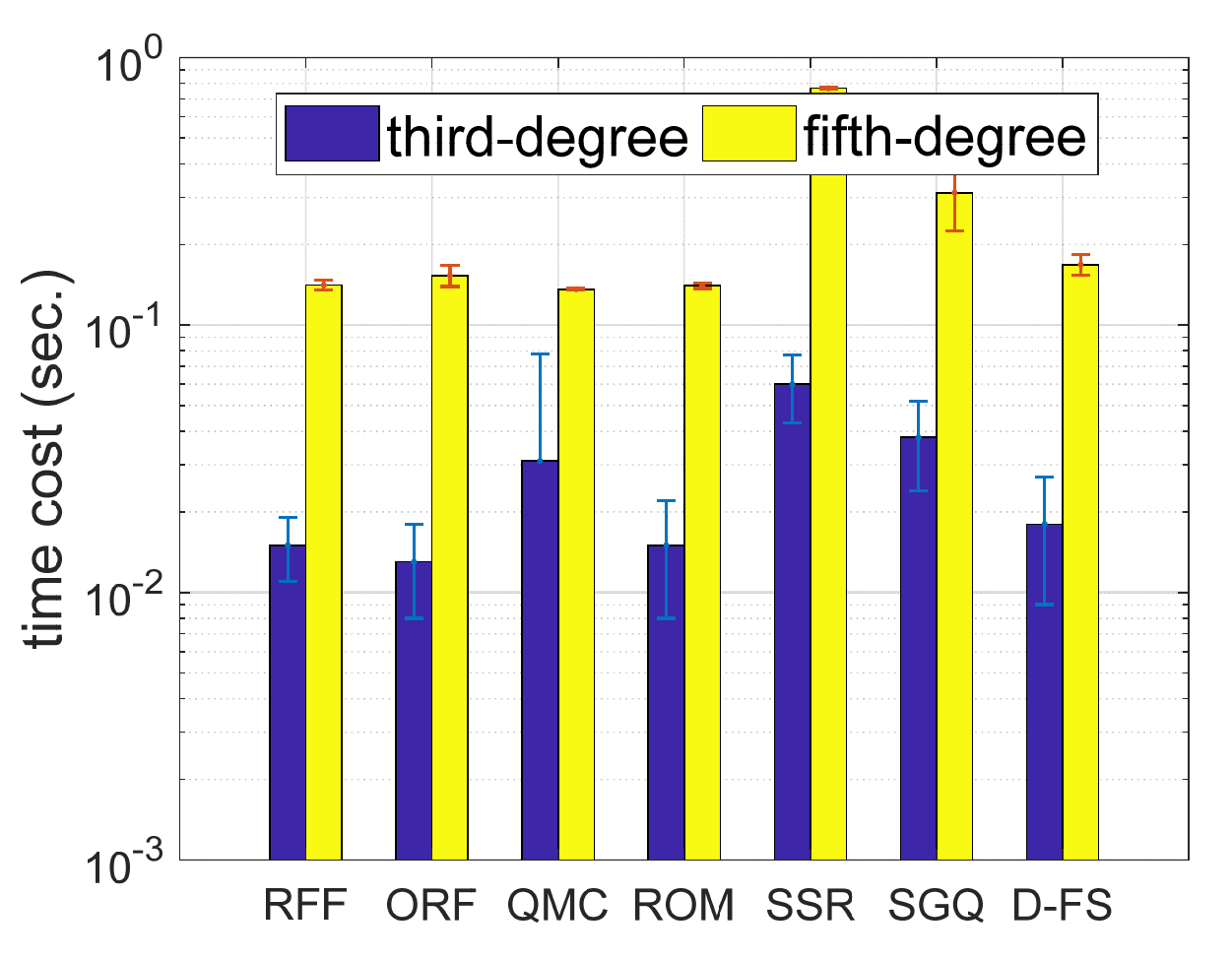}}
	\subfigure{
		\includegraphics[width=0.22\textwidth]{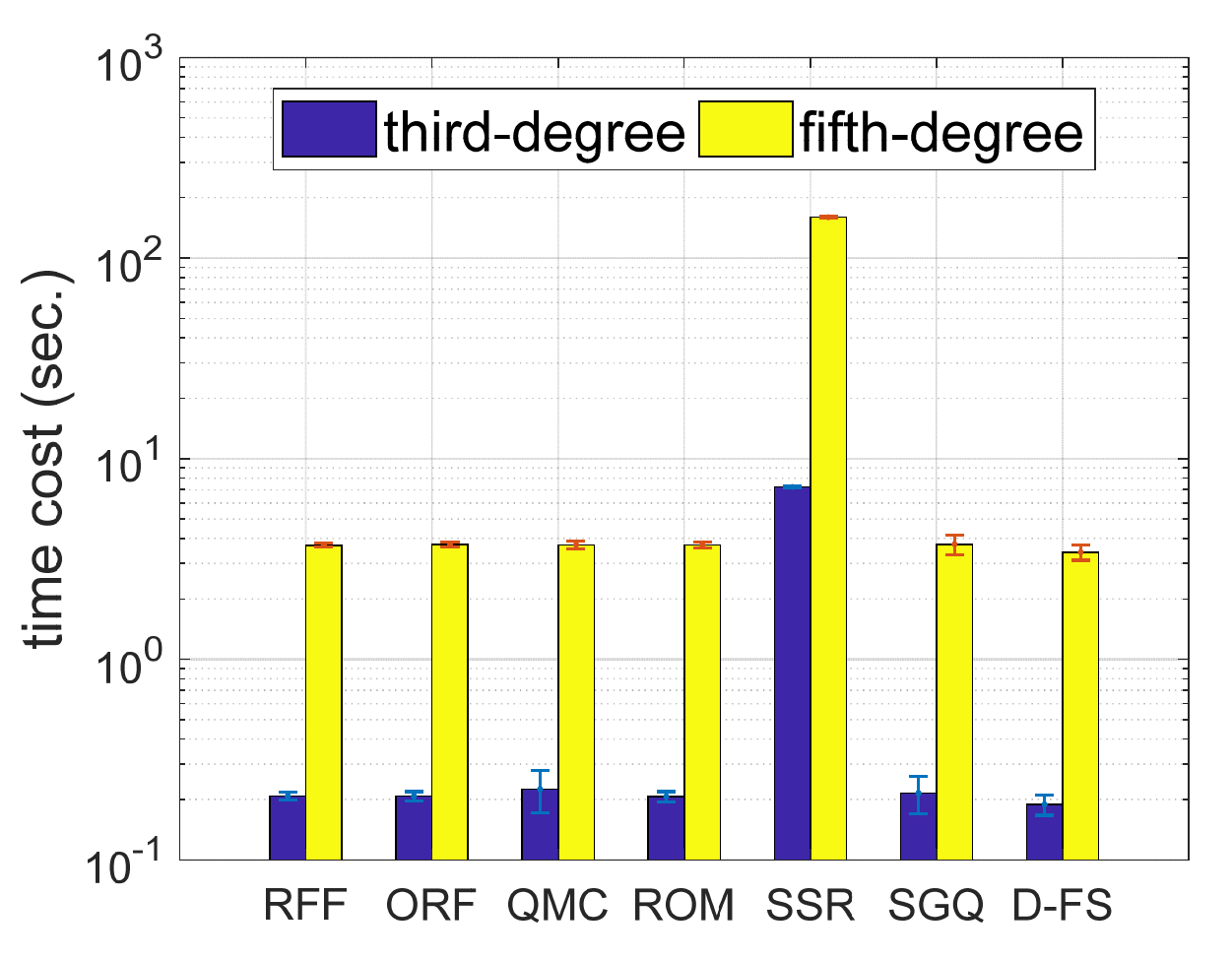}}
	\subfigure{
		\includegraphics[width=0.22\textwidth]{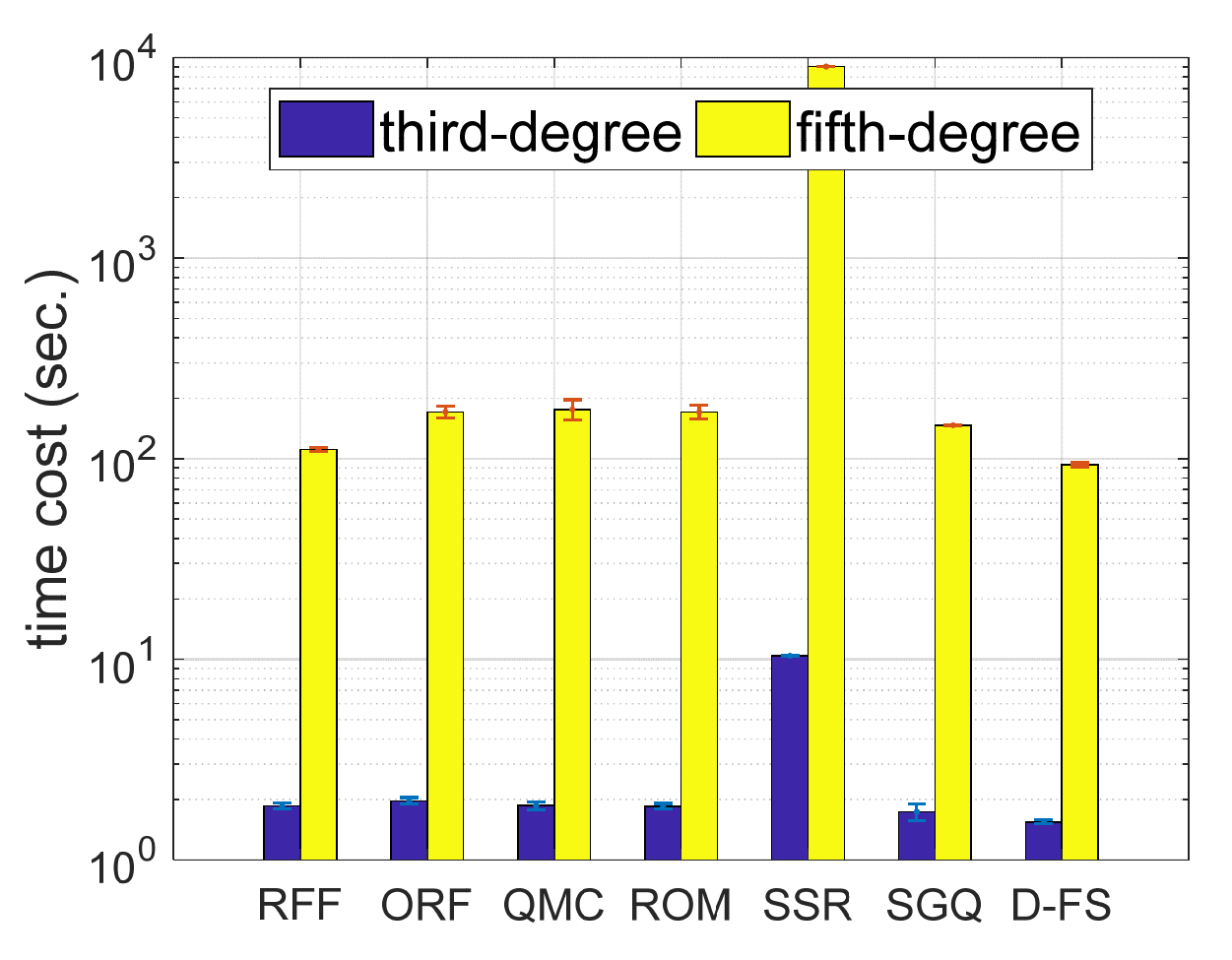}}
	
	\subfigure[\emph{magic04}]{
		\includegraphics[width=0.22\textwidth]{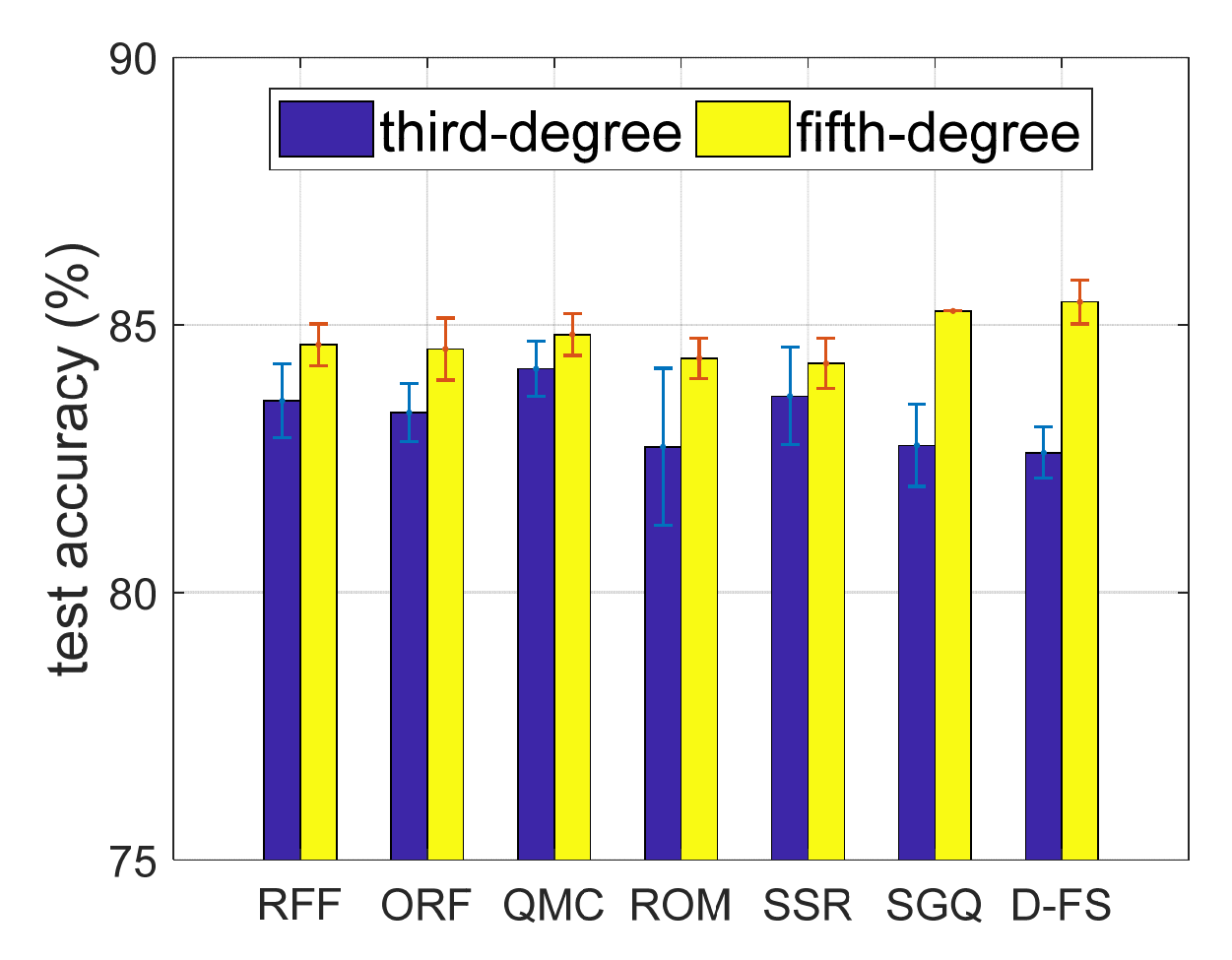}}
	\subfigure[\emph{letter}]{
		\includegraphics[width=0.22\textwidth]{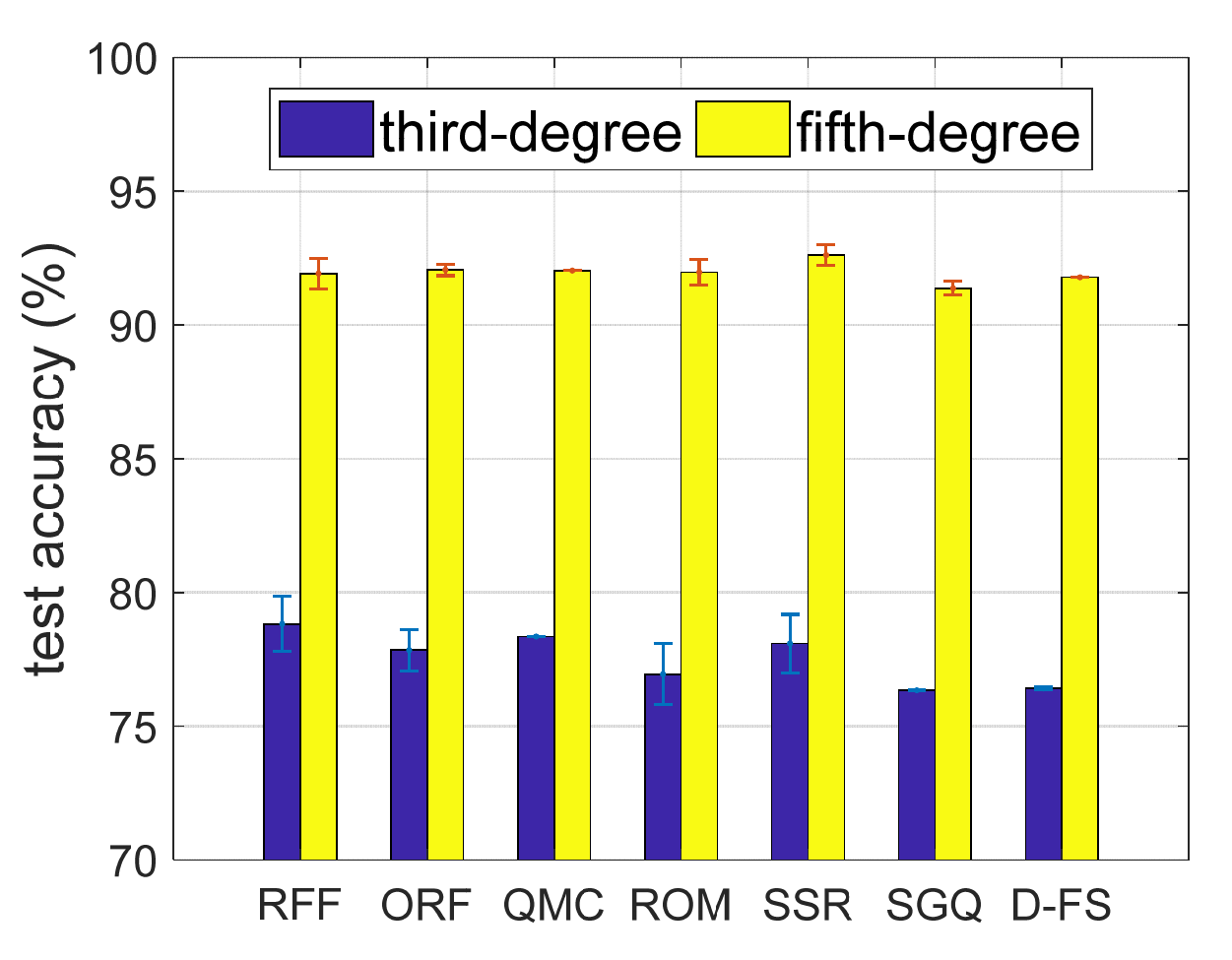}}
	\subfigure[\emph{ijcnn1}]{
		\includegraphics[width=0.22\textwidth]{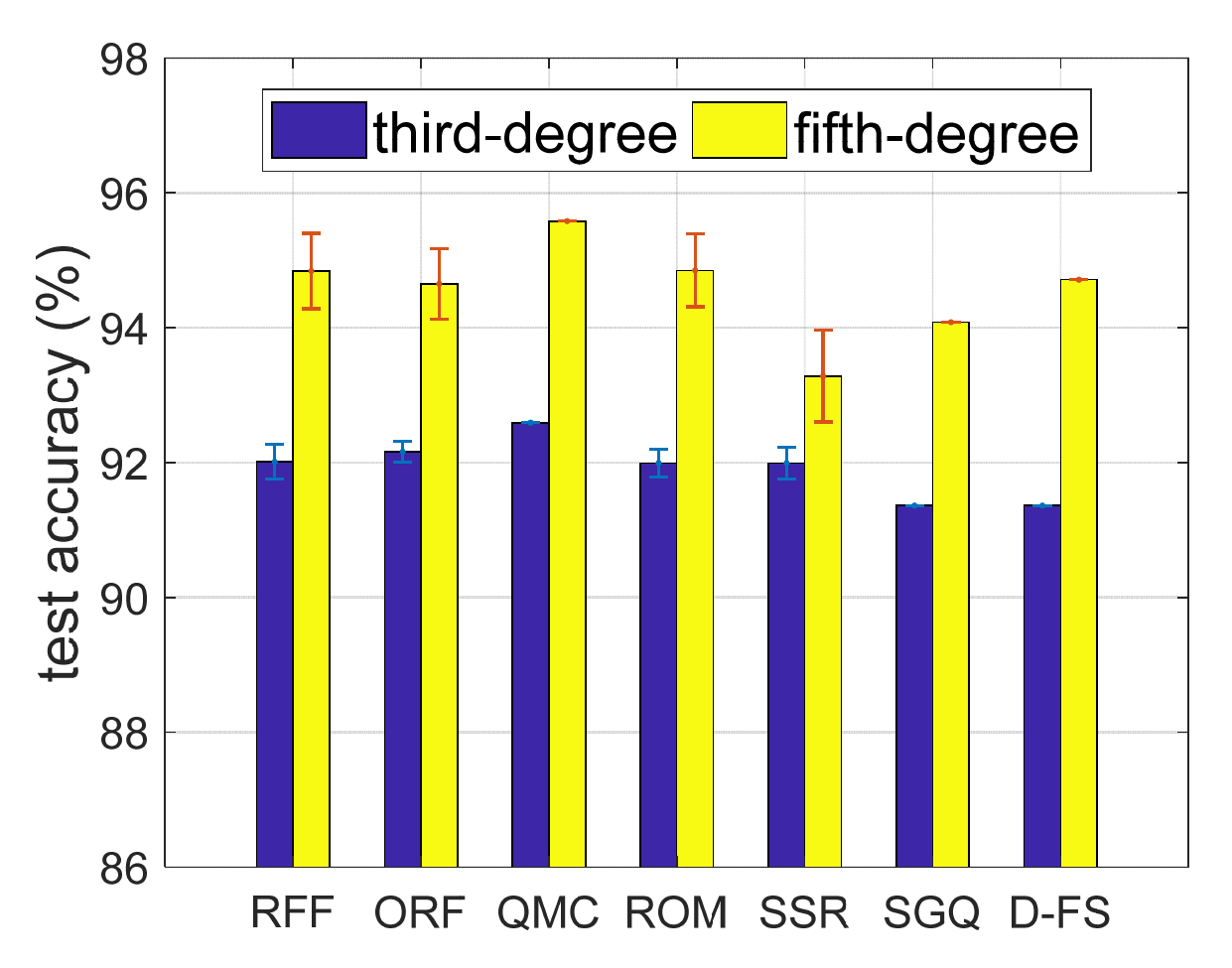}}
	\subfigure[\emph{covtype}]{
		\includegraphics[width=0.22\textwidth]{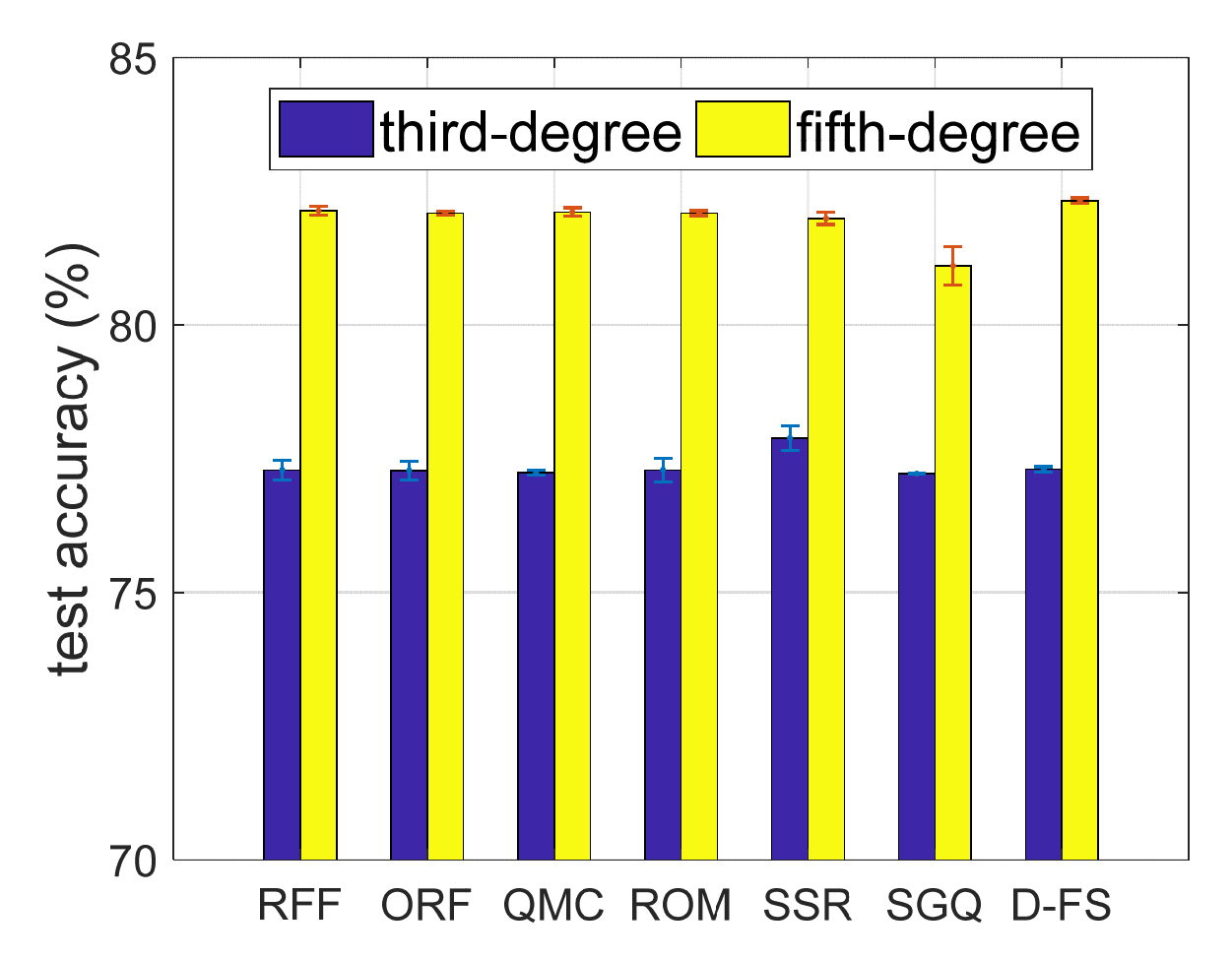}}
	
	\caption{Results on the Gaussian kernel in terms of approximation error (top), time cost (middle), and test accuracy (bottom).}\label{figappgauss}
\end{figure*}

\subsection{Experimental Settings}
{\bf Kernel:} According to the integral representation~\eqref{originte}, we choose the popular Gaussian kernel and the first-order arc-cosine kernel for experimental validation.
Here we use the following formulation of the Gaussian kernel
\begin{equation}\label{gaussnew}
	k(\bm x, \bm y) = \exp \left(-\frac{\| \bm x - \bm y \|_2^2}{2d \sigma^2}\right)\,,
\end{equation}
where the feature dimension $d$ is introduced into the kernel width for scaling as suggested by the remark in Theorem~\ref{thmvar}.
The parameter $\sigma^2$ is tuned via 5-fold inner cross validation over a grid of $\{ 0.1,0.5,1,5,10 \}$.
The first-order arc-cosine kernel \cite{cho2009kernel} used in this paper is given by
\begin{equation*}
k(\bm x, \bm y) = \frac{1}{\pi} \| \bm x \|_2 \| \bm y \|_2 \left(\sin \theta + (\pi - \theta) \cos \theta \right) \,,
\end{equation*}
with $\theta = \cos^{-1} \left(\frac{\bm x^{\!\top} \bm y}{\| \bm x \|_2 \| \bm y \|_2} \right)$.

\noindent{\bf Datasets:}
We consider four typical classification datasets including \emph{magic04}, \emph{letter}, \emph{ijcnn1}, and \emph{covtype}; see Table~\ref{tablarge} for an overview.
These datasets can be downloaded from \url{https://www.csie.ntu.edu.tw/~cjlin/libsvmtools/datasets/} or the UCI Machine Learning Repository\footnote{\url{https://archive.ics.uci.edu/ml/datasets.html.}}.
{The \emph{ijcnn1} dataset by the provider has been already scaled to $[0,1]^d$ by the winner's transformation \cite{chang2001ijcnn}. The remaining three datasets rescale each attribute/feature to $[0,1]$ by the min-max normalization.}
Regarding to the training/test partition, it has been pre-given on the \emph{letter} and \emph{ijcnn1} datasets.
For the remaining two datasets, we randomly pick half of the data for training and the rest for test.

\noindent {\bf Compared methods:}
We compare the developed D-FS/S-FS with the following algorithms:
\begin{itemize}
	\item RFF/MC \cite{rahimi2007random}: The transformation matrix $\bm W_{\text{RFF}}$ is constructed by the standard Monte Carlo sampling scheme with $W_{ij} \sim \mathcal{N}(0,1/(d\sigma^2))$ in Eq.~\eqref{gaussnew} for Gaussian kernel approximation and $W_{ij} \sim \mathcal{N}(0,1)$ for the first-order arc-cosine kernel approximation.
	\item ORF \cite{Yu2016Orthogonal}: The transformation matrix $\bm W_{\text{ORF}}$ is constructed by a random orthogonal matrix with $\bm W = \bm \Lambda \bm Q$, where $\bm \Lambda$ is a diagonal matrix with $\Lambda_{ii} \sim \chi(d)$ and $\bm Q$ is obtained from the QR decomposition of $\bm W_{\text{RFF}}$. Note that, this approach can be applied to the arc-cosine kernel in practice but lacks theoretical guarantees.
	\item ROM \cite{choromanski2017unreasonable}: The transformation matrix $\bm W_{\text{ROM}}$ is constructed by a series of structural random orthogonal matrices with $\bm W = c\prod_{i=1}^t \bm H \bm \Lambda_i$, where $\bm H$ is a normalized Hadamard matrix and $\bm \Lambda_i$ is the Rademacher matrix with $\mathbb{P}(\Lambda_{ii} = \pm 1) = 1/2$. Here $c$ is chosen as $\sqrt{2/\sigma^2}$ for Gaussian kernel approximation and $\sqrt{d}$ for arc-cosine kernel approximation.
	\item QMC \cite{Avron2016Quasi}: The transformation matrix $\bm W_{\text{QMC}}$ is constructed by a deterministic low-discrepancy Halton sequence.
	\item GQ/SGQ \cite{dao2017gaussian}: These two algorithms are deterministic quadrature methods. GQ generates nodes and weights along each dimension and thus the dimension of the obtained feature mapping can be manually adjusted. However, the feature dimension generated by SGQ is directly fixed if $d$ is given. Accordingly, we compare SGQ with D-FS in Section~\ref{sec:expdeter} and compare GQ with S-FS in Section~\ref{sec:exp:sto}. For fair comparison, we set the \emph{generator} vector $\bm \lambda = [0, \sqrt{3}]^{\!\top}$ in GQ, SGQ and D-FS/S-FS to be the same. 
	\item SSR \cite{munkhoeva2018quadrature}: The feature mapping is constructed by the third-degree stochastic spherical-radial rule with random orthogonal matrices obtained by butterfly matrices \cite{genz1998methods}.
\end{itemize}

\noindent{\bf Evaluation metrics:} We evaluate the performance of all the compared algorithms in terms of approximation error, time cost, and test accuracy.
The used kernel approximation measure here is the relative error in Frobenius form ${\| \bm K - \hat{\bm K}\|_{\mathrm{F}}}/{\| \bm K \|_{\mathrm{F}}}$ on a randomly selected subset with 1,000 samples.
We record the time cost of each algorithm on generating feature mappings.
For prediction (binary classification), we directly use the closed-form formula of KRR \cite{suykens2002least} and the sign function to output a binary label.
The regularization parameter in KRR is tuned via 5-fold inner cross validation over a grid of $\{ 0.0001, 0.001,0.01,0.1,0.5,1,10 \}$.
All experiments are repeated 10 trials.

\begin{figure*}[!htb]
	\centering
	
	\subfigure{
		\includegraphics[width=0.223\textwidth]{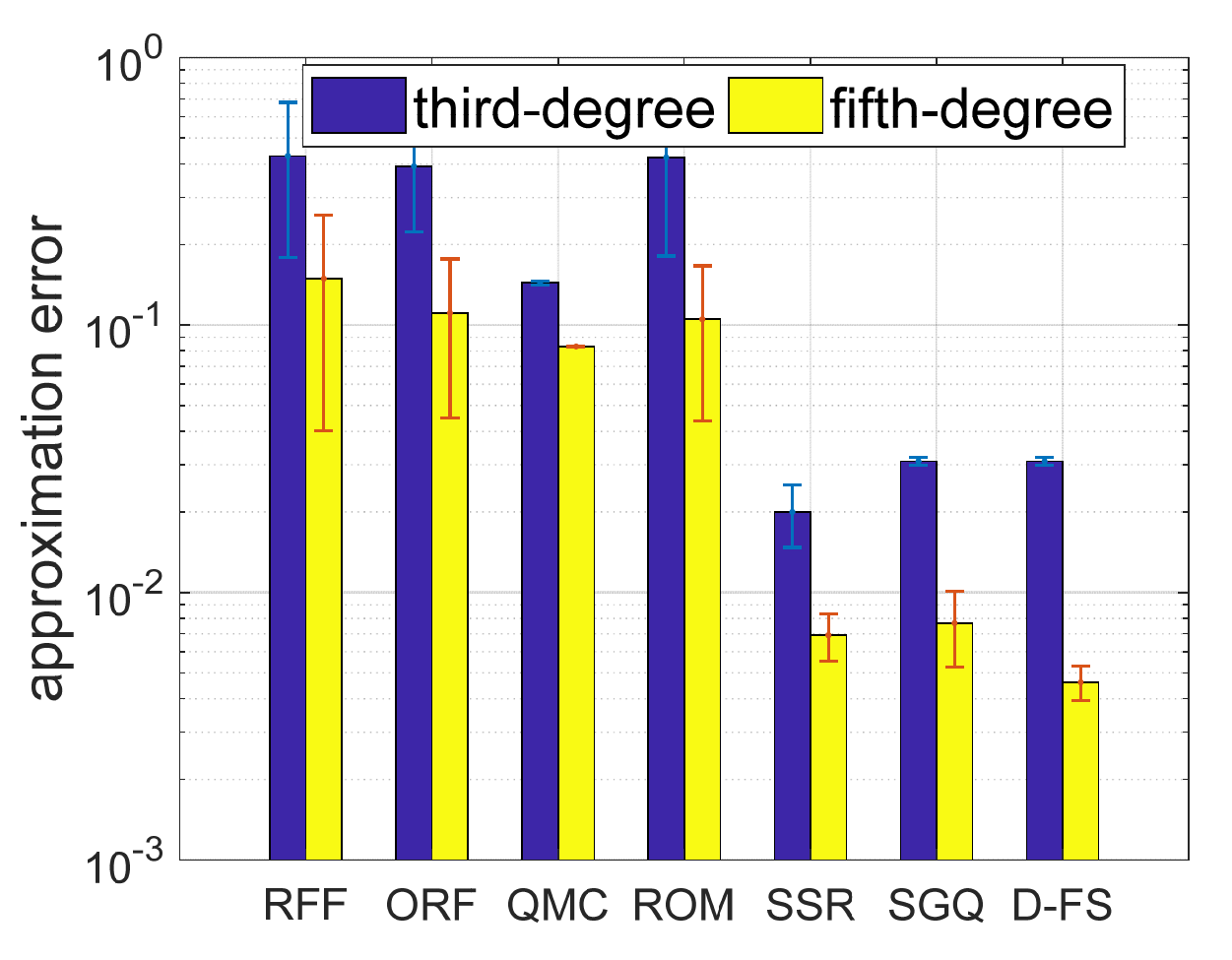}}
	\subfigure{
		\includegraphics[width=0.22\textwidth]{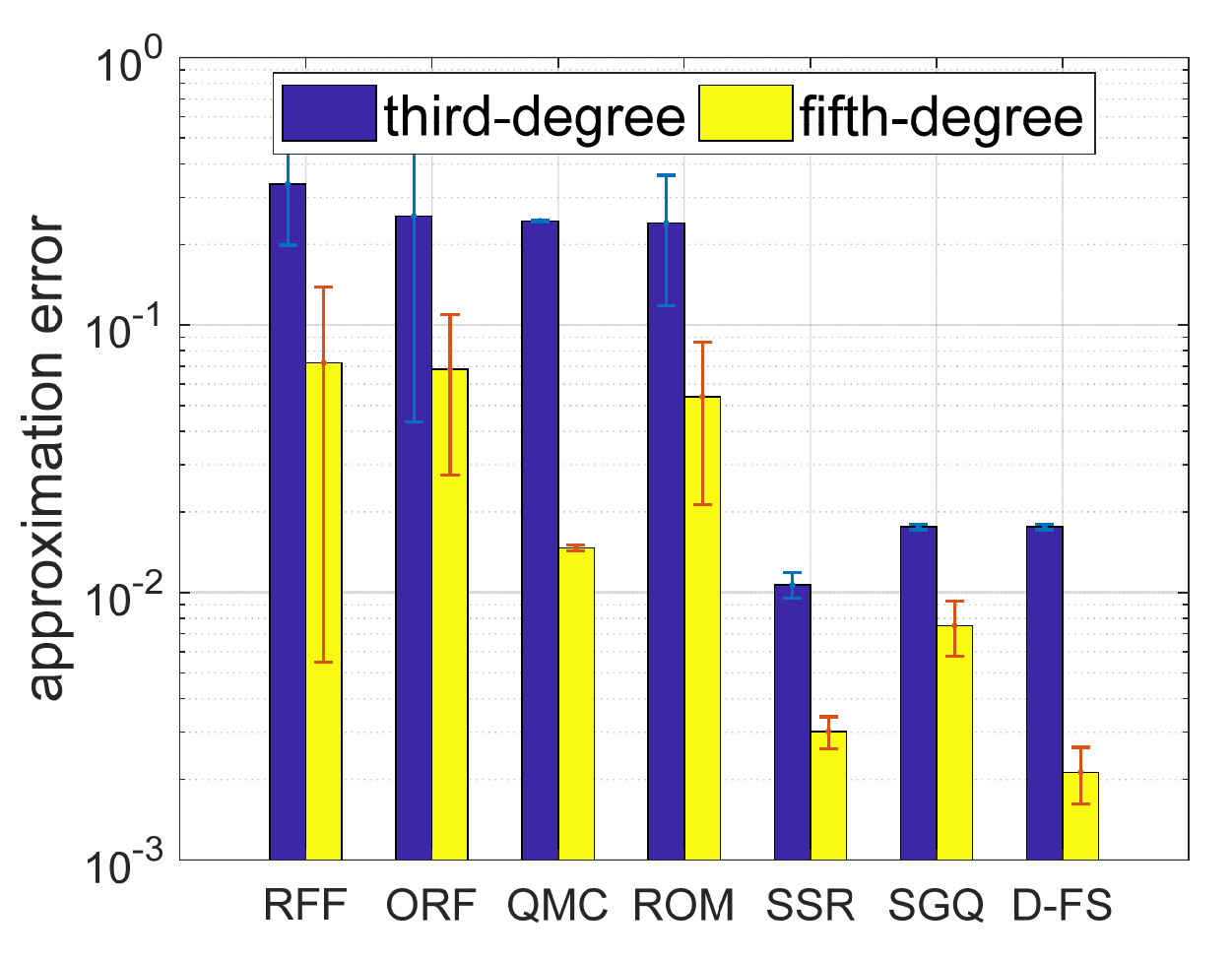}}
	\subfigure{
		\includegraphics[width=0.22\textwidth]{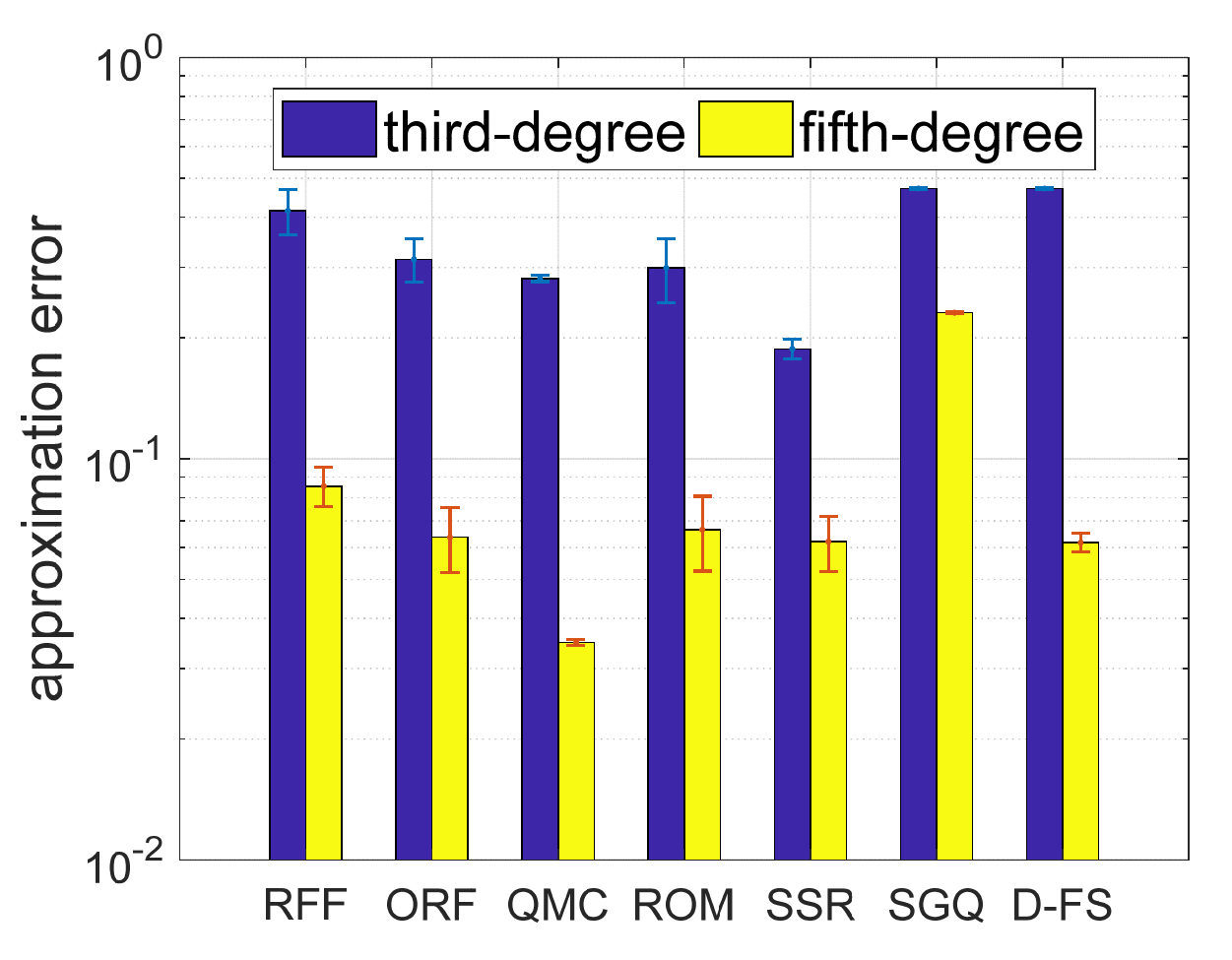}}
	\subfigure{
		\includegraphics[width=0.22\textwidth]{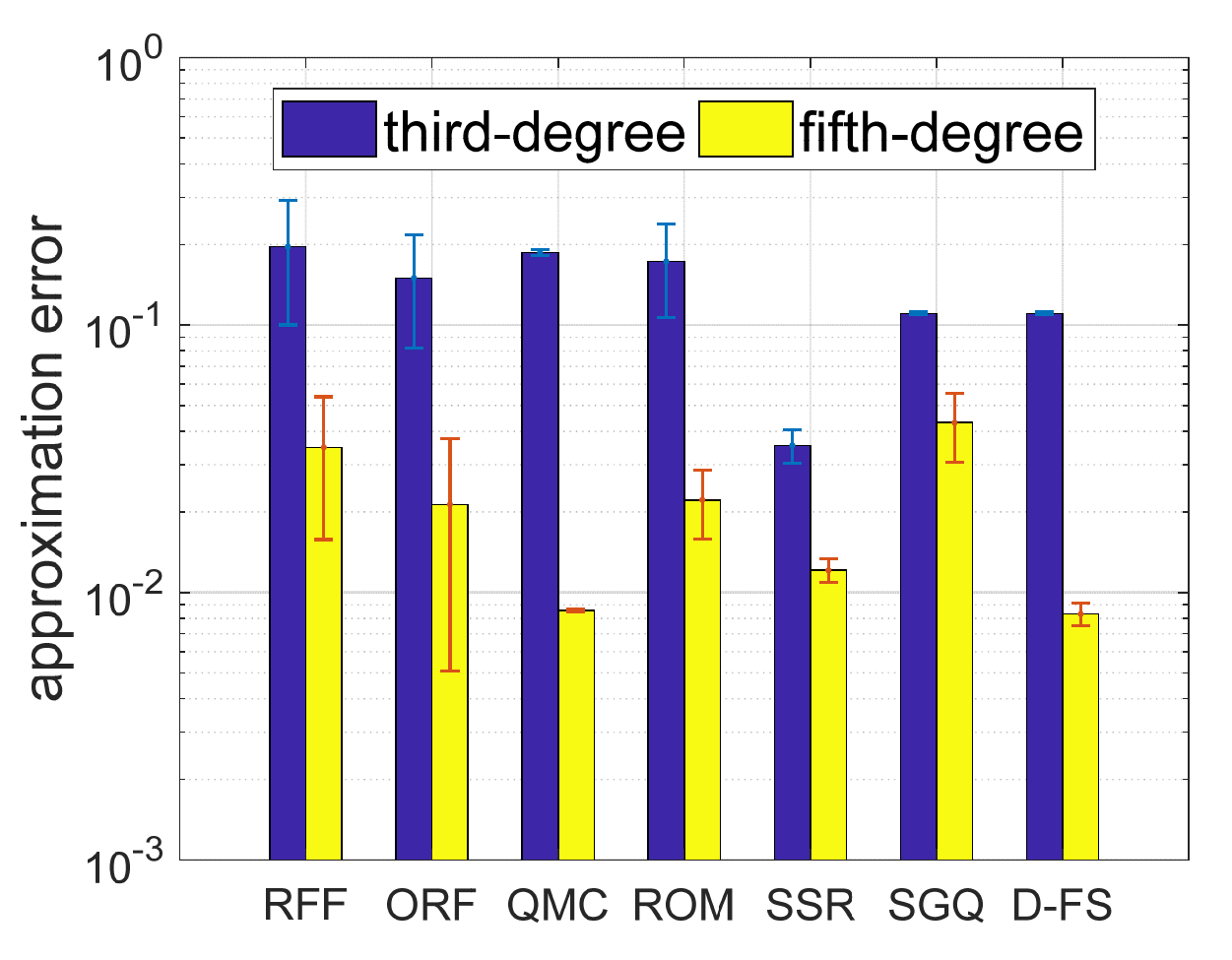}}

	\subfigure{
		\includegraphics[width=0.22\textwidth]{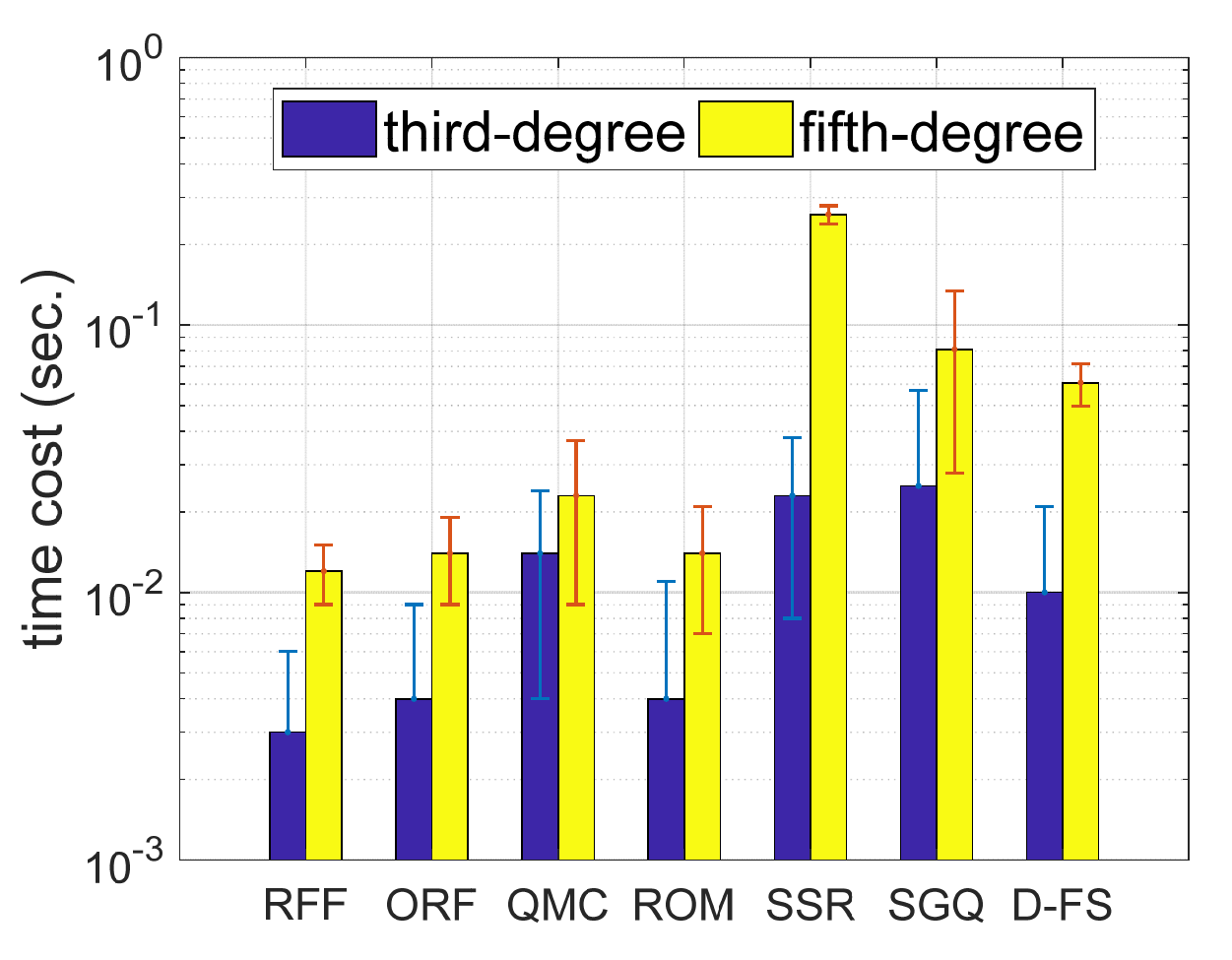}}
	\subfigure{
		\includegraphics[width=0.22\textwidth]{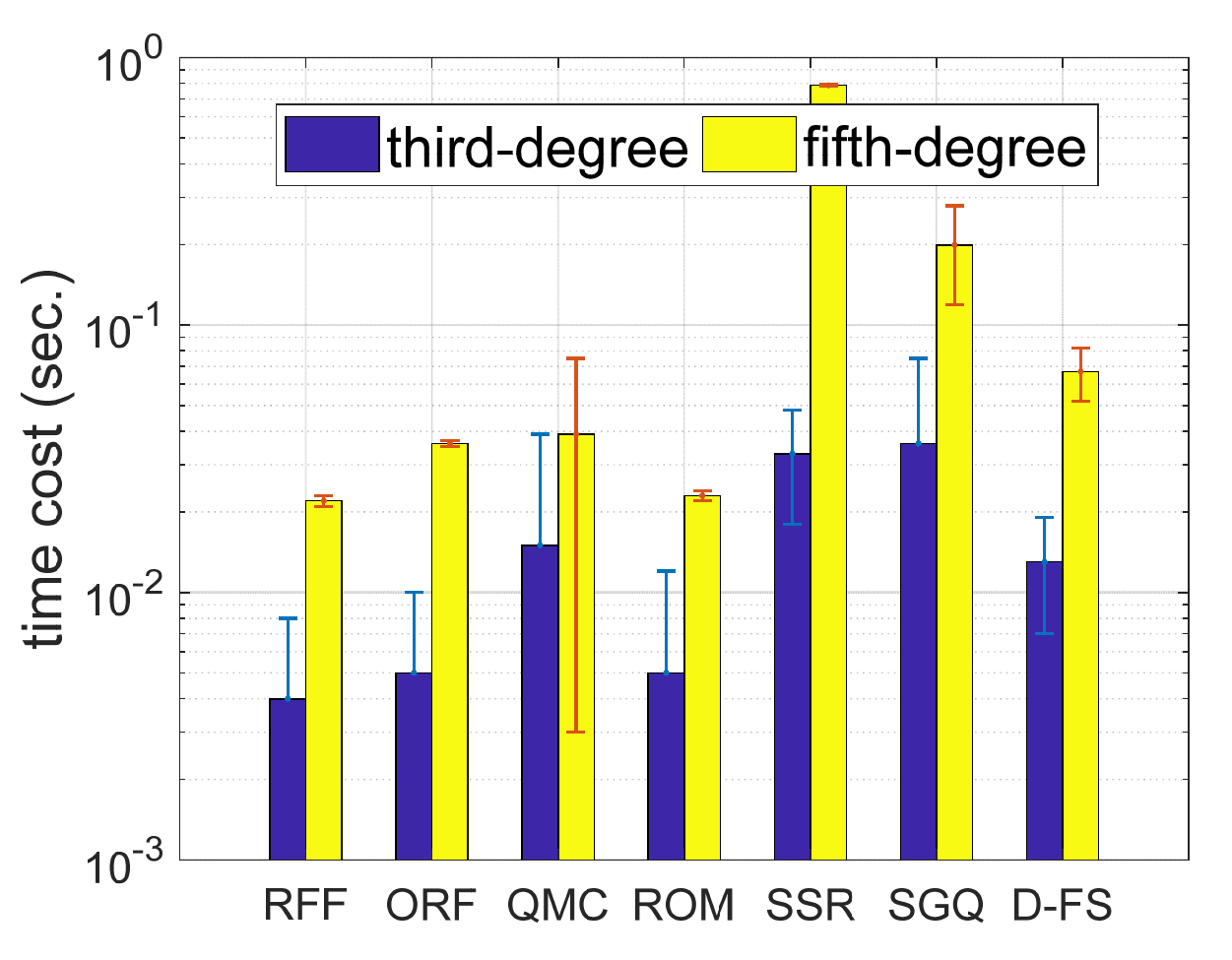}}
	\subfigure{
		\includegraphics[width=0.22\textwidth]{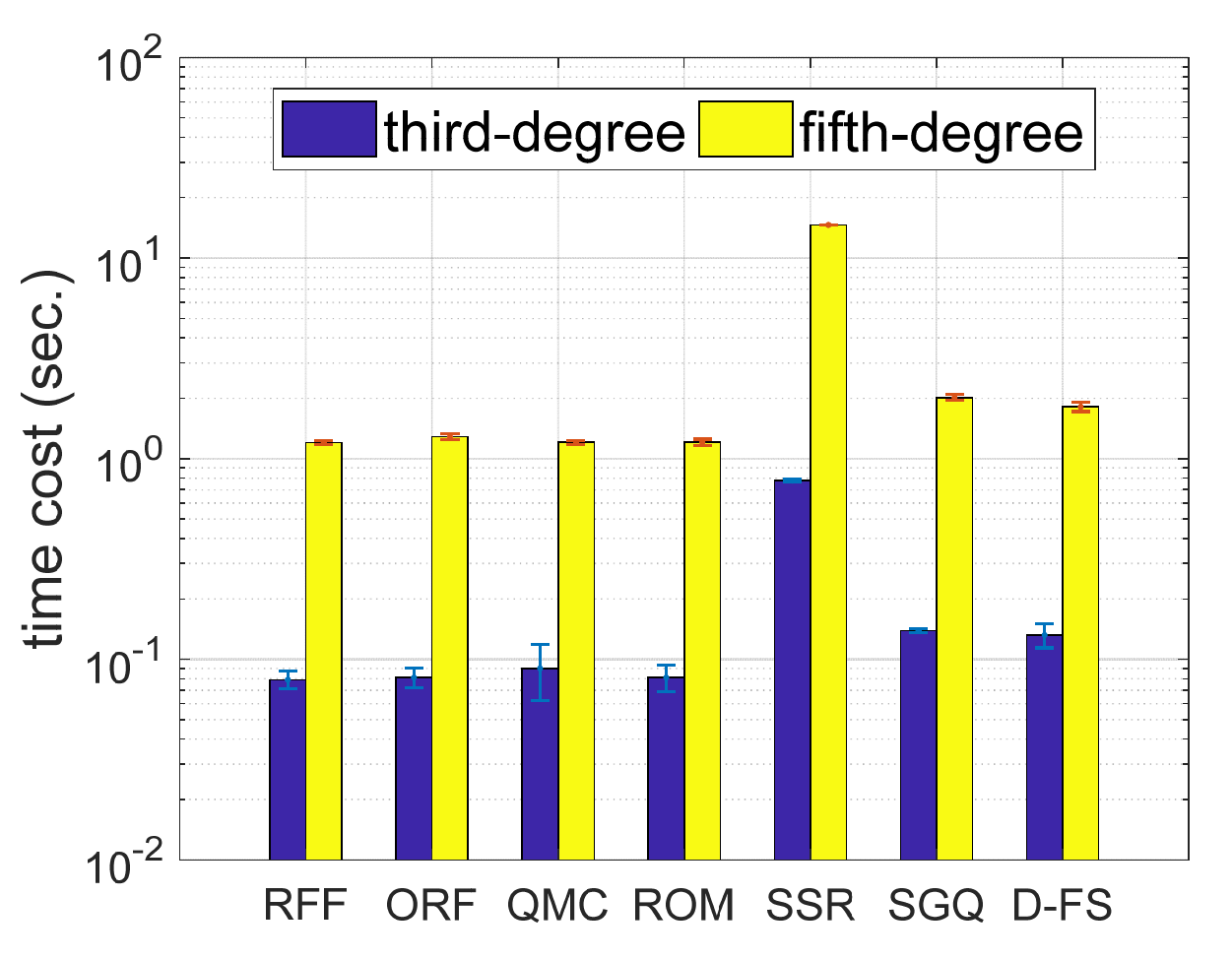}}
	\subfigure{
		\includegraphics[width=0.22\textwidth]{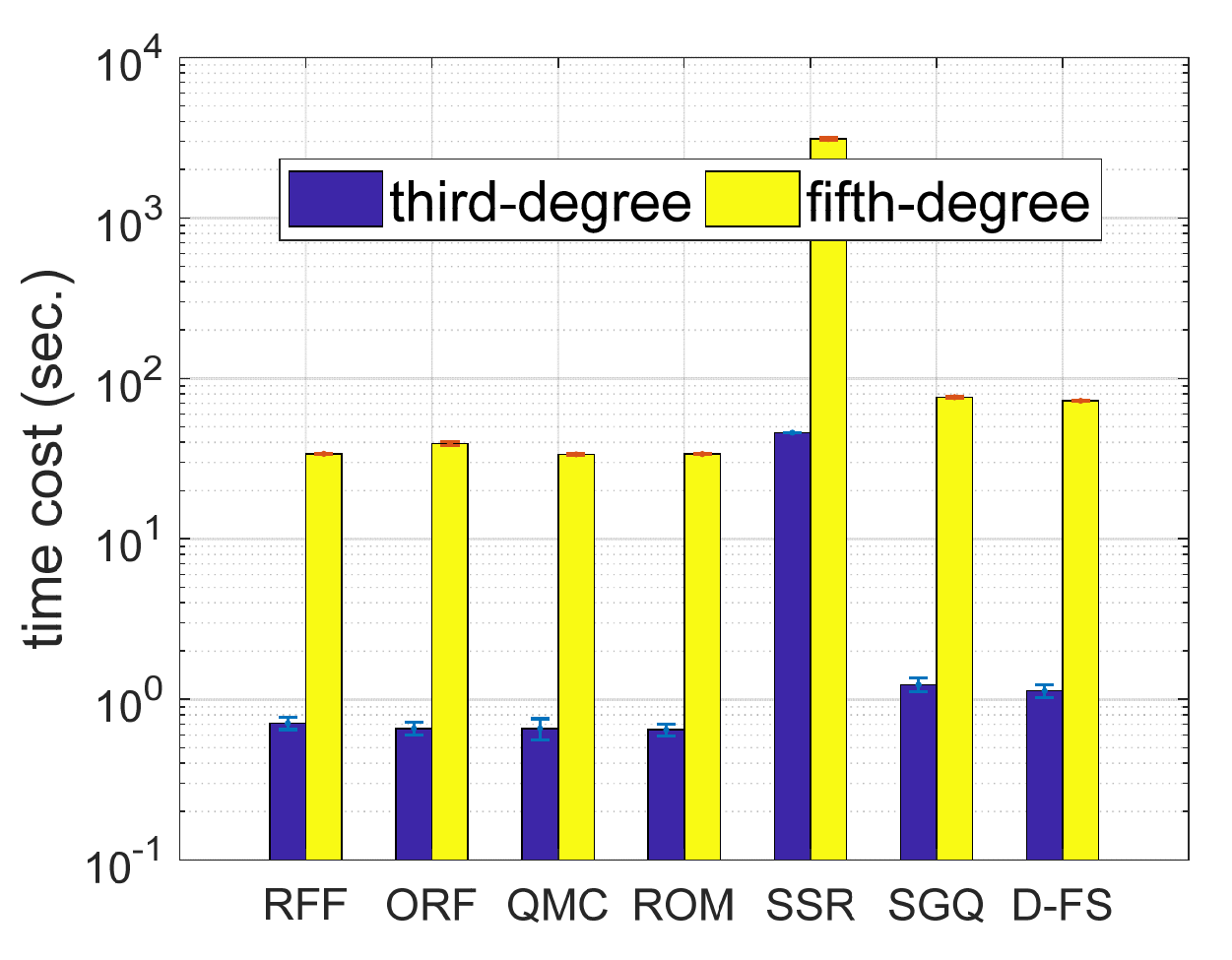}}
	
	\subfigure[\emph{magic04}]{
		\includegraphics[width=0.22\textwidth]{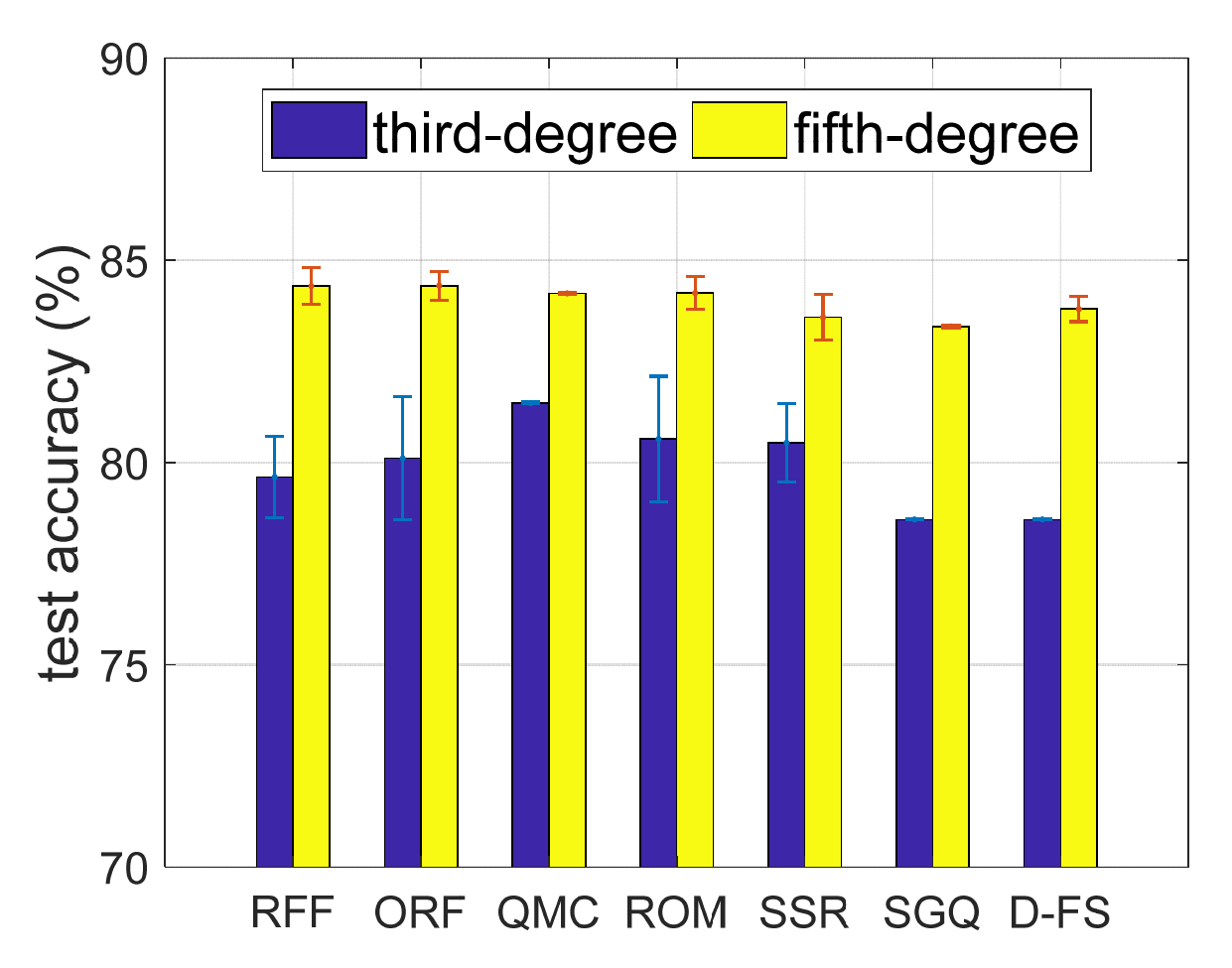}}
	\subfigure[\emph{letter}]{
		\includegraphics[width=0.22\textwidth]{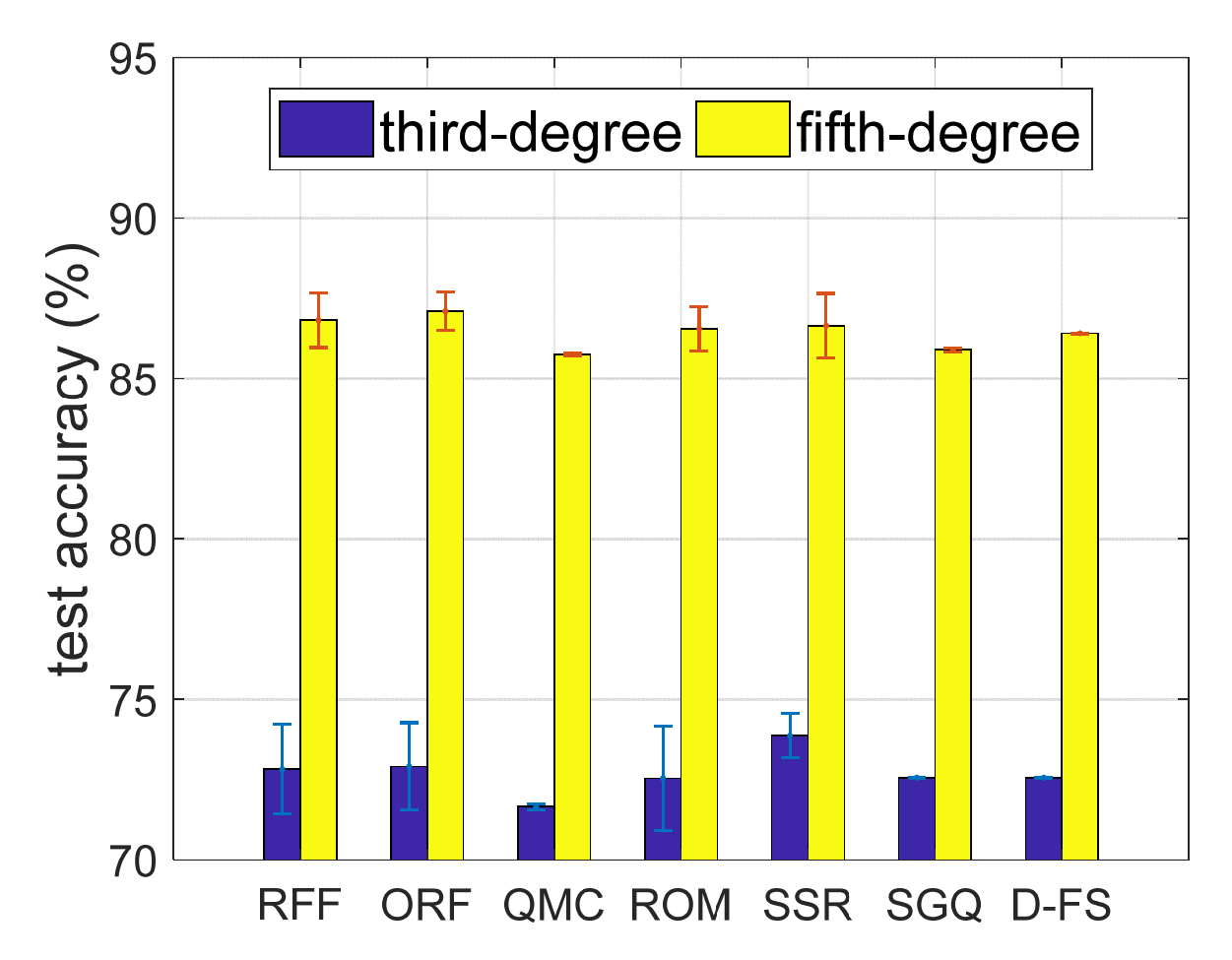}}
	\subfigure[\emph{ijcnn1}]{
		\includegraphics[width=0.22\textwidth]{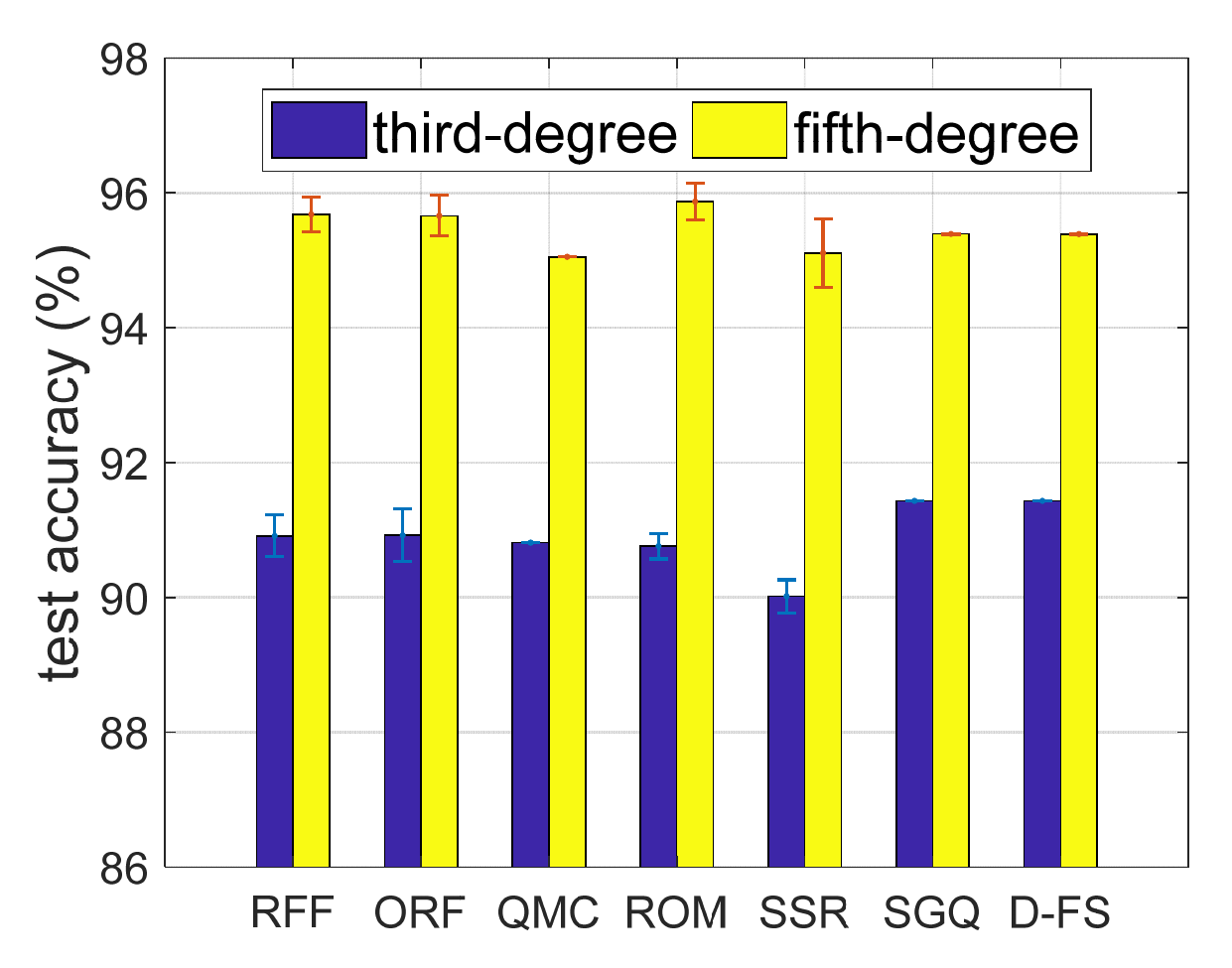}}
	\subfigure[\emph{covtype}]{
		\includegraphics[width=0.22\textwidth]{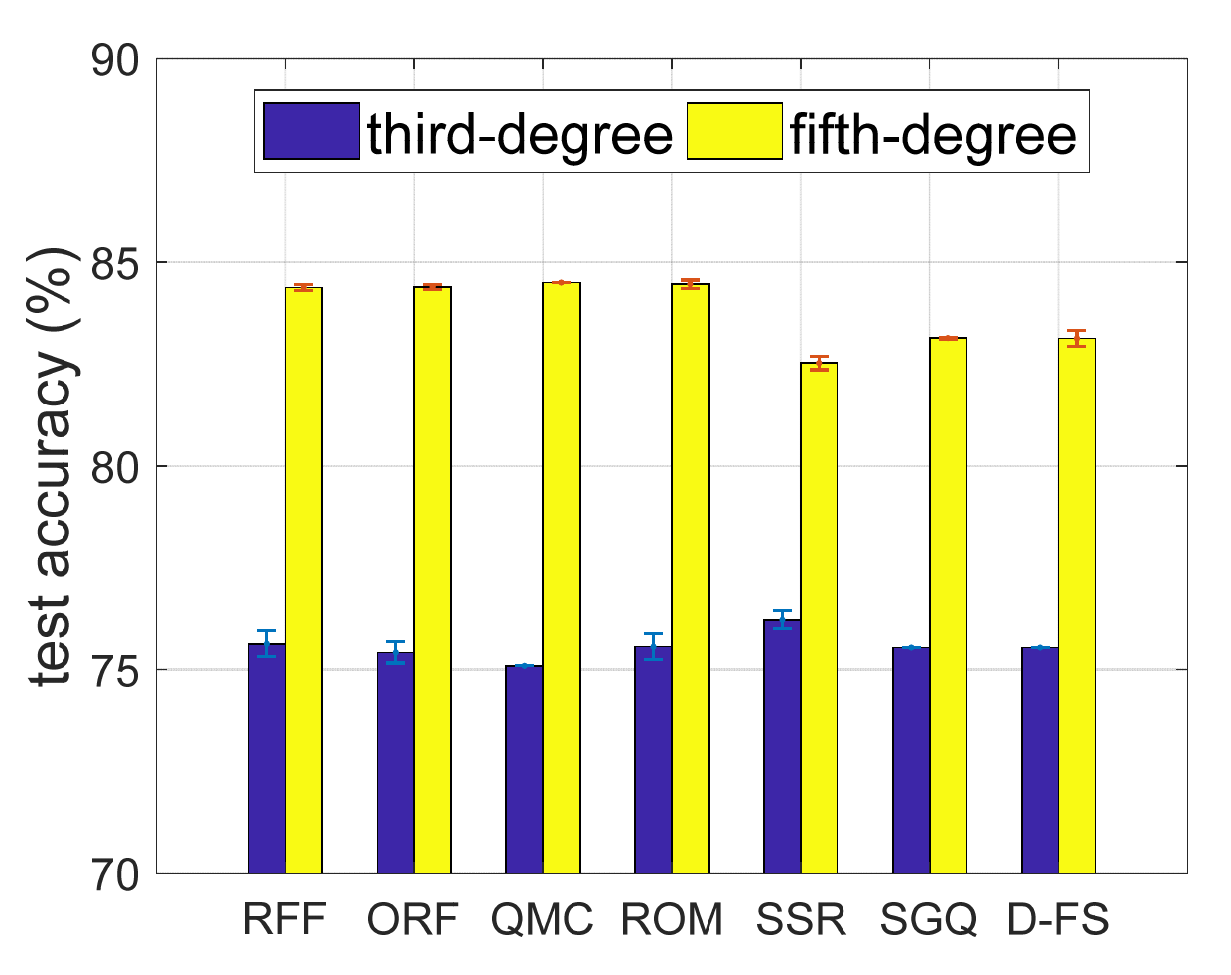}}
	
	\caption{Results on the first-order arc-cosine kernel in terms of approximation error (top), time cost (middle), and test accuracy (bottom).}\label{figapparccos}
\end{figure*}

\subsection{Evaluation for Deterministic Rules}
\label{sec:expdeter}
Our deterministic rules D-FS generate the fixed-size feature mapping $\Phi(\bm x) \in \mathbb{R}^N$ when $d$ is given, e.g., $N=2d+1$ in our third-degree rule ($m=1$) and $N=1+2d^2$ in our fifth-degree rule ($m=2$).
In this case, we consider a deterministic setting, in which RFF, ORF, ROM, QMC and SSR are conducted under the same feature dimension $N$ with our third/fifth-degree rules for fair comparison.
Note that SGQ generates the same feature dimension $N=2d+1$ with D-FS in the third-degree rules but outputs the larger one $N=1+2d^2 + 2d$ in the fifth rule, see in Table~\ref{tablarge}. 

\noindent {\bf Results on Gaussian kernel:}
Figure~\ref{figappgauss} shows approximation error, time cost, and test accuracy (mean$\pm$std.) of all the compared algorithms across the Gaussian kernel in terms of the third-degree rules (see the blue bar) and the fifth-degree rules (see the yellow bar), respectively.
We find that, our third-degree D-FS decreases the approximation error of RFF and QMC, achieves a comparable performance with ORF and ROM, but is slightly inferior to SSR. 
Besides, the third-degree SGQ performs the same with D-FS in terms of the approximation error as the generated nodes in these two algorithms are almost the same due to the same \emph{generator} used.
Nevertheless, our fifth-degree D-FS not only requires smaller $N$ than SGQ, but also achieves the best approximation quality (with noticeable reduction) of all the compared algorithms.

In terms of time cost on generating the feature mapping, there is no distinct difference between our third/fifth-degree D-FS and RFF. Interestingly, our fifth-degree D-FS is more efficient than quadrature methods SSR and SGQ.
For prediction, most algorithms achieve the similar test accuracy on these datasets.
Good kernel approximation quality cannot guarantee the final good prediction, which still remains an open question in theory.
The reason may be that the approximated kernel is not necessarily optimal for prediction, as discussed by \cite{avron2017random,zhang2019f,liu2020survey}.
Nevertheless, for the design of kernel approximation, it is reasonable to pursue small approximation errors.

\noindent {\bf Results on arc-cosine kernel:}
Figure~\ref{figapparccos} shows the related results across the first-order arc-cosine kernel. The trends of the compared algorithms are analogous to those across the Gaussian kernel in Figure~\ref{figappgauss}.
Generally, the approximation error of each algorithm on the arc-cosine kernel is larger than that of Gaussian kernel.
The reason may be that the integrand $f$ for the Gaussian kernel corresponds to trigonometric functions that are infinitely differentiable; while $f$ for the first-order arc-cosine kernel is actually a ReLU function that is non-differentiable.
In fact, as we discussed in the introduction, the differentiable property on the integrand significantly affects the approximation performance in Monte Carlo sampling, QMC, and quadrature methods.

Based on the above results, we conclude that our deterministic third/fifth-degree rules are quite efficient to achieve promising performance on the approximation quality, and comparable results on classification accuracy.
\begin{figure}[t]
	\centering
	\subfigure[approximation error]{
		\includegraphics[width=0.22\textwidth]{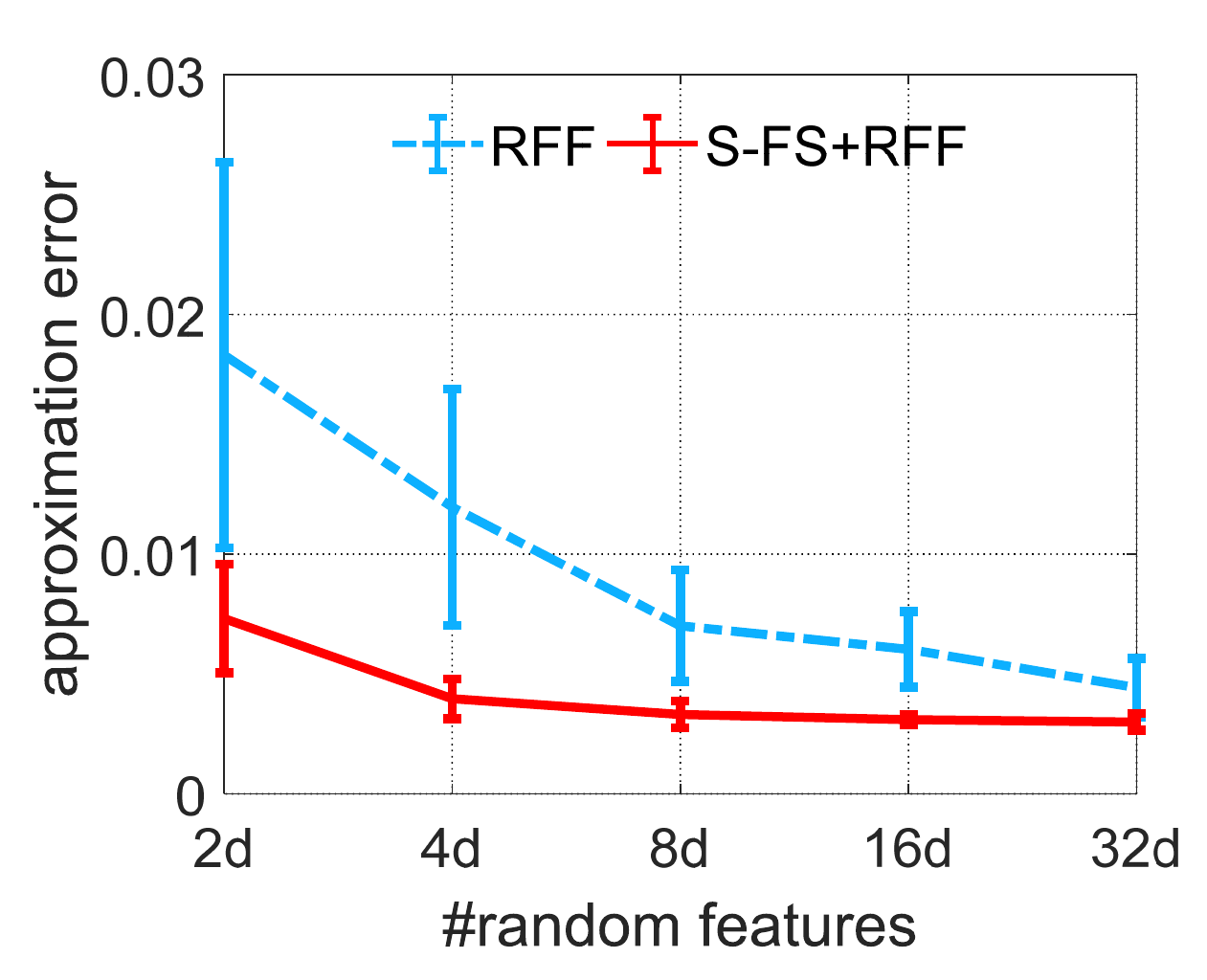}}
	\subfigure[time cost]{
		\includegraphics[width=0.22\textwidth]{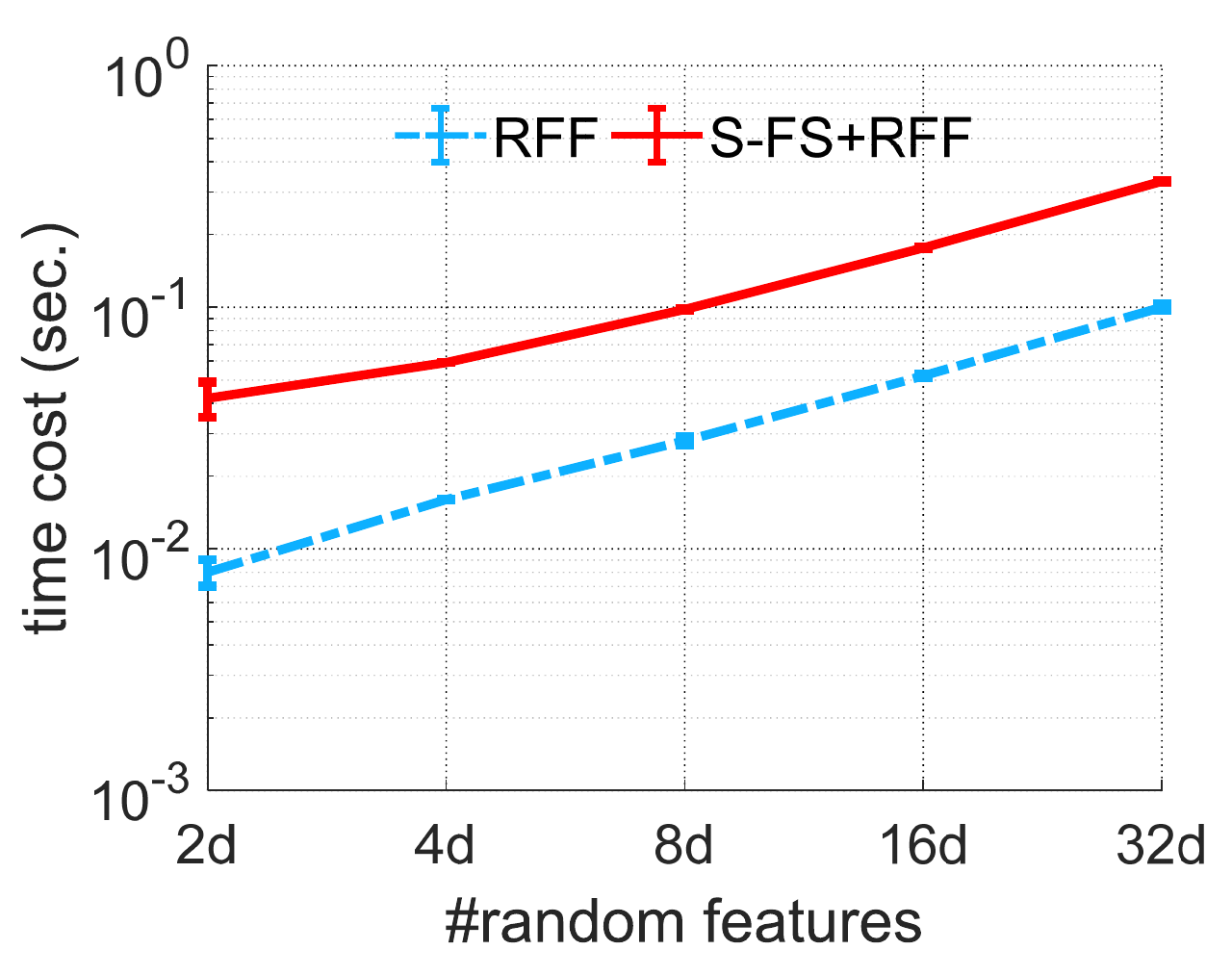}}
	\caption{Benefits of our S-FS rule in Eq.~\eqref{feamapf} against RFF across the Gaussian kernel on the \emph{magic04} data set.}\label{fvarious}
\end{figure}

\begin{figure*}[!htb]
	\centering
	
	\subfigure{
		\includegraphics[width=0.22\textwidth]{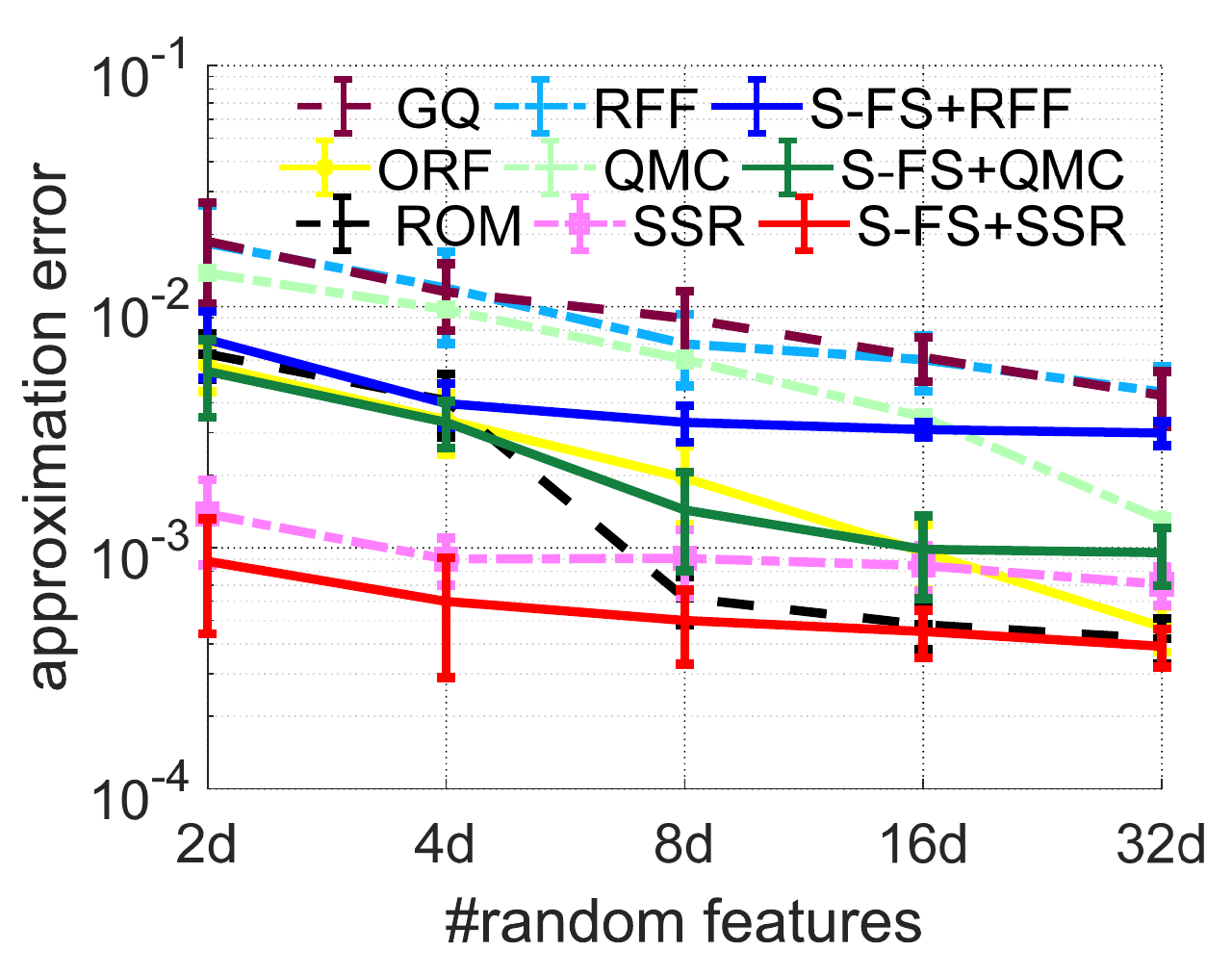}}
	\subfigure{
		\includegraphics[width=0.22\textwidth]{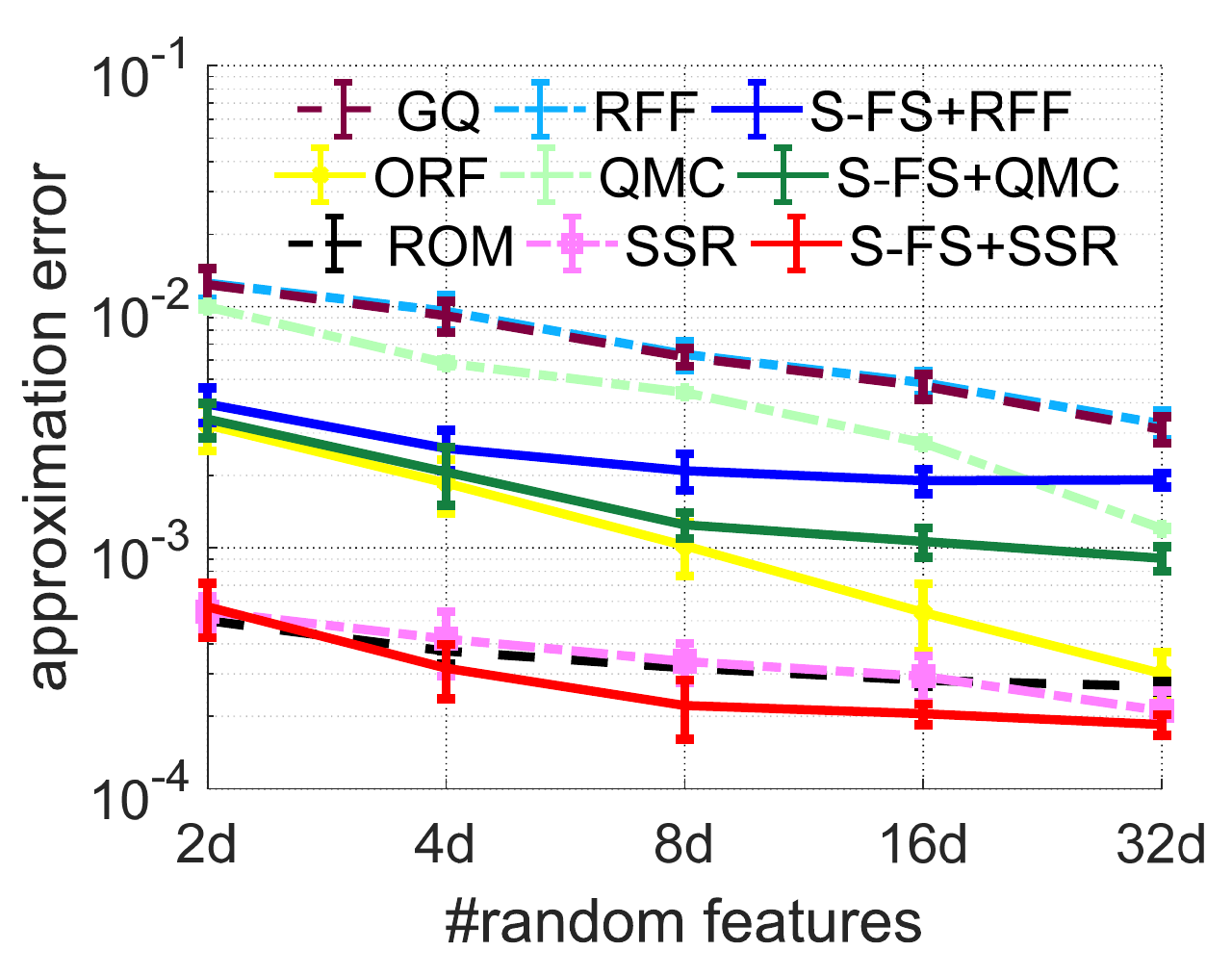}}
	\subfigure{
		\includegraphics[width=0.22\textwidth]{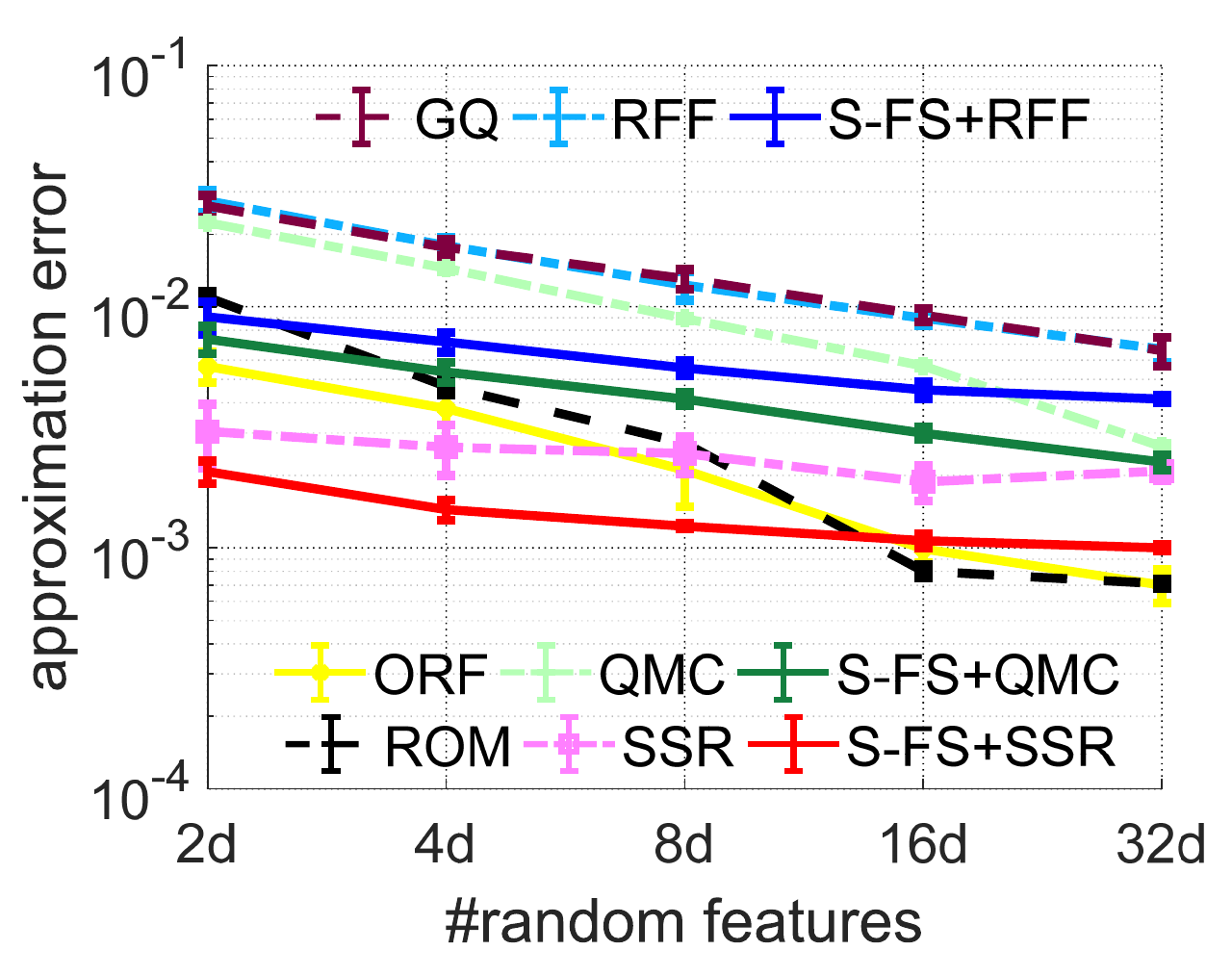}}
	\subfigure{
		\includegraphics[width=0.22\textwidth]{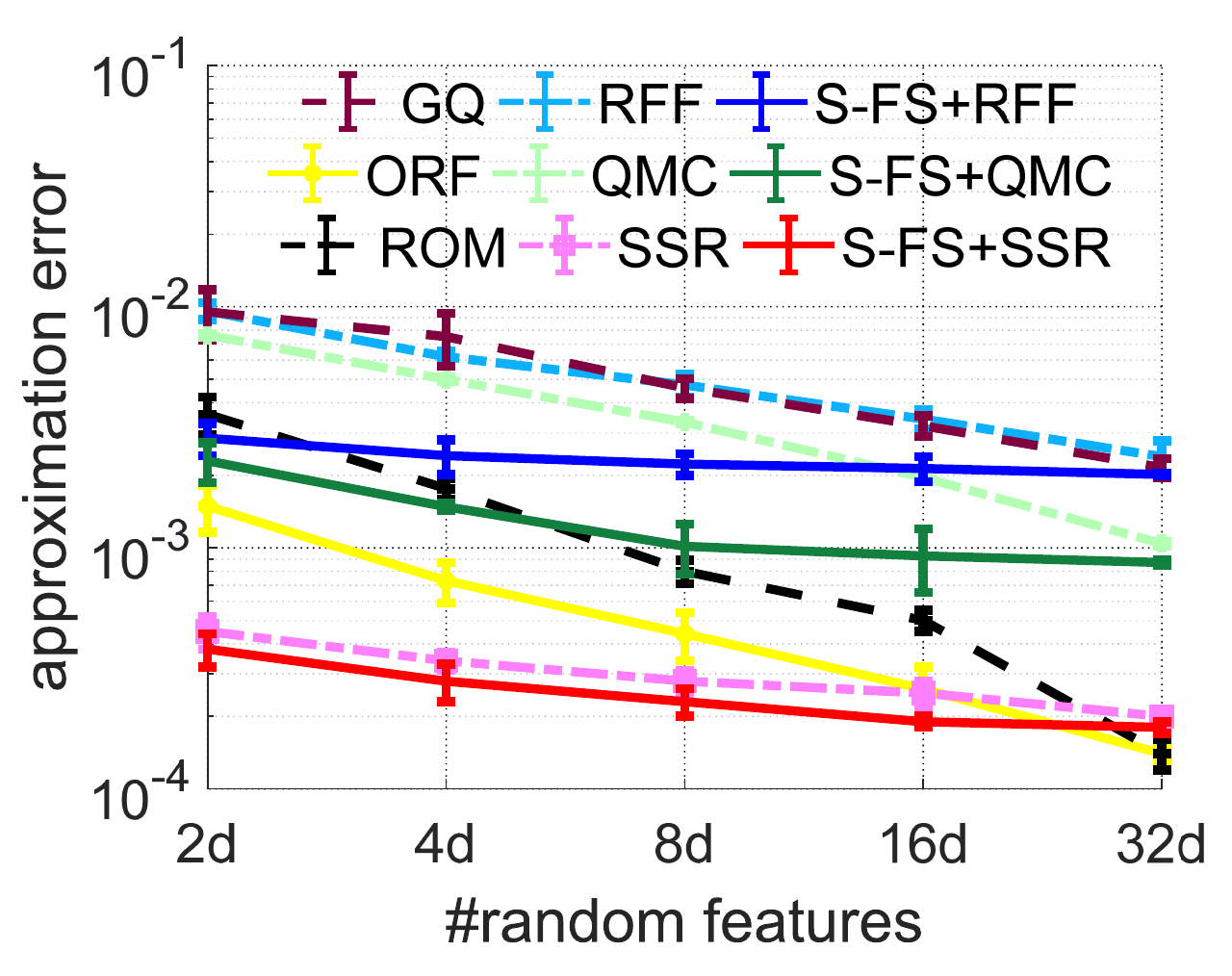}}
	
	\subfigure{
		\includegraphics[width=0.22\textwidth]{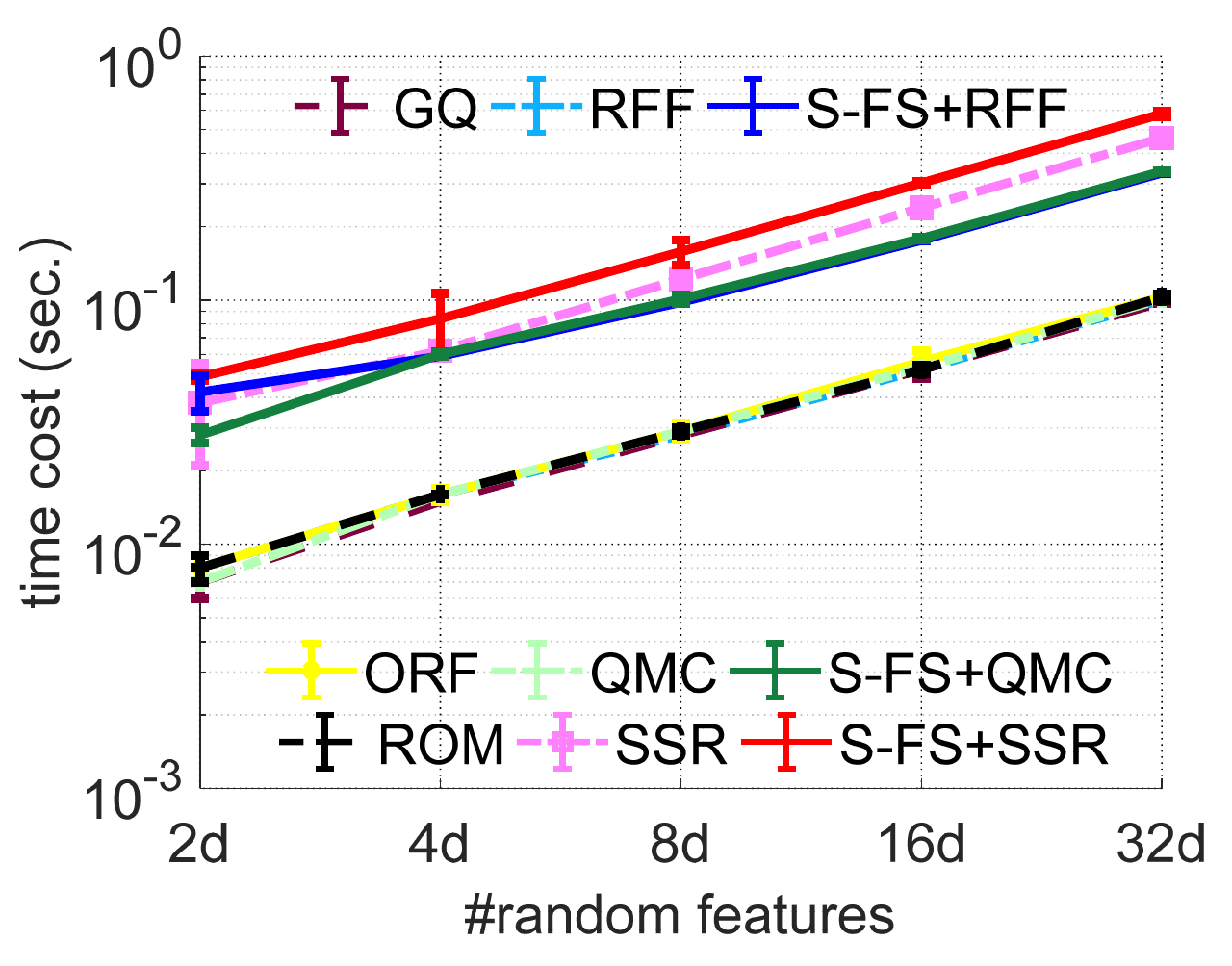}}
	\subfigure{
		\includegraphics[width=0.22\textwidth]{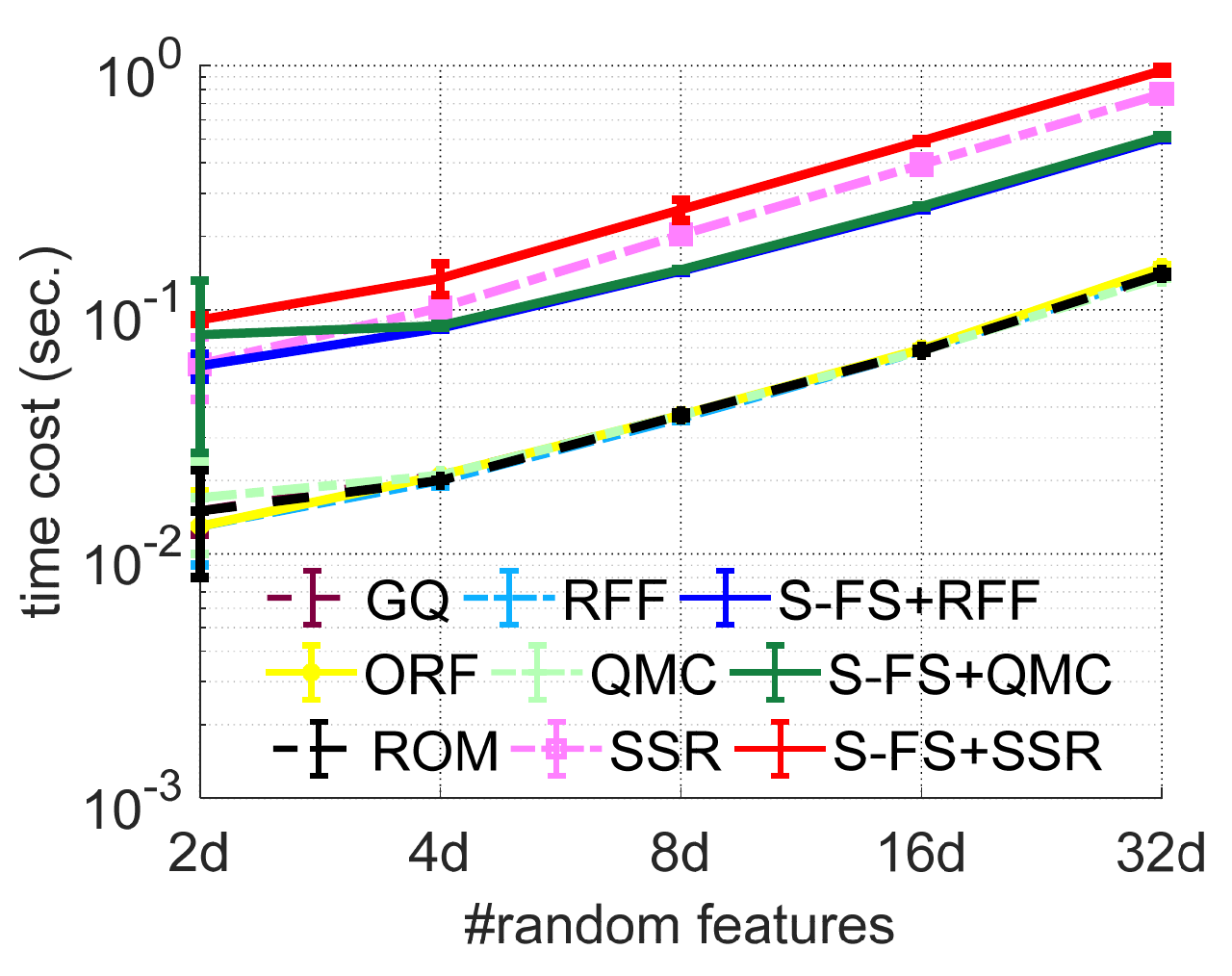}}
	\subfigure{
		\includegraphics[width=0.22\textwidth]{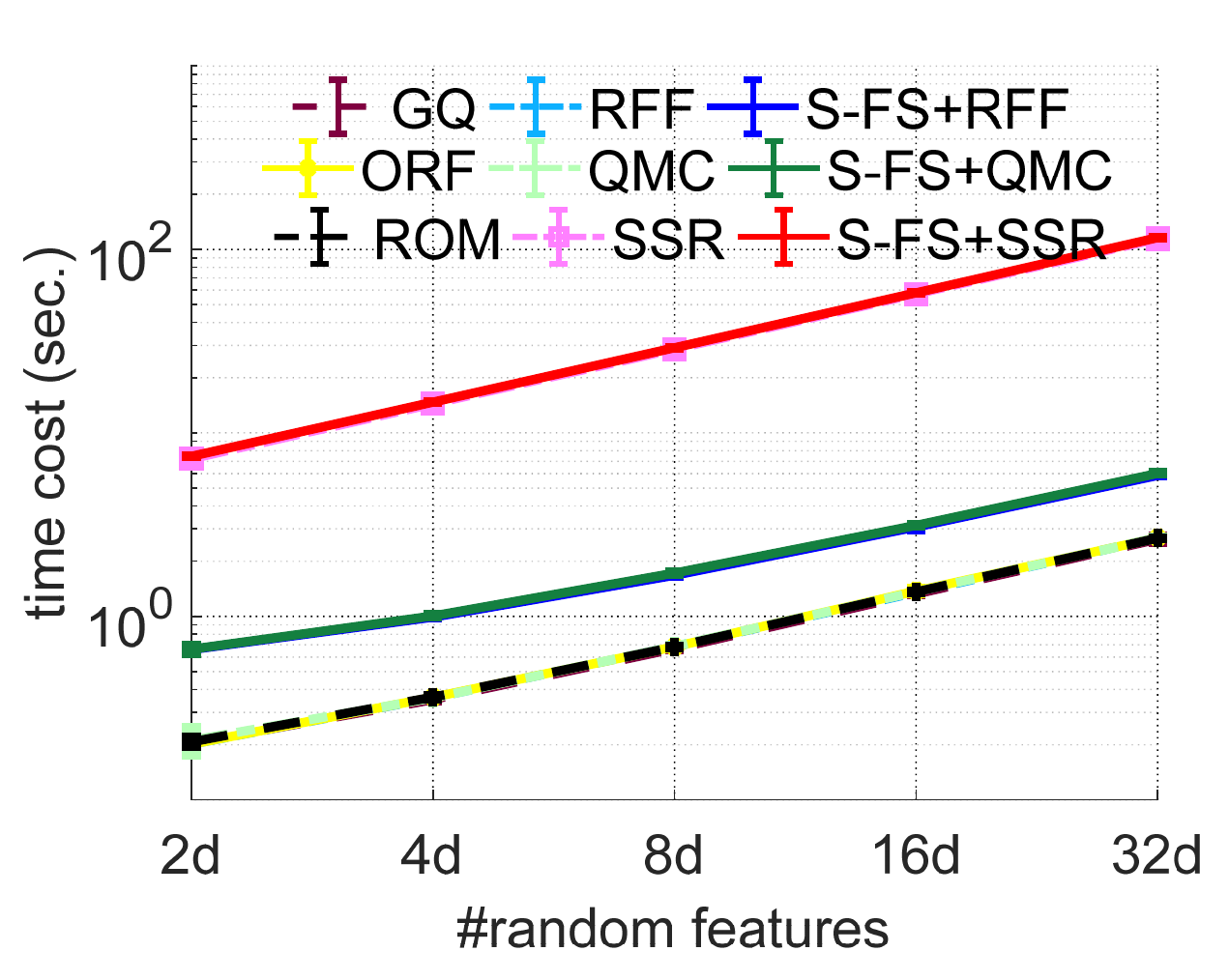}}
	\subfigure{
		\includegraphics[width=0.22\textwidth]{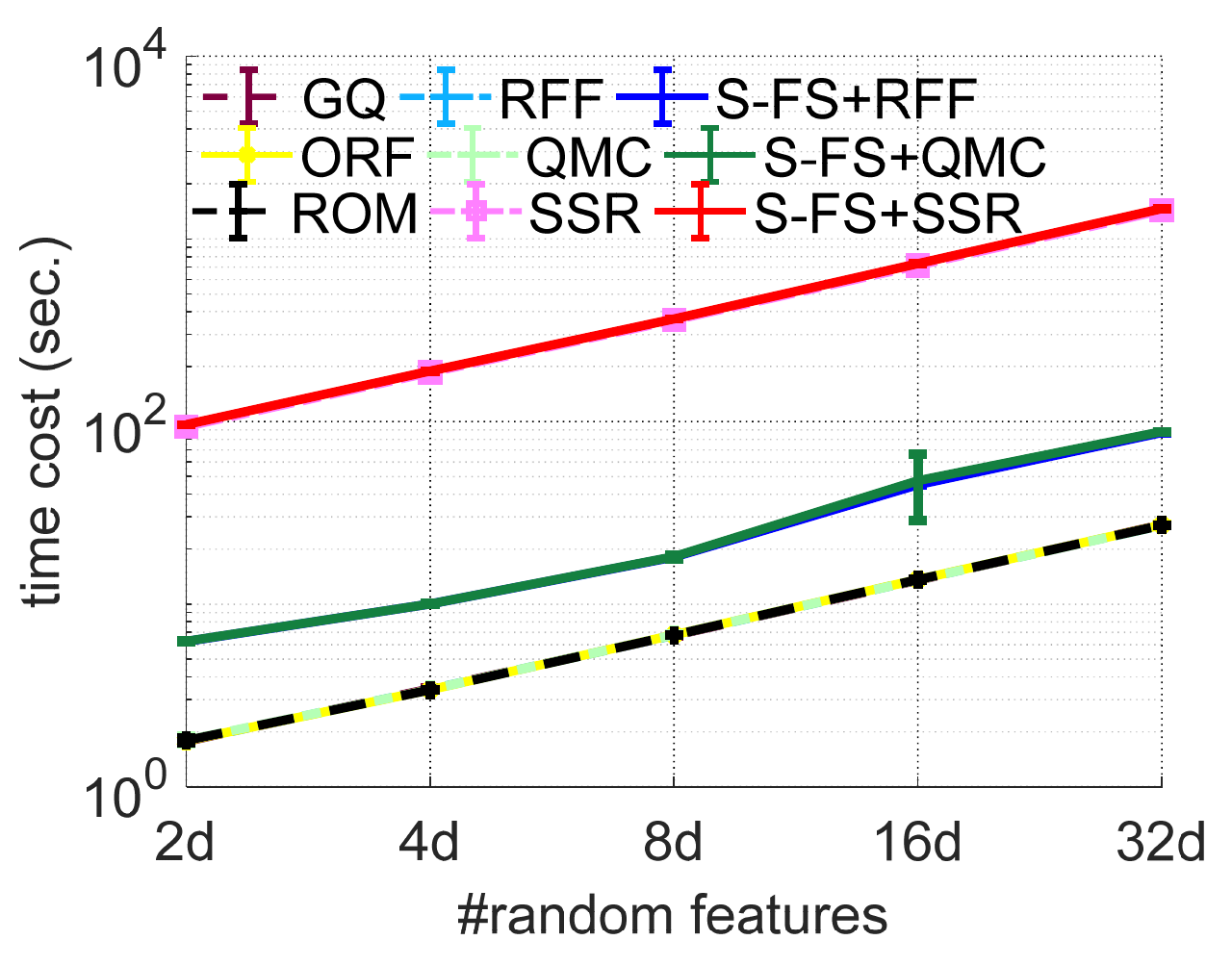}}
	
	\subfigure[\emph{magic04}]{
		\includegraphics[width=0.22\textwidth]{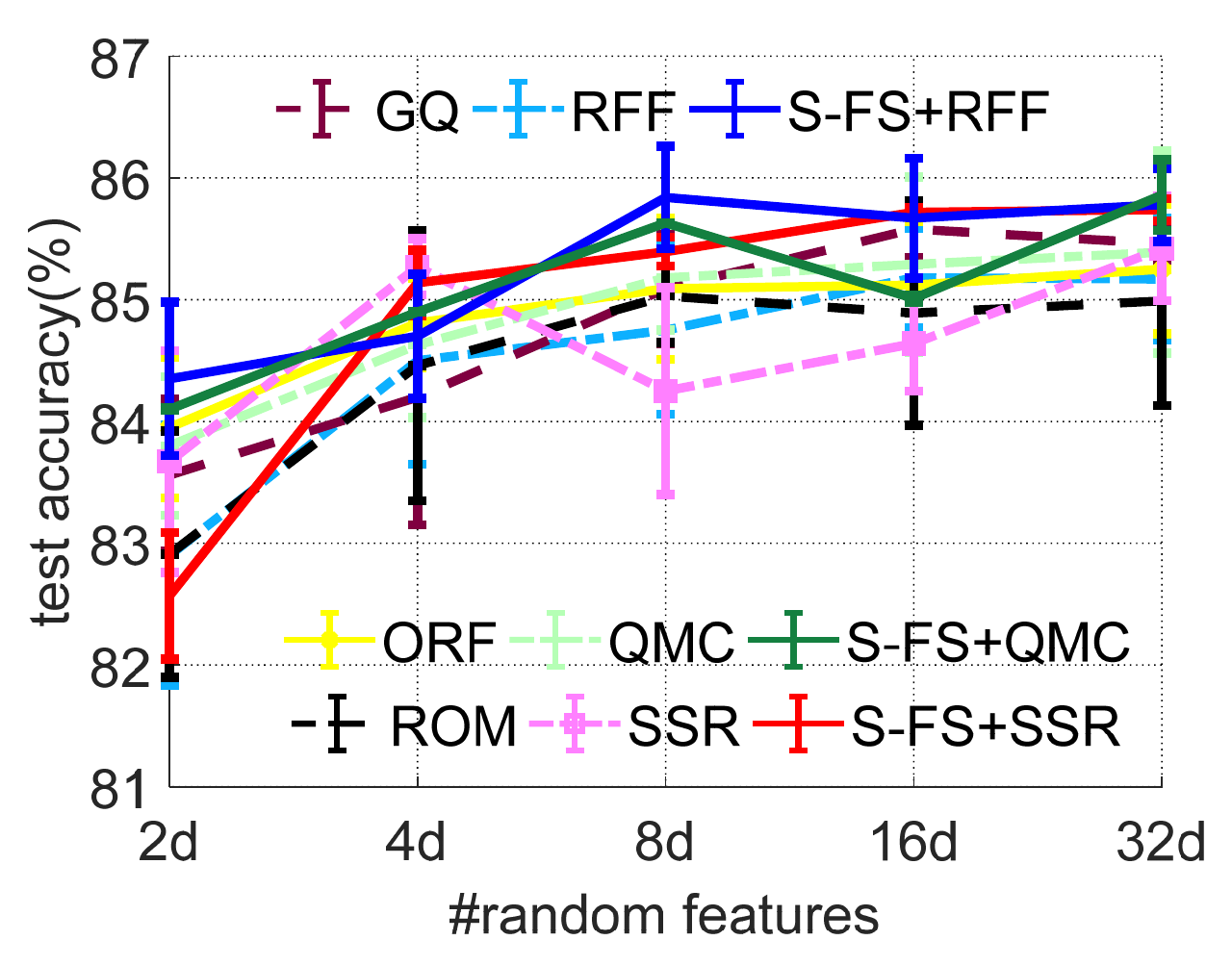}}
	\subfigure[\emph{letter}]{
		\includegraphics[width=0.22\textwidth]{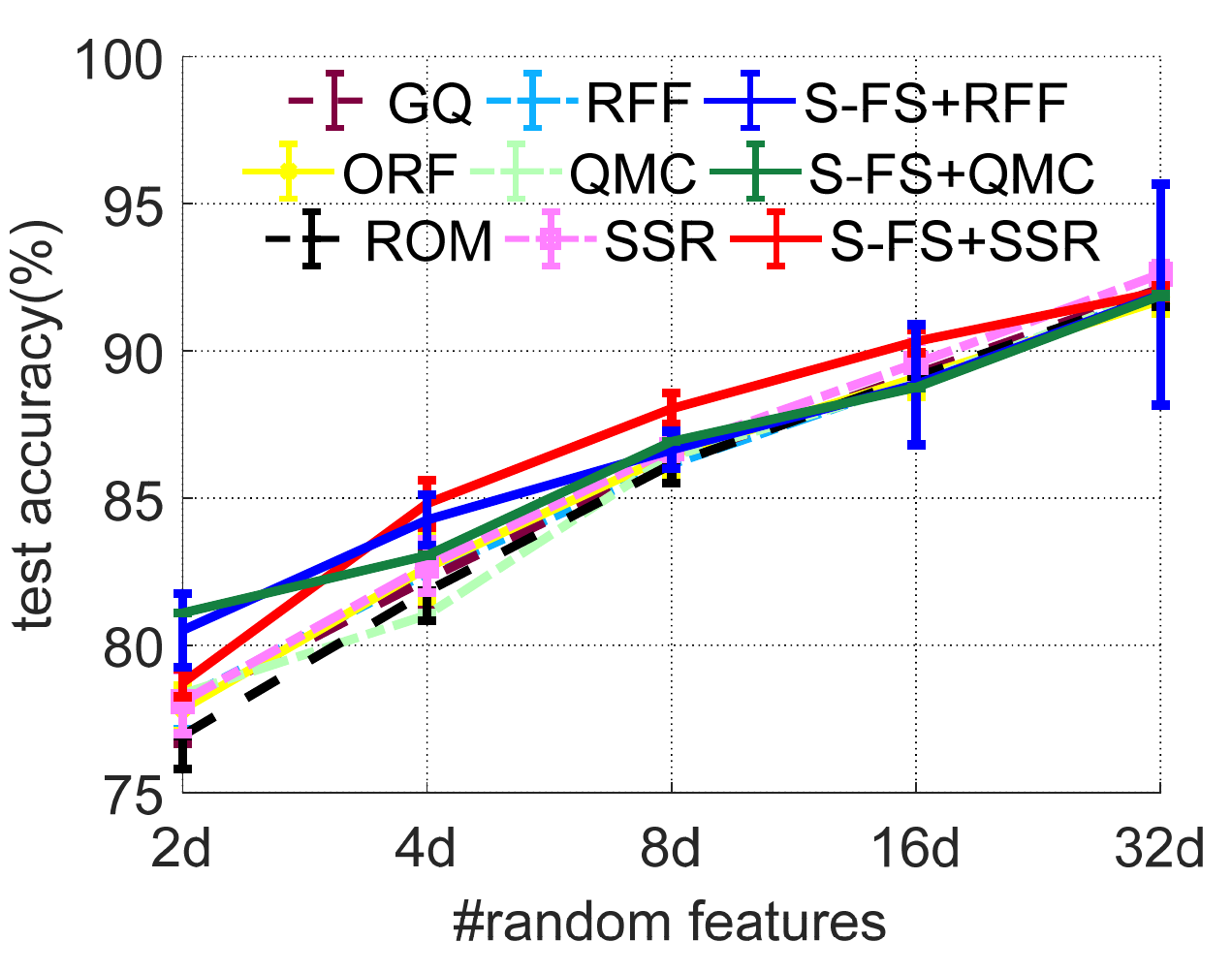}}
	\subfigure[\emph{ijcnn1}]{
		\includegraphics[width=0.22\textwidth]{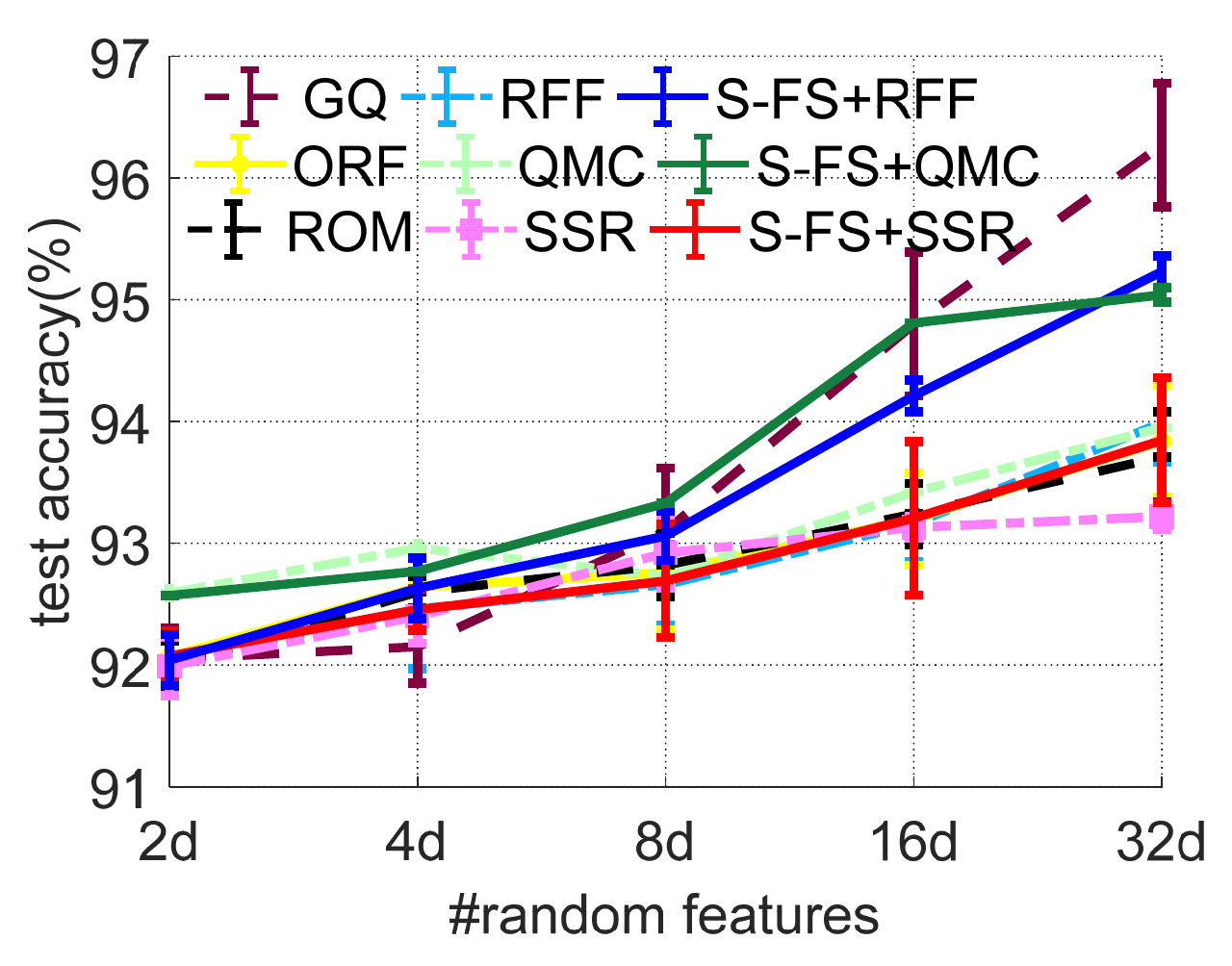}}
	\subfigure[\emph{covtype}]{
		\includegraphics[width=0.22\textwidth]{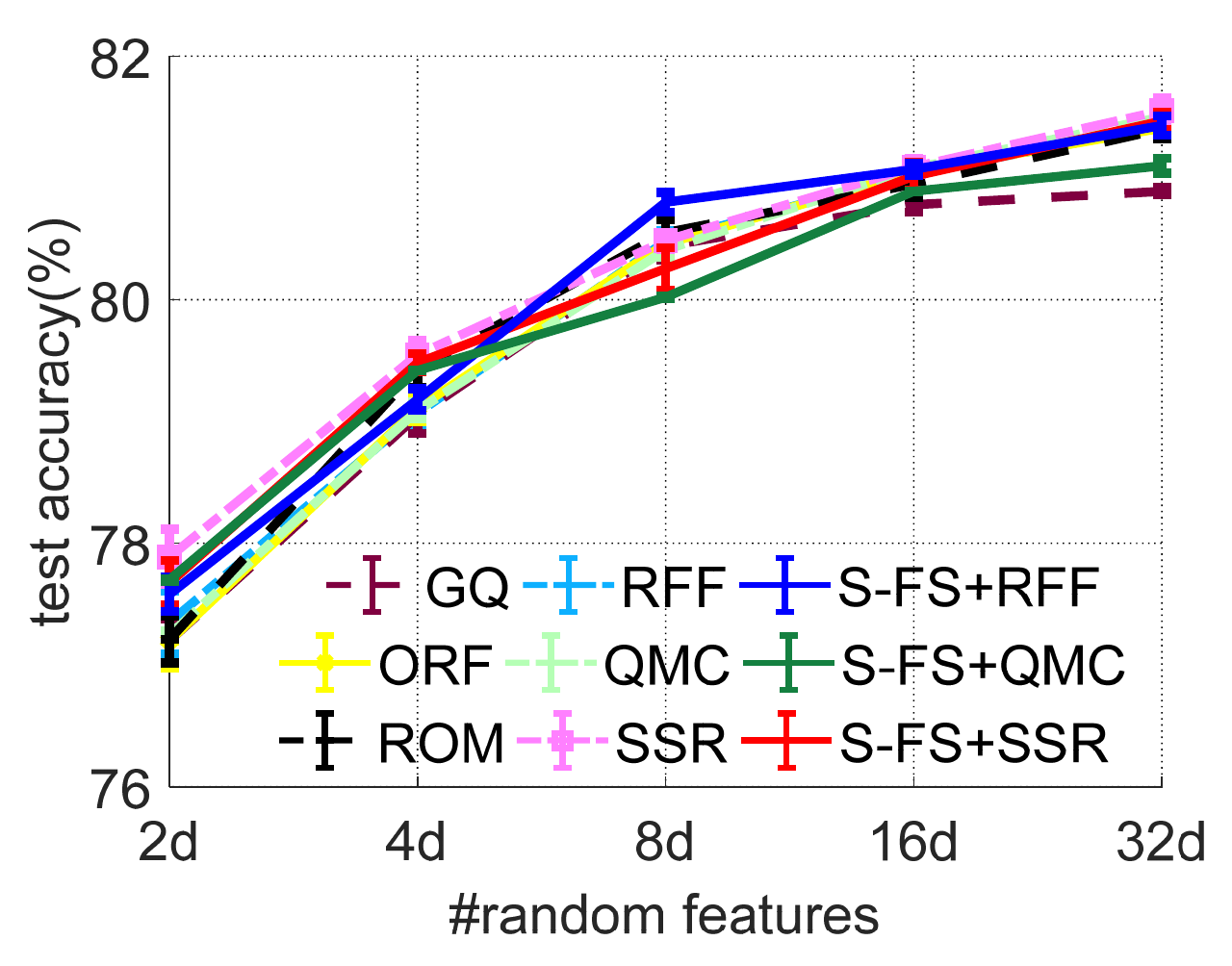}}
	\caption{Kernel approximation (top), time cost (middle), and test accuracy (bottom) across the Gaussian kernel.}\label{fapp}
\end{figure*}

\begin{figure*}[!htb]
	\centering
	
	\subfigure{
		\includegraphics[width=0.22\textwidth]{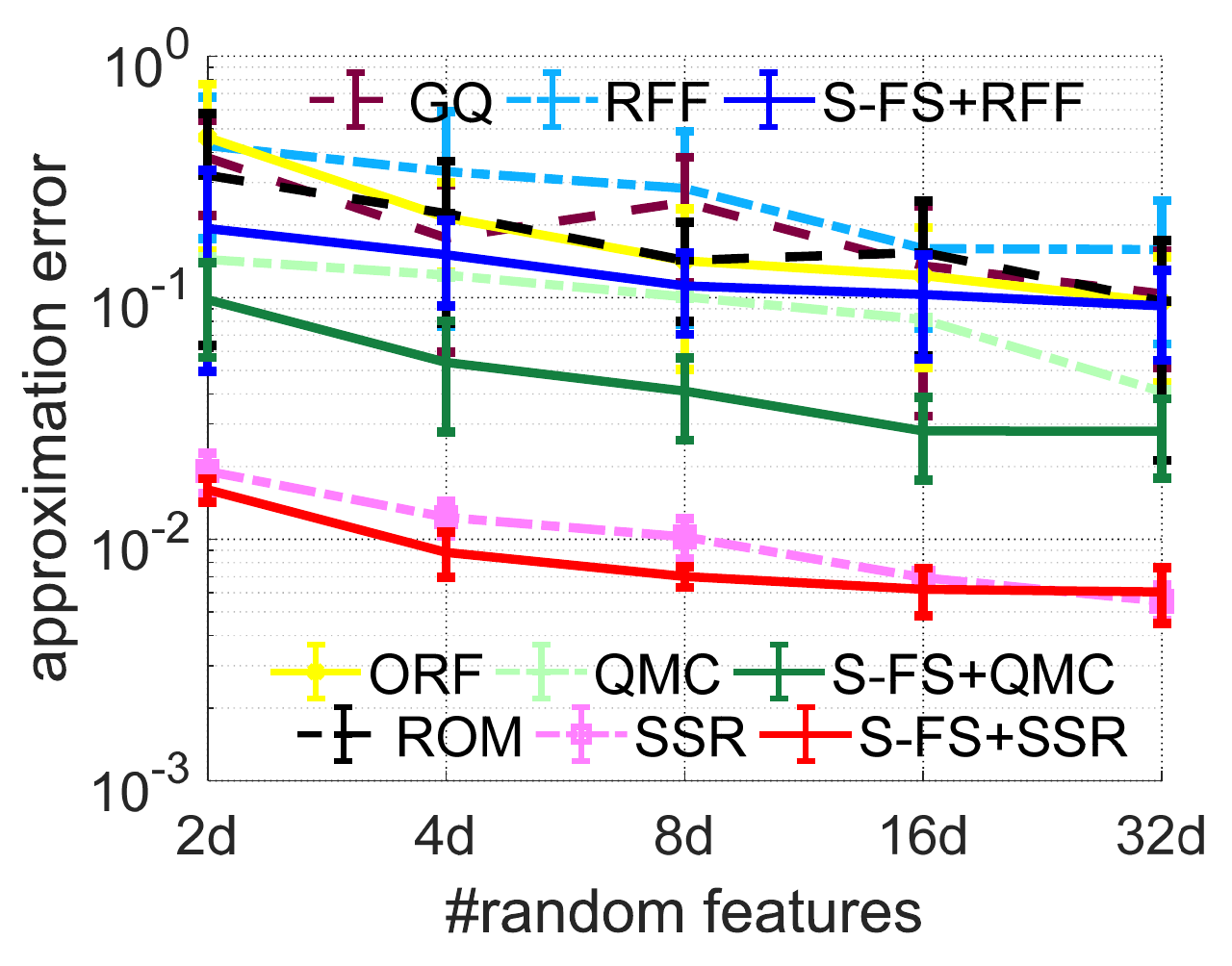}}
	\subfigure{
		\includegraphics[width=0.22\textwidth]{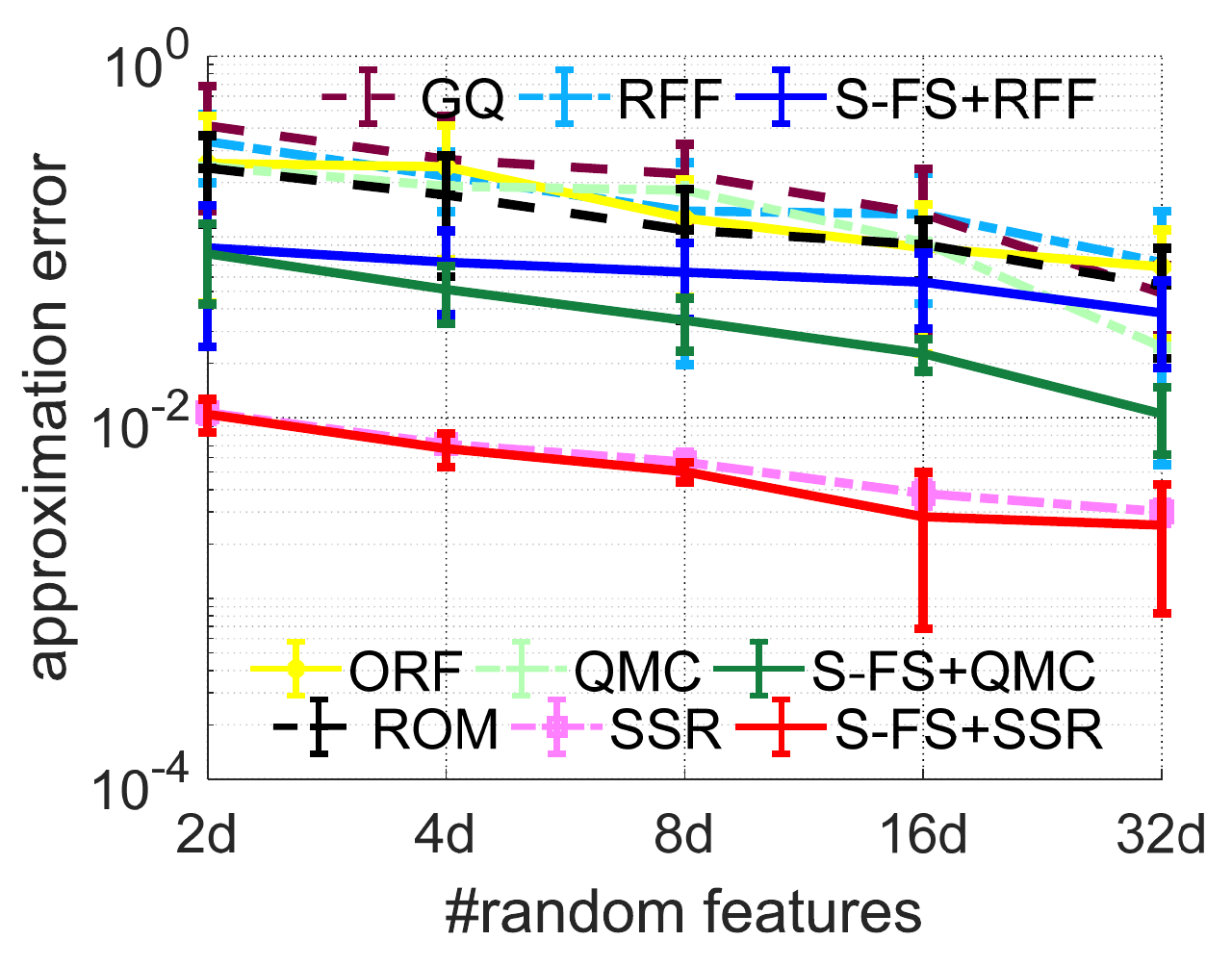}}
	\subfigure{
		\includegraphics[width=0.22\textwidth]{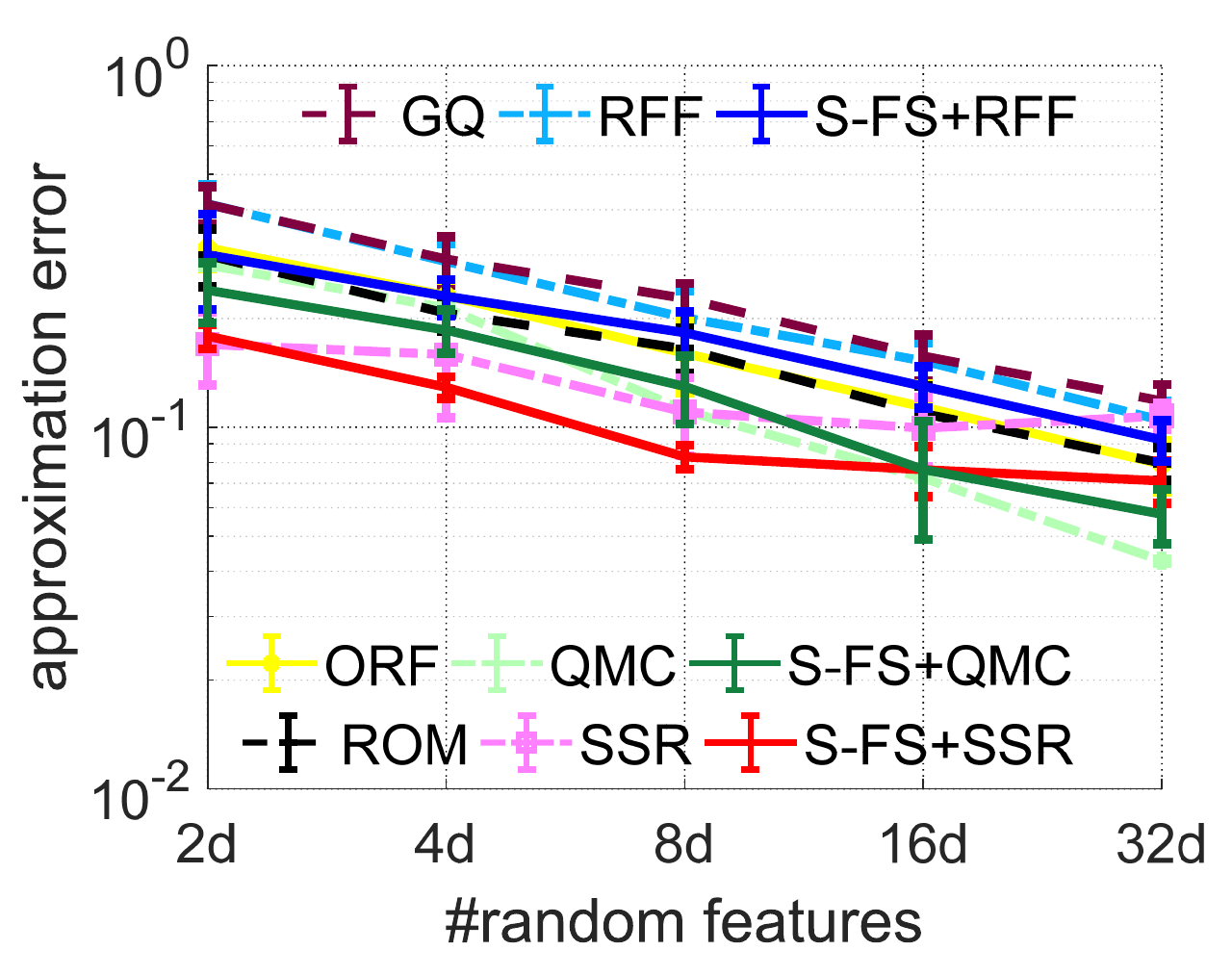}}
	\subfigure{
		\includegraphics[width=0.22\textwidth]{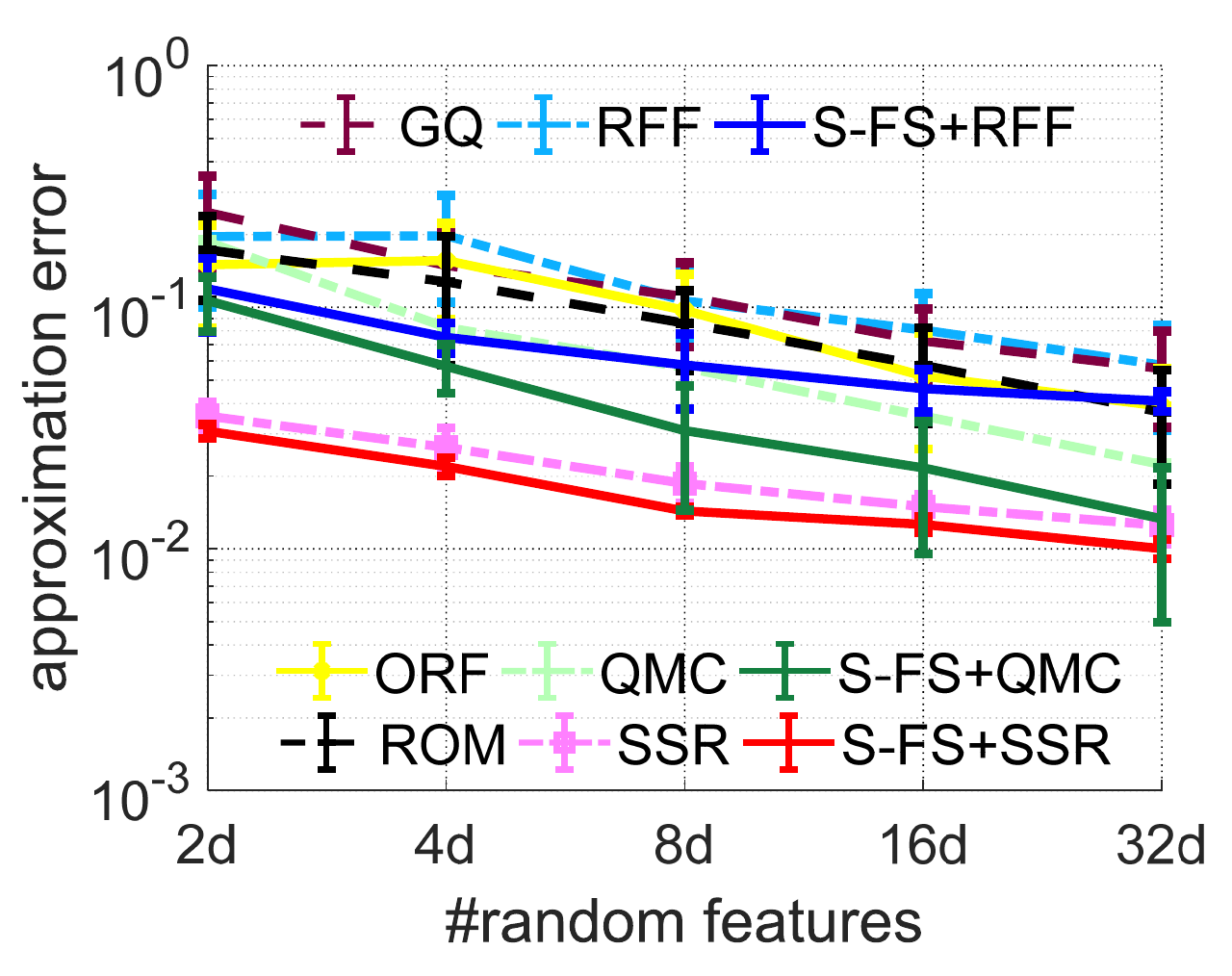}}
	
	\subfigure[\emph{magic04}]{
		\includegraphics[width=0.22\textwidth]{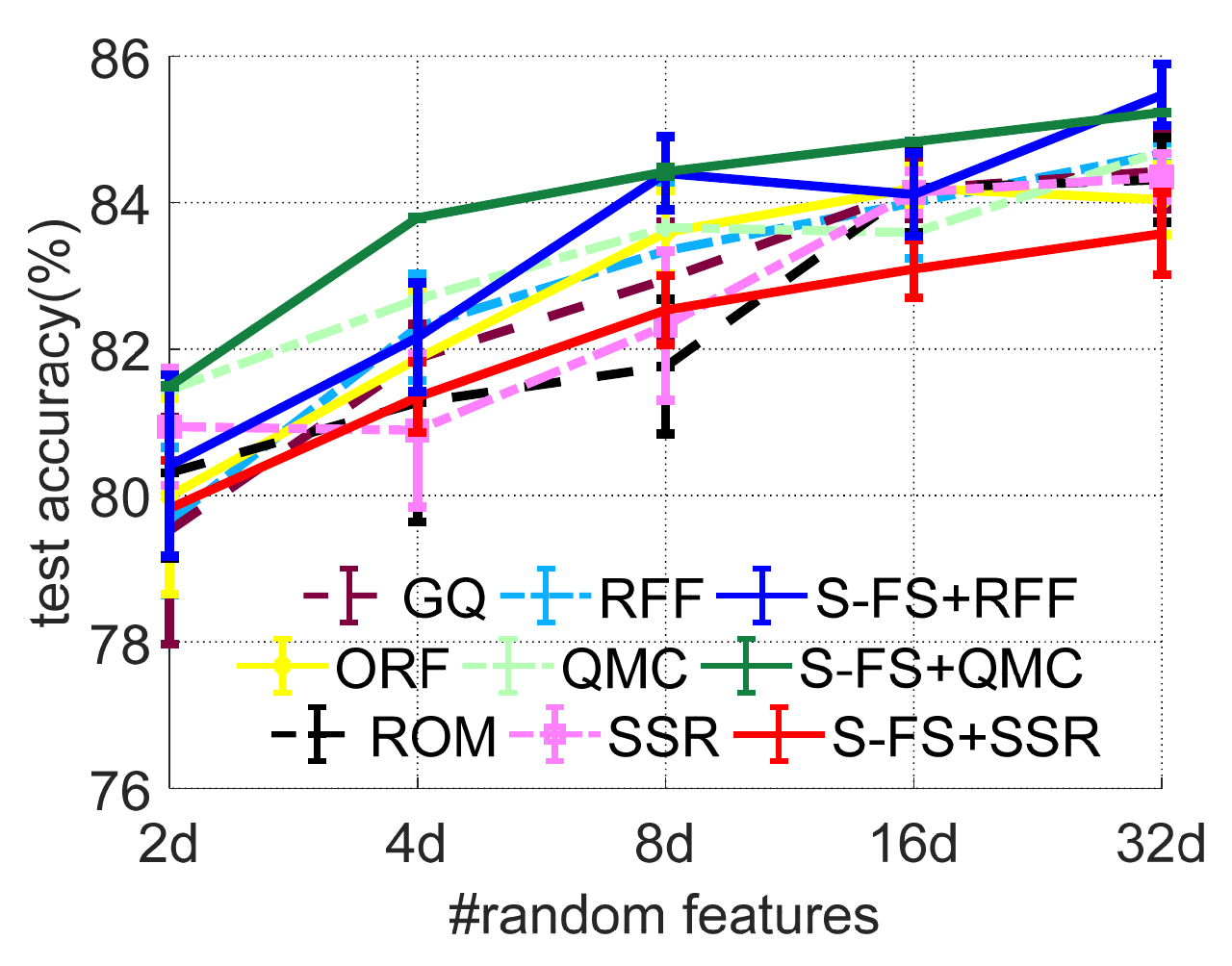}}
	\subfigure[\emph{letter}]{
		\includegraphics[width=0.22\textwidth]{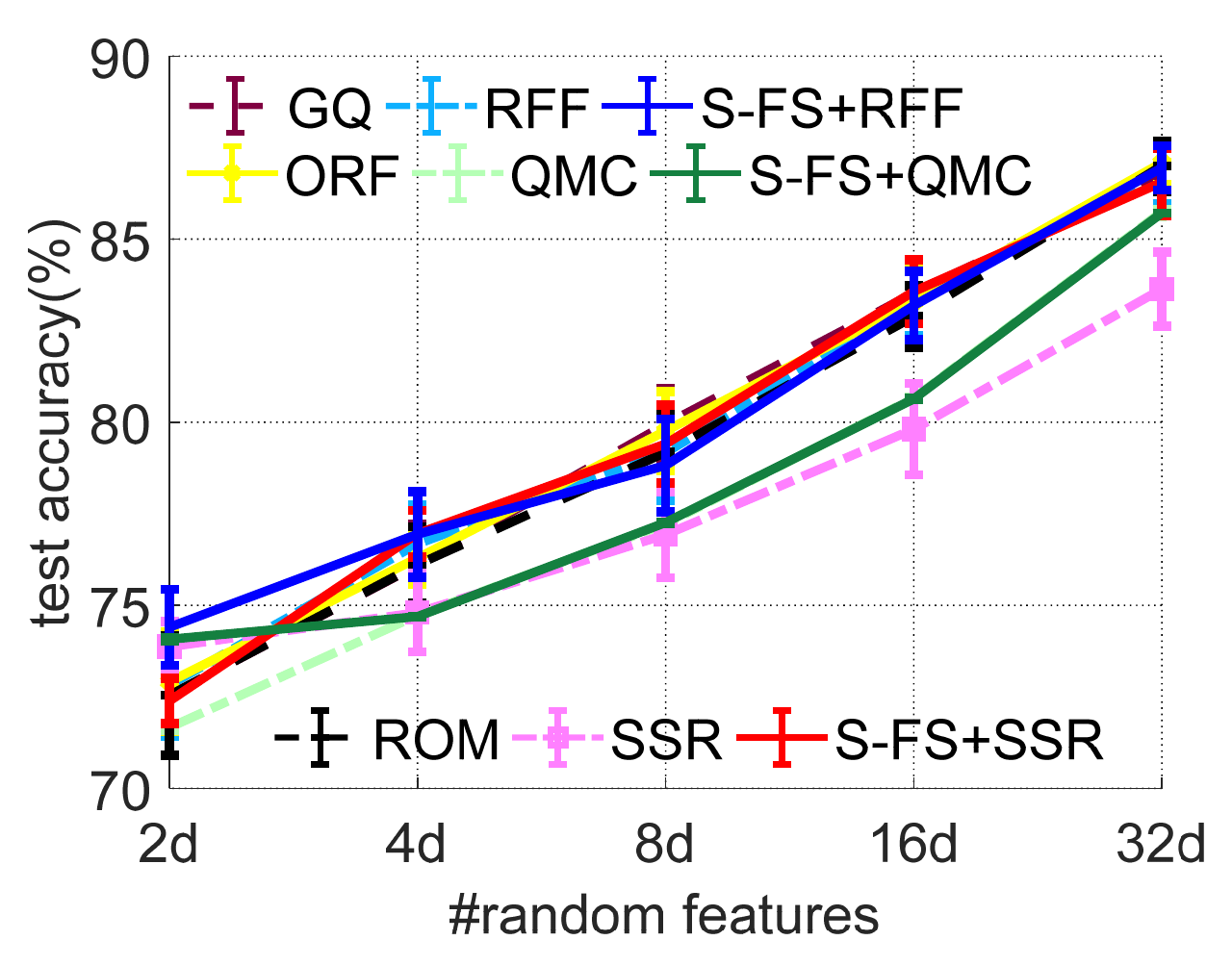}}
	\subfigure[\emph{ijcnn1}]{
		\includegraphics[width=0.22\textwidth]{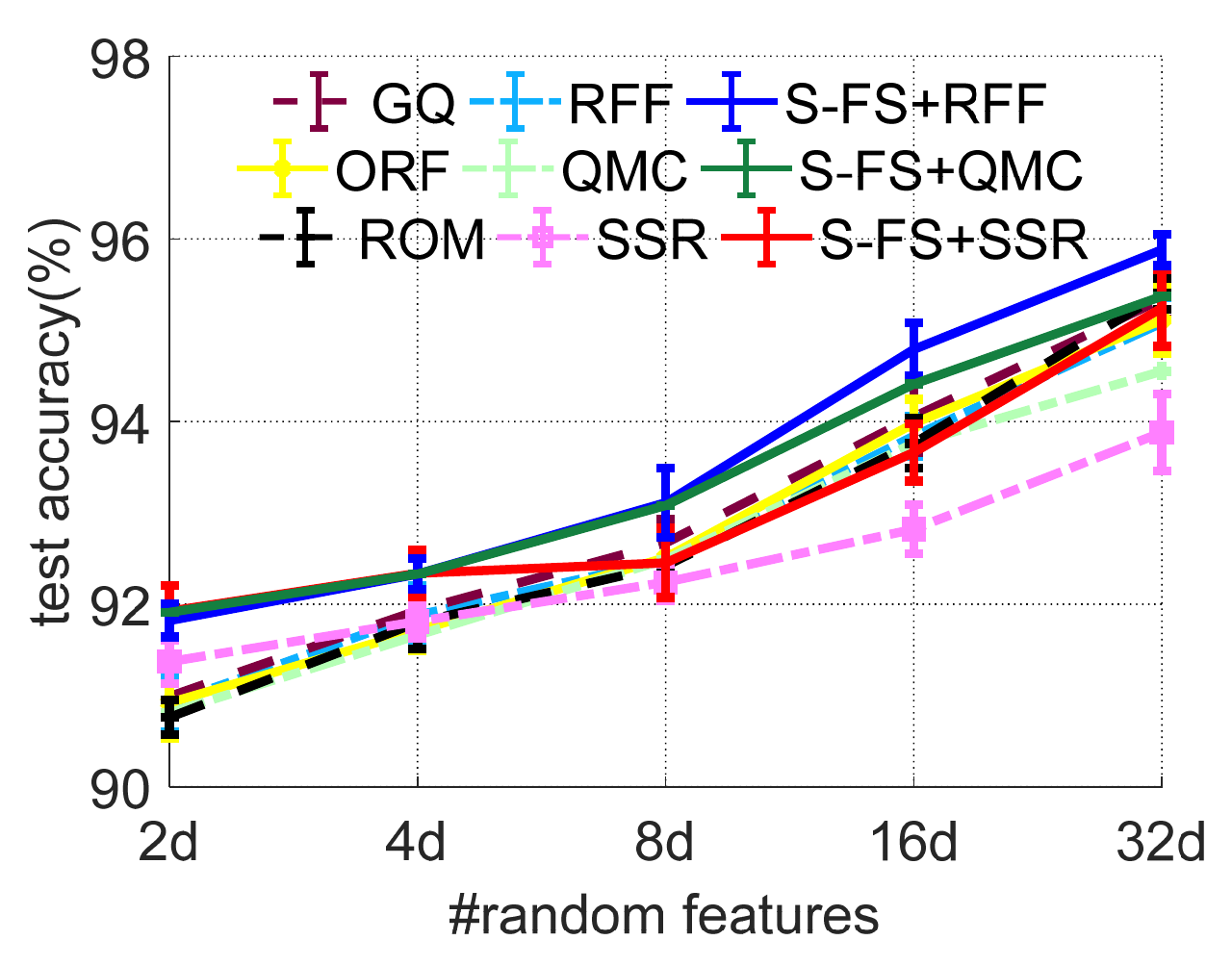}}
	\subfigure[\emph{covtype}]{
		\includegraphics[width=0.22\textwidth]{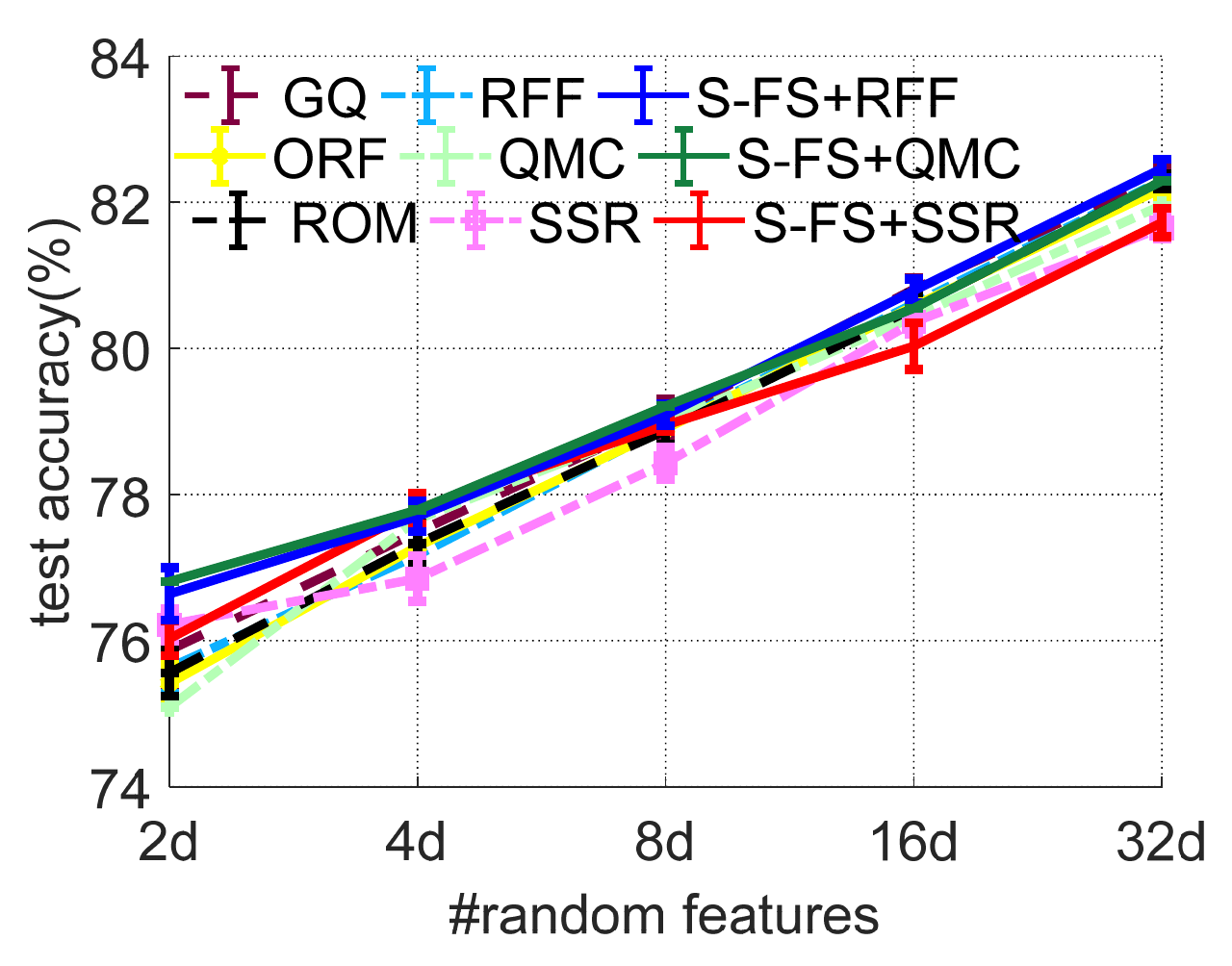}}
	
	\caption{Kernel approximation error (top) and test accuracy (bottom) across the first-order arc-cosine kernel.}\label{fapparc}
\end{figure*}

\subsection{Evaluation for Stochastic Rules}
\label{sec:exp:sto}
Here we evaluate the proposed third-degree S-FS under a dimension adjustment setting, in which the feature dimension in Eq.~\eqref{feamapf} is manually fixed with $D=\{ 2d,4d, 8d, 16d, 32d \}$.
In this case, S-FS generates the feature mapping $\widehat{\Phi}(\cdot) \in \mathbb{R}^{D+4d+2}$, but still achieves the same time/space complexity $\mathcal{O}(Dd)$ with RFF.
We begin with an intuitive comparison of S-FS in Eq.~\eqref{feamapf} against RFF and then conduct a comprehensive experimental evaluation of all the randomized algorithms.

First, to validate the effectiveness of S-FS on variance reduction, Figure~\ref{fvarious} shows the approximation error and the time cost across the Gaussian kernel on the \emph{magic04} data set between S-FS and RFF.
Both of them draw $\{ \bm \omega_i \}_{i=1}^D \sim \mathcal{N}(\bm 0, \bm I_d)$ by Monte-Carlo sampling, so S-FS under this setting is termed as ``S-FS+RFF". 
It can be noticed that, admittedly, ``S-FS+RFF" takes a little more time than RFF on generating the feature mapping. However, it achieves significant improvement on RFF in terms of the approximation quality, which demonstrates the effectiveness of the used control variates technique in Eq.~\eqref{feamapf}.
Besides, we observe that, the variance reduction effect weakens or even disappears when $D$ is large.
One reason might be that, the variance of S-FS converges to that of RFF at a fast $\mathcal{O}(1/(Dd))$ rate, as demonstrated by Theorem~\ref{thmvar}.

In the next, we present a comprehensive evaluation of the proposed S-FS rule with other representative approaches.
To purse a better approximation performance, apart from the original Monte-Carlo sampling in S-FS, we also incorporate various sampling strategies into S-FS: $\{ \bm \omega_i \}_{i=1}^D$ in  Eq.~\eqref{feamapf} are obtained by QMC and SSR, termed as ``S-FS+QMC" and ``S-FS+SSR" respectively.

\noindent {\bf Results on Gaussian kernel:} 
Figure~\ref{fapp} shows the approximation error, time cost, and test accuracy (mean$\pm$std.) of all the compared algorithms across the Gaussian kernel under different feature dimensionality with $D=\{ 2d, 4d, 8d, 16d, 32d \}$.
We find that, when compared to the original RFF, QMC, SSR, our stochastic rules including ``S-FS+RFF", ``S-FS+QMC", ``S-FS+SSR" manifest significant reduction on the approximation error, respectively.
In terms of time complexity, due to the used control variates technique, our stochastic rules take more time than the original RFF/ORF/QMC/ROM/GQ.
Among these three sampling strategies, ``S-FS+RFF" and ``S-FS+QMC" take the similar time cost on generating the feature mapping, achieving the same time complexity $\mathcal{O}(Dd)$ with RFF.
However, ``S-FS+SSR" is relatively time-consuming on the \emph{ijcnn1} and \emph{covtype} datasets as SSR itself requires more time to obtain random orthogonal matrices in large scale situations.

As mentioned before, the compared algorithms achieve the similar test accuracy in the deterministic setting. There is almost no distinct difference between these approaches on the final classification accuracy under varying feature dimensionlity.

\noindent {\bf Results on arc-cosine kernel:} 
Figure~\ref{fapparc} shows the approximation error and test accuracy of all the compared algorithms across the first-order arc-cosine kernel on these four datasets.
It can be found that, the compared algorithms across the arc-cosine kernel are generally inferior to them across the Gaussian kernel in terms of the approximation quality and generalization performance.

In sum, we experimentally validate that our stochastic rules are unbiased and achieve variance reduction in terms of the approximation error. 
Since ``S-FS+QMC" is more efficient than SSR on these datasets, and thus is demonstrated to achieve a good trade-off between the approximation quality and time cost.
\vspace{-0.2cm}

\section{Conclusion}
We present deterministic/stochastic quadrature methods D-FS/S-FS based on the fully symmetric interpolatory rule to approximate the Gaussian kernel and the first-order arc-cosine kernel via the integration representation~\eqref{originte}.
Our third/fifth-degree deterministic rules achieve promising approximation quality while retaining the same time cost with RFF.
Our S-FS rules exhibit variance reduction on the approximation error due to the used control variates technique, and performs well on real datasets.
By studying the relations among the third-degree quadrature based methods, our unified framework mainly demonstrates that, 1) D-FS recovers SGQ by choosing suitable parameters; 2) SSR can be regarded as a doubly stochastic version of D-FS via a \emph{random} projection scheme and a \emph{randomized} generator.

\section*{Acknowledgements}
The research leading to these results has received funding from the European Research Council under the European Union's Horizon 2020 research and innovation program / ERC Advanced Grant E-DUALITY (787960). This paper reflects only the authors' views and the Union is not liable for any use that may be made of the contained information.
This work was supported in part by Research Council KU Leuven: Optimization frameworks for deep kernel machines C14/18/068; Flemish Government: FWO projects: GOA4917N (Deep Restricted Kernel Machines: Methods and Foundations), PhD/Postdoc grant. This research received funding from the Flemish Government (AI Research Program). 
This work was supported in part by Ford KU Leuven Research Alliance Project KUL0076 (Stability analysis and performance improvement of deep reinforcement learning algorithms), EU H2020 ICT-48 Network TAILOR (Foundations of Trustworthy AI - Integrating Reasoning, Learning and Optimization), Leuven.AI Institute; and in part by the National Natural Science Foundation of China 61977046, in part by National Science Foundation grants CCF-1657420 and CCF-1704828, and in part by SJTU Global Strategic Partnership Fund (2020 SJTU-CORNELL) and Shanghai Municipal Science and Technology Major Project (2021SHZDZX0102).

\clearpage

\appendices

The outline of the appendix is stated as follows. First, the fifth-degree D-FS is derived in Appendix~\ref{sec:fd}. Then we analyze the statistical properties (unbiasness and variance reduction) of our third-degree S-FS in Appendix~\ref{sec:sgsir}. Besides, we give the proof of Theorem~\ref{thirdproj} in Appendix~\ref{sec:proofthirdproj} that reveals the relations between D-FS and the stochastic spherical rule.
\section{Fifth-degree Rule}
\label{sec:fd}
When choosing $m=2$ in Eq.~\eqref{IQmdfnew}, we obtain a fifth-degree rule $Q^{(2,d)}$ with $\| \bm p \|_1 \leq 2$ to further improve the approximation accuracy. 
To derive the fifth-degree rule, we cast it in three cases, i.e., $\| \bm p \|_1 = 0$, $\| \bm p \|_1 = 1$, and $\| \bm p \|_1 = 2$.

If $\| \bm p \|_1 = 0$, we have $p_i = 0$, $\bm \lambda = \bm 0$, and $K=0$.
Then the weight $a^{(2,d)}_0$ is
\begin{equation}\label{a02d}
\begin{aligned}  a_{0}^{(2,d)} &= \sum_{\|\bm u \|_1 \leq 2} \prod_{i=1}^{d} \frac{b_{u_{i}}}{\prod_{j=0, \neq 0}^{u_{i}}\left(\lambda_{0}^{2}-\lambda_{j}^{2}\right)} \\
& = 1-\frac{d}{\lambda_{1}^{2}}+\frac{d(d-1)}{2 \lambda_{1}^{4}}+\frac{d (3-\lambda_1^2)}{\lambda_{1}^{2} \lambda_{2}^{2}}\,,
\end{aligned}
\end{equation}
where $b_2 = 3 - \lambda_1^2$ is obtained by Eq.~\eqref{bi}. In our derivation, $\| \bm u \|_1 \leq 2$ is cast into three cases: $\bm u = \bm 0$, $\| \bm u\|_1 = 1$, and $\| \bm u\|_1 = 2$ for calculation.  

If $\| \bm p \|_1 = 1$, only one element of $\bm p$ is 1 and the remaining are zero.
We thereby have $K=1$ and $\bm \lambda = \lambda_1 \bm e_i$ with $i=1,2,\dots, d$, where $\bm e_i$ is a unit vector with the $i$-th element being 1.
Without loss of generality, assuming $\bm p =[1,0,\cdots,0]$, the weight $a^{(2,d)}_1$ is computed as
\begin{equation}\label{a12d}
\begin{aligned} 
a_{1}^{(2, d)} & = \frac{1}{2} \sum_{\| \bm u \|_1 \leq 1} \prod_{i=1}^{d} \frac{b_{u_{i}+p_{i}}}{\prod_{j=0, \neq p_{i}}^{u_{i}+p_{i}}\left(\lambda_{p_{i}}^{2}-\lambda_{j}^{2}\right)} \\
& =\frac{1}{2\lambda_{1}^{2}}+\frac{3-\lambda_{1}^{2}}{2\lambda_{1}^{2}\left(\lambda_{1}^{2}-\lambda_{2}^{2}\right)}-\frac{d-1}{2\lambda_{1}^{4}}\,,
\end{aligned}
\end{equation}
where $\| \bm u \|_1 \leq 1$ is cast into two cases: $\bm u = \bm 0$ and $\| \bm u \|_1 = 1$ for derivation.

If $\| \bm p \|_1 = 2$, the derivation is a little complex and we cast it into two cases.
One is that there are two elements in $\bm p$ being 1, i.e., $p_i = p_j = 1$ with $i \neq j$.
The other is only one element of $\bm p$ being 2, i.e., $p_i=2$.
For the $p_i = p_j = 1$ case, we have $K=2$ and $\bm \lambda = \lambda_1 \bm s_{l}^+$ or $\bm \lambda = \lambda_1 \bm s_{l}^-$, where the point sets of $\bm s_{l}^+$ and $\bm s_{l}^-$ are given by
\begin{equation*}
\begin{split}
& \{ \bm s_{l}^+ \}_{l=1}^{d(d-1)/2} := \{ \bm e_i + \bm e_j: i < j,~i,j = 1,2,\cdots,d \} \\
& \{ \bm s_{l}^- \}_{l=1}^{d(d-1)/2} := \{ \bm e_i - \bm e_j: i < j,~i,j = 1,2,\cdots,d \} \\
\end{split}
\end{equation*}
Without loss of generality, assuming $\bm p =[1,1,0,\cdots,0]$, the weight $a_2^{(2,d)}$ is
\begin{equation}\label{a22d}
\begin{split} 
a_{2}^{(2, d)} &=\frac{1}{4} \prod_{i=1}^{d} \frac{b_{p_{i}}}{\prod_{j=0, \neq p_{i}}^{p_{i}}\left(\lambda_{p_{i}}^{2}-\lambda_{j}^{2}\right)} \\ &=\frac{1}{4} \left[ \frac{b_{1}}{\left(\lambda_{1}^{2}-\lambda_{0}^{2}\right)} \right]^2 =\frac{1}{4 \lambda_{1}^{4}} \,.
\end{split}
\end{equation}

For the $p_i =2$ case, we have $K=1$ and $\bm \lambda = \lambda_2 \bm e_i$. Without loss of generality, assuming  $\bm p=[2,0,\cdots,0]$, the weight $a_3^{(2,d)}$ is computed as
\begin{equation}\label{a32d}
\begin{split} 
a_{3}^{(2, d)} &=\frac{1}{2} \prod_{i=1}^{d} \frac{b_{p_{i}}}{\prod_{j=0, \neq p_{i}}^{p_{i}}\left(\lambda_{p_{i}}^{2}-\lambda_{j}^{2}\right)}  = \frac{3- \lambda_1^2}{2 \lambda_2^2 (\lambda_2^2 - \lambda_1^2)} \,.
\end{split}
\end{equation}
Accordingly, combining the derived weights in Eqs.~\eqref{a02d},~\eqref{a12d},~\eqref{a22d}, and ~\eqref{a32d},
the fifth-degree full symmetric interpolatory rule is
\begin{equation*}
\begin{split} 
&Q^{(2, d)}(f)= a_{0}^{(2, d)} f(\bm 0)+a_{1}^{(2, d)} \sum_{i=1}^{d}\left[f\left(\lambda_{1} \bm{e}_{i}\right)+f\left(-\lambda_{1} \bm{e}_{i}\right)\right] \\ &+a_{2}^{(2, d)} \sum_{i=1}^{d(d-1) / 2} \!\!\! \!\left[f\left(\lambda_{1} \bm{s}_{i}^{+}\right) \!+\! f\left(-\lambda_{1} \bm{s}_{i}^{+}\right)\!+\! f\left(\lambda_{1} \bm{s}_{i}^{-}\right) \!+\! f\left(-\lambda_{1} \bm{s}_{i}^{-}\right) \right] \\ &+a_{3}^{(2, d)} \sum_{i=1}^{d}\left[f\left(\lambda_{2} \bm{e}_{i}\right)+f\left(-\lambda_{2} \bm{e}_{i}\right)\right] \\
&:=a_{0}^{(2, d)} f(\bm 0) + \sum_{j=1}^{2d} \left( a_{1}^{(2, d)}f(\bm P_{j,1}) 
+ a_{3}^{(2, d)} f(\bm P_{j,3}) \right) \\
&+ a_{2}^{(2, d)} \sum_{j=1}^{2d(d-1)} f(\bm P_{j,2})\,,
\end{split}
\end{equation*}
where 
\begin{equation*}
\bm P_{j,1} = \left\{
\begin{array}{rcl}
\begin{split}
& \lambda_1 \bm{e}_{i};~ a_{i} =\frac{1}{2d};~ 1  \leq i \leq d \\
& -\lambda_1 \bm{e}_{i-d};~ d+1 \leq i \leq 2 d
\end{split}
\end{array}\right.
\end{equation*}

\begin{equation*}
\bm P_{j,3} = \left\{
\begin{array}{rcl}
\begin{split}
& \lambda_2 \bm{e}_{i};~ 1  \leq i \leq d \\
& -\lambda_2 \bm{e}_{i-d};~ d+1 \leq i \leq 2 d
\end{split}
\end{array}\right.
\end{equation*}

\begin{equation*}
\bm{P}_{j, 2}=\left\{\begin{array}{ll}{\lambda_{1}\left(\bm{e}_{i}+\bm{e}_{t}\right)} & {i, t=1, \ldots, d ; i<t} \\ {\lambda_{1}\left(\bm{e}_{i}-\bm{e}_{t}\right)} & {i, t=1, \ldots, d ; i<t} \\ {\lambda_{1}\left(-\bm{e}_{i}+\bm{e}_{t}\right)} & {i, t=1, \ldots, d ; i<t} \\ {\lambda_{1}\left(-\bm{e}_{i}-\bm{e}_{t}\right)} & {i, t=1, \ldots, d ; i<t}\end{array}\right.
\end{equation*}
Similar to the third-degree rule, we also choose $\lambda_i$ by successive extensions of the one-dimensional 3-point Gauss–Hermite rule.

\section{Statistical Guarantees of Stochastic Interpolatory Rules}
\label{sec:sgsir}
This section includes three parts: 
\begin{itemize}
	\item in Section~\ref{sec:theun}, we prove Theorem~\ref{theun}, that is, our third-degree S-FS $\bar{R}_1(f)$ is an unbiased third degree rule for $I_d(f)$.
	\item in Section~\ref{sec:thmvar}, we prove Theorem~\ref{thmvar} that gives the variance of our third-degree stochastic interpolatory rule, i.e., $\mathbb{V}[\bar{R}_1(f, \bm \omega)]$.
\end{itemize}

\subsection{Proof of Theorem~\ref{theun}}
\label{sec:theun}
\begin{proof}
	
	We compute $I_d(\tilde{a}_{\bm p}^{(1,d)}(\bm \omega))$ as follows. 
	For $\tilde{a}_0^{(1,d)}$
	\begin{equation*}
	I_d(\tilde{a}_0^{(1,d)}) = \int_{\mathbb{R}^d} \left( 1 - \frac{\sum_{i=1}^{d} \omega_i^2}{\lambda_1^2} \right) \mu(\mathrm{d} \bm \omega) = a_{0}^{(1, d)}.
	\end{equation*}
	For $\tilde{a}_1^{(1,d)}$, we have
	\begin{equation*}
	I_d(\tilde{a}_1^{(1,d)}) = \frac{1}{2d\lambda_1^2}
	\int_{\mathbb{R}^d} \left( \sum_{i=1}^d \omega_i^2 \right) \mu(\mathrm{d} \bm \omega) = a_{1}^{(1, d)}.
	\end{equation*}
	So we have $a_{\bm p}^{(1, d)} = I_d[\tilde{a}_{\bm p}^{(1,d)}(\bm \omega)]$, and thus $Q^{(1,d)}(f) = I_d[M^{(1,d)}(f, \bm \omega)]$.
	Due to $I(f) = \mathbb{E}_{\bm \omega \sim \mu}[f(\bm \omega)]$, we have $Q^{(1,d)}(f) = \mathbb{E}_{\bm \omega \sim \mu}[M^{(1,d)}(f, \bm \omega)]$.
	Based on this, the expectation of $R_1(f, \bm \omega)$ is 
	\begin{equation*}
	\begin{split}
	\mathbb{E}_{\bm \omega} [R_1(f, \bm \omega)] & \!=\!\! \mathbb{E} [f(\bm \omega)] \!-\! \mathbb{E} [M^{(1,d)}(f, \bm \omega)]  \!+\! \mathbb{E} \{ I_d[M^{(1,d)}(f)] \} \\
	& = I_d[f(\bm \omega)] - Q^{(1,d)}(f) + Q^{(1,d)}(f) \\
	&= I_d(f)\,.
	\end{split}
	\end{equation*}
	Accordingly, due to $\{ \bm \omega_i \}_{i=1}^D \sim \mu$, the average $\bar{R}_1(f)$ is unbiased for $I_d(f)$.
	
	Besides, if we choose $f(\bm \omega) = \bm \omega^{2 \bm u}$ with $\| \bm u \|_1 \leq 1$, then we have $R_1(f, \bm \omega) = Q^{(1,d)}(f) = I_d[M^{(1,d)}(f, \bm \omega)]$.
	If we choose $f(\bm \omega) = \bm \omega^{\bm u}$ in which at least one element of $\bm u$ is odd, we have $R_1(f, \bm \omega) = 0$.
	So it means that $R_1(f, \bm \omega)$ is a third-degree rule for $I_d(f)$.
	Hence, $\bar{R}_1(f, \bm \omega)$ is an unbiased third-degree stochastic rule for $I_d(f)$, which concludes the proof.
\end{proof}

\subsection{Proof of Theorem~\ref{thmvar}} 
\label{sec:thmvar}

This section aims to prove Theorem~\ref{thmvar} including two parts.
In Section~\ref{sec:lemmafw}, we present Lemma~\ref{lemmafw} that is used to prove Theorem~\ref{thmvar}.
The proof of Theorem~\ref{thmvar} can be found in Section~\ref{sec:proofthmvar}.

\subsubsection{Proof of Lemma~\ref{lemmafw}}
\label{sec:lemmafw}
To aid the proof of Theorem~\ref{thmvar}, we need the following lemma.
\begin{lemma}\label{lemmafw}
	Denote $\bm \omega = [\omega_1, \omega_2, \cdots, \omega_d]^{\!\top} \sim \mathcal{N}(\bm 0, \bm I_d)$,  $\bm z := {\bm x - \bm y}/{\sigma} = [z_1, z_2, \cdots, z_d]^{\!\top}$, and $f(\bm \omega) = \cos(\bm \omega^{\!\top} \bm z) $, we have
	\begin{equation*}
	\mathbb{E}_{\bm \omega} \left(f(\bm \omega) \sum_{j=1}^{d} \omega_j^2\right) = e^{-\frac{\| \bm z\|_2^2}{2}} (d-\| \bm z\|_2^2)\,.
	\end{equation*}
\end{lemma}

\begin{proof}
	We expand $\mathbb{E}_{\bm \omega}\left(f(\bm \omega) \sum_{j=1}^{d} \omega_j^2\right)$ as
	\begin{equation*}
	\mathbb{E}_{\bm \omega}\left(f(\bm \omega) \sum_{j=1}^{d} \omega_j^2\right) = \sum_{j=1}^{d} \mathbb{E}_{\bm \omega} \left[\omega_j^2 \cos(\bm \omega^{\!\top} \bm z)\right] \,.
	\end{equation*}
	The $j$-th term $\omega_j^2 \cos(\bm \omega^{\!\top} \bm z)$ can be reformulated as 
	\begin{equation}\label{omegaj2}
	\begin{split}
	& \omega_j^2 \cos(\bm \omega^{\!\top} \bm z) = \omega_j^2 \cos\left(\omega_j z_j + \sum_{t=1, \neq j }^d \omega_t z_t \right) \\
	&=\! \omega_j^2 \!\!\left[  \cos(\omega_j z_j ) \cos\!\bigg(\! \sum_{t=1, \neq j }^d \!\!\!\! \omega_t z_t \!\! \bigg)  \!\!-\! \sin(\omega_j z_j ) \! \sin \!\bigg(\! \sum_{t=1, \neq j }^d \!\!\! \omega_t z_t \!\bigg)\! \right]\!.
	\end{split}
	\end{equation}
	Now we compute $\mathbb{E}_{\omega_j}[\omega_j^2 \cos(\omega_j z_j ) ]$ as follows.
	\begin{equation}\label{ecos2}
	\begin{split}
	\mathbb{E}[\omega_j^2 \cos(\omega_j z_j ) ] & = \frac{1}{\sqrt{2\pi}} \int_{-\infty}^{\infty} \omega_j^2 \cos(\omega_j z_j) e^{-\frac{\omega_j^{\!\top} \omega_j}{2}} \mathrm{d} \omega_j \\
	& = \mathrm{Re} \left( \int_{-\infty}^{\infty} \frac{1}{\sqrt{2\pi}} \omega_j^2 e^{-\frac{z_j^2}{2}} e^{-\frac{(\omega_j - \mathrm{i} z_j)^2}{2}} \mathrm{d} \omega_j \! \right) \\
	& = e^{-\frac{z_j^2}{2}} \mathbb{E}_{\omega_j} \left( \omega_j^2 \right)~\mbox{with}~\omega_j \sim \mathcal{N}(\mathfrak{i}z_j,1)\\
	& = e^{-\frac{z_j^2}{2}} \left[1+ (\mathfrak{i} z_j )^2 \right] = e^{-\frac{z_j^2}{2}} (1 - z_j^2)\,,
	\end{split}
	\end{equation}
	where $\mathfrak{i}$ denotes the imaginary unit.
	Similarly, $\mathbb{E}[\omega_j^2 \sin(\omega_j z_j ) ]$ can be computed as
	\begin{equation}\label{esin2}
	\mathbb{E}[\omega_j^2 \sin(\omega_j z_j )] \!=\! \mathrm{Im} \left( \int_{-\infty}^{\infty} \frac{\omega_j^2 e^{-\frac{z_j^2}{2}}}{\sqrt{2\pi}} e^{-\frac{(\omega_j - \mathfrak{i} z_j)^2}{2}} \mathrm{d} \omega_j \right)  \!=\! 0 \,.
	\end{equation}
	Besides, by virtue of the above derivation, or directly using Bochner theorem \cite{bochner2005harmonic} for the Gaussian kernel $\mathbb{E}_{\bm \omega}[\cos(\bm \omega^{\!\top} \bm z)] = e^{-\| \bm z\|_2^2/2}$, we have 
	\begin{equation}\label{bochnercos}
	\mathbb{E}_{\{\omega_1, \cdots ,\omega_d\}\backslash \omega_j } \left[ \cos\left(\sum_{t=1, \neq j }^d \omega_t z_t \right) \right] = e^{-\frac{\|\bm z\|^2_2 - z_j^2}{2}} \,.
	\end{equation}
	
	Combine the above equations in Eqs.~\eqref{omegaj2},~\eqref{ecos2},~\eqref{esin2}, and \eqref{bochnercos}, we have
	\begin{equation}
	\begin{split}
	\mathbb{E}_{\bm \omega}\left[\omega_j^2 cos(\bm \omega^{\!\top} \bm z)\right] &= \mathbb{E}_{\omega_j}\left[\omega_j^2 \cos(\omega_j z_j ) \right] \\
	&\quad \times \mathbb{E}_{\{\omega_1, \cdots ,\omega_d\}\backslash \omega_j }\!\! \left[\! \cos \! \left(\sum_{t=1, \neq j }^d \omega_t z_t \right) \! \right]  \\
	&= e^{-\frac{z_j^2}{2}} (1 - z_j^2) e^{-\frac{\|\bm z\|^2_2 - z_j^2}{2}} \\
	&= (1 - z_j^2)e^{-\frac{\|\bm z\|^2_2}{2}} \,.
	\end{split}
	\end{equation} 
	Accordingly, we can conclude
	\begin{equation*}
	\begin{split}
	\mathbb{E}_{\bm \omega}\left(f(\bm \omega) \sum_{j=1}^{d} \omega_j^2\right) &= \sum_{j=1}^{d} \mathbb{E}_{\bm \omega}\left[\omega_j^2 \cos(\bm \omega^{\!\top} \bm z)\right] \\
	&= e^{-\frac{\|\bm z\|^2_2}{2}} \sum_{i=1}^d (1 - z_j^2) \,,
	\end{split}
	\end{equation*}
	which yields the final result.
\end{proof}

\subsubsection{Proof of Theorem~\ref{thmvar}}
\label{sec:proofthmvar}
In the next, we are ready to prove Theorem~\ref{thmvar}.
\begin{proof}
	For ease of description, we use some short notations including $\sum_{i=1}^d[f]:= \sum_{i=1}^d \left[ f(\lambda_1 \bm e_i) + f(-\lambda_1 \bm e_i) \right]$, $Q:= Q^{(1,d)}(f)$, and $z:=\| \bm z \|_2$.
	
	Recall RFF, its kernel approximation form is obtained via the Monte Carlo sampling
	\begin{equation*}
	\frac{1}{D}\sum_{i=1}^D \cos (\bm \omega_i^{\!\top}\bm z), \quad \bm \omega_i \sim \mathcal{N}(\bm 0, \bm I_d)\,.
	\end{equation*}
	By virtue of $\mathbb{E}[\cos(\bm \omega^{\!\top} \bm z)] = e^{-z^2/2}$  and $\mathbb{V}[\cos(\bm \omega^{\!\top} \bm z)] = {(1-e^{-z^2})^2}/{2}$ \cite{Yu2016Orthogonal}, we have 
	\begin{equation*}
	\mathbb{E}[\text{RFF}] = e^{-z^2/2}\,, \mathbb{V}[\text{RFF}] = \frac{(1-e^{-z^2})^2}{2D}\,.
	\end{equation*}
	
	Due to $\mathbb{V}[\bar{R}_1(f, \bm \omega)] = {1}/{D}(\mathbb{V}[R_1(f, \bm \omega)])$ with $\bm \omega \sim \mathcal{N}(\bm 0, \bm I_d)$, in the next we focus on $\mathbb{V}[R_1(f)]$.
	The variance of $R_1(f)$ can be formulated as
	\begin{equation*}
	\begin{split}
	&\mathbb{V}[R_1(f)] = \mathbb{E} \left[f(\bm \omega) - M^{(1,d)}(f) + Q^{(1,d)}(f) \right]^2 \\
	& \qquad \qquad -  \left\{ \mathbb{E} \left[ f(\bm \omega) - M^{(1,d)}(f) + Q^{(1,d)}(f) \right] \right\}^2 \\
	&= \mathbb{E} \left[f(\bm \omega) - M^{(1,d)}(f) + Q^{(1,d)}(f)\right]^2 - \Big\{ \mathbb{E} [f(\bm \omega)] \Big\}^2 \\
	&= \mathbb{E} \left[f(\bm \omega) - M^{(1,d)}(f) + Q^{(1,d)}(f)\right]^2 - e^{-z^2} \,,
	\end{split}
	\end{equation*}
	where the used Gaussian kernel admits $f(\bm \omega) = \cos(\bm \omega^{\!\top} \bm z)$.
	
	Further, the above equation can be rewritten as	
	\begin{equation}\label{varrmf}
	\begin{split}
	\mathbb{V}[R_1(f)] & = \mathbb{E}[f^2(\bm \omega)] + \mathbb{E}[(M^{(1,d)}(f))^2] + [Q^{(1,d)}(f)]^2 \\
	&\quad + 2\mathbb{E}[f(\bm \omega) Q^{(1,d)}(f)] - 2\mathbb{E}[f(\bm \omega)M^{(1,d)}(f)] \\
	&\quad - 2\mathbb{E}[M^{(1,d)}(f) Q^{(1,d)}(f)] - e^{-z^2} \\
	& = D \mathbb{V}[\text{RFF}] \!-\! [Q^{(1,d)}(f)]^2 \!+\! 2e^{-z^2/2}Q^{(1,d)}(f) \\
	&\quad + \mathbb{E}[(M^{(1,d)}(f))^2] - 2\mathbb{E}[f(\bm \omega)M^{(1,d)}(f)] \,.
	\end{split}
	\end{equation}
	
	Hence, we need to bound $\mathbb{E}[(M^{(1,d)}(f))^2]$ and $\mathbb{E}[f(\bm \omega)M^{(1,d)}(f)]$ in Eq.~\eqref{varrmf}. First, byy expanding $[M^{(1,d)}(f)]^2$ in Eq.~\eqref{m1df}, we have
	\begin{equation*}
	\begin{split}
	& [M^{(1,d)}(f)]^2 = \left( 1-\frac{\sum_{i=1}^d \omega_i^2}{\lambda_1^2} \right)^2 + \frac{(\sum_{i=1}^d \omega_i^2)^2}{4 \lambda_1^4d^2} \left\{\sum_{i=1}^d[f] \right\}^{\!2}\\
	& \qquad \qquad \quad ~~+ \left( 1-\frac{\sum_{i=1}^d \omega_i^2}{\lambda_1^2} \right) \frac{\sum_{i=1}^d \omega_i^2}{\lambda_1^2 d} \sum_{i=1}^d[f]  \\
	& = 1 \!-\! \frac{2 \sum_{i=1}^d \omega_i^2}{\lambda_1^2} \!+\! \frac{(\sum_{i=1}^d \omega_i^2)^2}{\lambda_1^4}  \!+\! \frac{(\sum_{i=1}^d \omega_i^2)^2}{4\lambda_1^4d^2} \left\{\sum_{i=1}^d[f] \right\}^{\!2} \\
	&\quad + \frac{\sum_{i=1}^d \omega_i^2}{\lambda_1^2 d} \sum_{i=1}^d[f] - \frac{(\sum_{i=1}^d \omega_i^2)^2}{\lambda_1^4d} \sum_{i=1}^d[f] \,,
	\end{split}
	\end{equation*}
	where we use $f(\bm 0) = 1$ for the Gaussian kernel.
	Accordingly, we have
	\begin{equation}\label{em1df2}
	\begin{split}
	\mathbb{E}[M^{(1,d)}(f)]^2 & = 1 -\frac{2d}{\lambda_1^2}+ \frac{d^2+2d}{\lambda_1^4} + \frac{d+2}{4\lambda_1^4d} \left\{\sum_{i=1}^d[f] \right\}^2 \\
	& + \frac{1}{\lambda_1^2} \sum_{i=1}^d[f] - \frac{d+2}{\lambda_1^4} \sum_{i=1}^d[f]
	\end{split}
	\end{equation}
	where $\sum_{i=1}^d \omega_i^2 \sim \chi(d)$, $\mathbb{E}(\sum_{i=1}^d \omega_i^2) = d$, $\mathbb{V}(\sum_{i=1}^d \omega_i^2) = 2d$ and $\mathbb{E}([\sum_{i=1}^d \omega_i^2]^2) = d^2+2d$. 
	
	Second, we estimate $\mathbb{E}[f(\bm \omega) M^{(1,d)}(f)]$ in Eq.~\eqref{varrmf}.
	The notation $f(\bm \omega) M^{(1,d)}(f)$ is formulated as 
	\begin{equation*}
	\begin{split}
	f(\bm \omega) M^{(1,d)}(f) \!\!=\!\! \left(\! 1\!-\!\frac{\sum_{i=1}^d \omega_i^2}{\lambda_1^2} \!\!\right) \! \! f(\bm \omega) \!+\! \frac{\sum_{i=1}^d \omega_i^2}{2\lambda_1^2d} f(\bm \omega) \!\! \sum_{i=1}^d [f] \,.
	\end{split}
	\end{equation*}
	Accordingly, we have  
	\begin{equation}\label{efm1df}
	\begin{split}
	& \mathbb{E}[f(\bm \omega) M^{(1,d)}(f)]  \!=\!  e^{-z^2/2} \!+\!\! \left(\frac{\sum_{i=1}^d [f]}{2\lambda_1^2d} \!-\! \frac{1}{\lambda_1^2} \right) \!\mathbb{E}\!\left[\! f(\bm \omega) \sum_{i=1}^d \omega_i^2 \!\right] \\
	& = e^{-z^2/2} + \frac{1}{\lambda_1^2} \left(-1 + \frac{\sum_{i=1}^d [f]}{2d}\right) e^{-\frac{z^2}{2}} (d-z^2)\,,
	\end{split}
	\end{equation}
	where we use $f(\bm 0) = 1$ and $\mathbb{E}[f(\bm \omega)] = e^{-z^2/2}$ and Lemma~\ref{lemmafw}.
	In our third-degree rule with $m=1$, Eq.~\eqref{q1df} implies $\sum_{i=1}^d [f] = 2\lambda_1^2 Q - 2\lambda_1^2 + 2d$, and thus we have
	\begin{equation*}
		\mathbb{E}[M^{(1,d)}(f)]^2 = 1 + \frac{d+2}{d}(Q-1)^2 + 2(Q-1) \,, 
	\end{equation*}
	and
	\begin{equation*}
		\begin{split}
			-2\mathbb{E}[f(\bm \omega) M^{(1,d)}(f)] = -2 e^{-z^2/2} -\frac{2(Q-1)}{d} e^{-\frac{z^2}{2}} (d - z^2) \,.
		\end{split}
	\end{equation*}
	
	Hence, combining the above equations into Eq.~\eqref{varrmf}, we have
	\begin{equation}\label{DRmf}
	\begin{split}
	& \mathbb{V}[\bar{R}_1(f)] - \mathbb{V}[\text{RFF}]  \!=\! \frac{2}{Dd} \left((1-Q)^2 \!-\! (1-Q)z^2 e^{-\frac{z^2}{2}} \right) \\
	& \quad = \frac{2}{Dd} \left(\! \left[\! (1\!-\!Q) \!-\! \frac{1}{2} z^2 e^{-\frac{z^2}{2}} \!\right]^2\!\! \!\!-\! \frac{1}{4} z^4 e^{-{z^2}} \! \right)  \,.
	\end{split}
	\end{equation}
	Since $Q$ is the approximation of $I_{d}(f) = k(\bm x, \bm y) \in [0,1]$ for the Gaussian kernel, we can also consider $Q^{(1,d)}(f) \in [0,1]$. Note that, even if the estimation $Q^{(1,d)}(f)$ is out of $[0,1]$, we can still set it to $[0,1]$ by a threshold operator and thus $1-Q \geq 0$.
	Finally, Eq.~\eqref{DRmf} can be formulated as 
	\begin{equation*}
	\mathbb{V}[\bar{R}_1(f)] -  \mathbb{V}[\text{RFF}] < 0 ~~\mbox{when}~~ 1-Q < z^2 e^{-\frac{z^2}{2}} \,,
	\end{equation*}
	which concludes the proof.
	
\end{proof}



\section{Proof of Theorem~\ref{thirdproj}}
\label{sec:proofthirdproj}

To prove Theorem~\ref{thirdproj}, we need the following lemma.
\begin{lemma}\label{theorem41}
	(Theorem 4.1 in \cite{jia2015relations})
	Denote $\bm x^M$ and $\bm s^M$ as polynomials with total degree $M$, then the following integral satisfies 
	\begin{equation*}
	\begin{split} I'\left({\bm \omega}^{d}\right) &=\int_{\mathbb{R}^{n}} \omega_{1}^{\alpha_{1}} \omega_{2}^{\alpha_{2}} \cdots \omega_{d}^{\alpha_{d}} \exp \left(-\bm{\omega}^{\!\top} \bm{\omega}\right) \mathrm{d} \bm{x} \\  &=\int_{0}^{\infty} r^{d-1+M} \exp \left(-r^{2}\right) \mathrm{d} r \int_{U_{d}} \bm{s}^{M} \mathrm{d} \tau(\bm{s}) \,,\end{split}
	\end{equation*}
	can be exactly calculated by the quadrature rules $I'(\bm \omega) = \sum_{j=1}^{N_g} \bar{a}_j f(\bar{\bm \gamma})$ with $f'(\bm \omega) = \bm \omega^M$. Then the spherical integral can be expressed as 
	$\int_{U_{d}} \bm{s}^{M} \mathrm{d} \tau(\bm s)  =\sum_{j=1}^{N_{p}} a_{s, j}\left(\mathrm{s}_{j}\right)^{M}$ 
	with the nodes $ \bm s_j = \frac{\bar{\bm \gamma}_j}{\|\bar{\bm \gamma}_j \|_2}$ and
	and the weights $a_{s, j}$ of the spherical rule are
	\begin{equation*}
	a_{s,j} = \frac{\bar{a}_{j}\left(\bm{s}_{j}\right)^{M}}{\int_{0}^{\infty} r^{d-1+M} \exp \left(-r^{2}\right) \mathrm{d} r} = \frac{\bar{a}_{j}\left(\bm{s}_{j}\right)^{M}}{\Gamma(d/2+M/2)/2} \,,
	\end{equation*}
	where $N_p$ is the number of projected quadrature non-zero nodes. Note that $N_p \leq N_g$.
\end{lemma}

Formally, we are ready to prove Theorem~\ref{thirdproj}.
\begin{proof}
	The integral in Eq.~\eqref{originte} can be reformulated as
	\begin{equation*}
	\begin{split}
	I_d\left(f_{\bm{x} \bm y}\right) 
	& =  \int_{\mathbb{R}^d} f_{\bm x \bm y}(\bm \omega) \mathcal{N}(\bm \omega; \bm 0, \bm I_d) \mathrm{d}{\bm \omega} \\
	& = \pi^{-\frac{d}{2}} \int_{\mathbb{R}^d}  e^{-\bm \omega^{\!\top} \bm \omega} f(\sqrt{2} \bm \omega) \mathrm{d}{\bm \omega} \,.
	\end{split}
	\end{equation*}
	Hence, the integral $I'(\bm \omega) = \int_{\mathbb{R}^d} f(\bm \omega) \exp (- \bm \omega^{\!\top} \bm \omega) \mathrm{d} \bm \omega$ can be approximated by our third-degree D-FS in Eq.~\eqref{q1df}, that is
	\begin{equation*}
	\begin{split}
	& I'(\bm \omega) = \pi^{\frac{d}{2}} \int_{\mathbb{R}^d} f(\frac{\bm \omega}{\sqrt{2}}) \mathcal{N}(\bm \omega; \bm 0, \bm I_d) \mathrm{d} \bm \omega \approx \sum_{i=1}^{2d+1} \bar{a}_i f(\bar{\bm \gamma}_i) \\
	& = (1-\frac{d}{\lambda_1^2}) \pi^{\frac{d}{2}} f(\bm 0) + \frac{\pi^{\frac{d}{2}}}{2 \lambda_1^2} \sum_{i=1}^d \Big( f(\frac{\lambda_1}{\sqrt{2}} \bm e_i) + f(-\frac{\lambda_1}{\sqrt{2}} \bm e_i) \Big) \,, 
	\end{split}
	\end{equation*}
	with
	\begin{equation*}\label{3fsirp}
	\left\{
	\begin{array}{rcl}
	\begin{split}
	& \bar{\bm \gamma}_{i} =\bm 0 ;~  \bar{a}_{i}=(1-\frac{d}{\lambda_1^2}) \pi^{\frac{d}{2}};~ i =0 \\
	& \bar{\bm \gamma}_{i} =\frac{\lambda_1}{\sqrt{2}} \bm{e}_{i};~ \bar{a}_{i} =\frac{\pi^{\frac{d}{2}}}{2 \lambda_1^2};~ 1  \leq i \leq d \\
	& \bar{\bm \gamma}_{i} =-\frac{\lambda_1}{\sqrt{2}} \bm{e}_{i-d} ;~ \bar{a}_{i} =\frac{\pi^{\frac{d}{2}}}{2 \lambda_1^2};~ d+1 \leq i \leq 2 d\,.
	\end{split}
	\end{array}\right.
	\end{equation*}
	By projecting $\bar{\bm \gamma}_i$ on the surface of the unit $U_d$ sphere with an uniform random orthogonal matrix $\bm Q$, we have
	\begin{equation}\label{orthfsir}
	\bm s_i = \frac{\bm Q \bar{\bm \gamma}_i}{\| \bm Q \bar{\bm \gamma}_i \|_2^2} = \left\{
	\begin{array}{rcl}
	\begin{split}
	& \bm Q \bm e_i;~ 1 \leq i \leq d, \\
	& -\bm Q \bm e_{i-d};~ d+1  \leq i \leq 2d\,, \\
	\end{split}
	\end{array}\right.
	\end{equation}
	with $\| \bm Q \bar{\bm \gamma}_i \|_2 = \| \bar{\bm \gamma}_i \|_2$. 
	Note that the point at the origin has been omitted. 
	By Lemma~\ref{theorem41}, for the third-degree, the polynomial degree $M$ is set to 2.
	Accordingly, the weight $a_{s,j}$ of the spherical rule can be obtained by
	\begin{equation}\label{orthweight}
	a_{s,j} = \frac{\bar{a}_{j}\left(\bm{s}_{j}\right)^{M}}{\Gamma(d/2+M/2)/2} = \frac{\pi^{\frac{d}{2}}}{2\lambda_1^2} \frac{\lambda_1^2}{2} / \left(\frac{\frac{d}{2}\Gamma(\frac{d}{2})}{2} \right) = \frac{|U_d|}{2d}\,.
	\end{equation}
	Hence, using Eq.~\eqref{orthfsir} and Eq.~\eqref{orthweight} yields the spherical rule
	\begin{equation*}
	I_{\bm Q, U_{d}}(s)=\frac{|U_{d}|}{2 d} \sum_{j=1}^{d}\left[s\left(\bm Q \bm{e}_{j}\right)+s\left(-\bm Q \bm{e}_{j}\right)\right] \,,
	\end{equation*}
	which is identical to the third-degree stochastic spherical integration rule in Eq.~\eqref{ssr3}.
\end{proof}

\end{document}